\documentclass[11pt]{article}

\usepackage[margin=1in]{geometry}
\usepackage{graphicx}
\usepackage{amsthm,amsmath,amsfonts,amssymb}
\usepackage[authoryear,round]{natbib}
\usepackage{url}
\usepackage{subcaption}
\usepackage{longtable}
\usepackage{threeparttable}
\usepackage{multirow}
\usepackage{booktabs}
\usepackage{array}
\usepackage{lscape}
\usepackage{float}
\usepackage[colorlinks,citecolor=blue,urlcolor=blue,linkcolor=blue]{hyperref}

\allowdisplaybreaks
\def\mathbold{\boldsymbol}

\def\bv{\mathbold{v}}
\def\bw{\mathbold{w}}
\def\bA{\mathbold{A}}
\def\ba{\mathbold{a}}
\def\bX{\mathbold{X}}
\def\bx{\mathbold{x}}
\def\by{\mathbold{y}}

\def\R{{\mathbb{R}}}
\def\P{{\mathbb{P}}}
\def\E{{\mathbb{E}}}

\def\bdelta{\mathbold{\delta}}
\def\btheta{\mathbold{\theta}}

\def\hbdelta{\widehat{\bdelta}{}}
\def\bSigma{\mathbold{\Sigma}}

\def\bvar{\mathbold{\varepsilon}}
\def\bbeta{\mathbold{\beta}}
\def\hbbeta{\widehat{\bbeta}{}}
\def\bbbeta{\bar{\bbeta}{}}
\def\balpha{\mathbold{\alpha}}
\def\hbalpha{\widehat{\balpha}{}}
\def\bgamma{\mathbold{\gamma}}
\def\hbgamma{\widehat{\bgamma}{}}

\def\bu{\mathbold{u}}
\def\hbu{\widehat{\bu}{}}
\def\hbv{\widehat{\bv}{}}

\theoremstyle{plain}

\newtheorem{theorem}{Theorem}[section]

\theoremstyle{definition}
\newtheorem{assumption}{Assumption}[section]
\newtheorem*{remark}{Remark}

\title{Harnessing Source Heterogeneity for Cluster-Structured Transfer Learning}
\author{
Xiaohui Yin\thanks{These authors contributed equally to this work.}\\
Department of Statistics, University of Connecticut
\and
Jun Jin\footnotemark[1]\\
Department of Public Health Sciences, Henry Ford Health
\and
Shane J. Sacco\\
Center for Population Health, University of Connecticut Health Center
\and
Robert H. Aseltine\\
Center for Population Health, University of Connecticut Health Center
\and
Kun Chen\thanks{Corresponding author: \texttt{kun.chen@uconn.edu}.}\\
Department of Statistics, University of Connecticut
}
\date{}

\begin{document}
\maketitle

\begin{abstract}
Transfer learning is a natural strategy when a target population has
limited data but multiple related auxiliary sources are available. A
central difficulty is source heterogeneity: auxiliary sources may not
be equally useful, and their usefulness may vary in a structured,
cluster-like fashion. Existing transfer-learning
methods often reduce source selection to a binary
informative/non-informative decision, overlooking subgroups of sources
with differential transferability. Motivated by a suicide-risk study
using data from the Connecticut Hospital Information Management
Exchange (CHIME), comprising $636{,}758$ patients across $27$
hospitals, we propose Trans-GLMC, a cluster-structured
transfer-learning procedure for generalized linear models. The CHIME
setting illustrates the core challenge: hospital-specific risk models
are unstable because suicide attempts are rare at any single facility,
whereas indiscriminate pooling across hospitals can obscure
facility-level differences in patient mix and risk profiles. Trans-GLMC
first constructs a coefficient-based distance among the target and
candidate sources to recover latent source clusters. It then combines
global fusion, within-cluster refinement, and target debiasing to
produce an estimator that adapts to the detected structure. We
establish a non-asymptotic error bound that improves over its
unclustered counterpart whenever a meaningful target cluster exists
and matches the unclustered rate up to constants otherwise. In
simulations and in the CHIME study, Trans-GLMC improves
facility-specific prediction, identifies interpretable communities of
hospitals with mutual transferability, and recovers clinically
coherent suicide-risk factors.
\end{abstract}


\section{Introduction}
\label{sec:intro}

Transfer learning is increasingly used in modern applied statistics
when a target population has limited labeled data but related auxiliary
datasets are available. The central statistical question is not simply
whether to borrow information, but how to borrow without erasing the
features that make the target distinct. This question is especially
pressing in high-dimensional regression problems, where a
target-specific estimator may be unstable, while indiscriminate pooling
across sources can introduce bias and even lead to negative transfer
\citep{rosenstein2005transfer,zang2024accuracy}. Source heterogeneity
is therefore not a nuisance to be ignored; it is part of the structure
that a useful transfer-learning method should exploit.

Existing transfer-learning methods have made substantial progress on
this problem. \citet{li2022transfer} and \citet{tian2023transfer}
studied transfer learning in linear models and generalized linear
models (GLMs), respectively, and related developments cover
high-dimensional classification
\citep{bastani2021predicting,reeve2021adaptive,cai2021transfer},
unsupervised learning \citep{tian2023unsupervised}, graphical models
\citep{li2023transfer}, and federated estimation with privacy
guarantees \citep{li2024federated}. A common strategy is to first
identify an informative set of sources and then borrow from that set,
as in \citet{li2023transfer}, \citet{li2023targeting},
\citet{tian2023transfer}, and \citet{jin2024transfer}. This protects
against clearly unhelpful sources, but it still treats the retained
sources as essentially exchangeable. In many applications, the source
landscape is richer: auxiliary datasets may form latent subgroups, and
the sources in the target's subgroup may be much more transferable
than sources that are still broadly informative but less similar.

The motivating application for this work is suicide-risk prediction
from electronic health records (EHR) in the Connecticut Hospital
Information Management Exchange (CHIME). Suicide remains a major
public health crisis in the United States. Between 2001 and 2021, the
age-adjusted suicide rate rose by over 30\%, from 10.7 to 14.1 deaths
per 100{,}000 people, making suicide one of the ten leading causes of
death overall and the second leading cause among adolescents and young
adults
\citep{kessler1999prevalence,nock2008suicide,doshi2020identifying,stone2018vital}.
Prevention initiatives such as the Zero Suicide Initiative
\citep{labouliere2018zero} and the Joint Commission National Patient
Safety Goal 15.01.01 depend on accurate and proactive identification
of at-risk individuals during routine clinical encounters. EHR data
are especially valuable for this task because nearly half of people
who die by suicide have contact with the healthcare system in the 30
days before death
\citep{luoma2002contact,ahmedani2014health,ilgen2012patterns}. A
large literature has therefore developed EHR-based prediction models
for children and adolescents \citep{su2020machine,sacco2023target},
adults
\citep{barak2020validation,wilimitis2022integration,walsh2021prospective},
patients with prior behavioral-health conditions
\citep{simon2018predicting}, and veterans \citep{kessler2020using}.

Despite these advances, clinical deployment often faces a tension that
is fundamentally a transfer-learning problem. The CHIME data analyzed
in this paper include $636{,}758$ patients across $27$ hospitals, but
suicide attempts are rare at any single facility. A facility-specific
model may therefore have too few positive cases to estimate a stable
risk model. At the same time, hospitals differ in patient mix, coding
patterns, payer composition, geographic catchment, and care
environment, so a pooled model can blur important facility-level
differences. This is precisely the setting in which one wants to
borrow information from other hospitals, but only in a way that
respects source heterogeneity.

The CHIME study further suggests that this heterogeneity is
cluster-structured rather than merely binary. Hospitals do not divide
cleanly into sources that are helpful and sources that are harmful for
a given target. Instead, facilities with similar risk-model profiles
appear to form communities of mutual transferability. These
communities may reflect shared patient demographics, regional
patterns, clinical practice, or combinations of these factors, and
they need not coincide with administrative health-system boundaries.
Such structure creates a statistical opportunity: all informative
sources can stabilize an initial model, while sources in the target's
own cluster can provide a sharper refinement before final adjustment
to the target facility.

We propose \emph{Trans-GLMC}, a cluster-structured transfer-learning
procedure for high-dimensional GLMs with heterogeneous sources. The
method first constructs a coefficient-based pairwise distance among
the target and candidate sources, then applies hierarchical
density-based clustering to recover latent subgroups of mutually
similar sources. Given the detected structure, Trans-GLMC proceeds in
three stages. A global fusion step pools information from all detected
informative sources to obtain a stable initial estimator. A
within-cluster refinement step then re-estimates using only sources in
the target's cluster, exploiting tighter local homogeneity. Finally, a
target debiasing step corrects the remaining shift between the cluster
aggregate and the target-specific coefficient. When the detected
partition is trivial, the procedure reduces to the unclustered
Trans-GLM estimator of \citet{tian2023transfer}.

Our theoretical results show how the cluster structure improves
estimation. We establish a non-asymptotic error bound for Trans-GLMC
that separates the contribution of the global informative set from the
contribution of the target's cluster. When a nontrivial cluster around
the target exists, so that within-cluster heterogeneity is smaller
than the global source-target discrepancy, the bound is strictly
tighter than the corresponding unclustered transfer bound. When no
meaningful cluster structure exists, the rate matches the unclustered
bound up to constants. We also prove that the proposed
coefficient-distance clustering consistently recovers the latent
partition under a minimum-gap condition on between-cluster distances.

Empirically, we evaluate Trans-GLMC in simulations and in the CHIME
suicide-risk study. The simulations compare Trans-GLMC with
target-only learning, unclustered Trans-GLM, and several weighted
transfer variants, including a Q-aggregation approach adapted from
\citet{li2022transfer} and inverse-distance and spherical weighting
schemes. The proposed method produces the largest gains when the
target has a meaningful cluster of related sources, in agreement with
the theory. In the CHIME application, Trans-GLMC improves
facility-specific prediction over both target-only models and
unclustered transfer for most hospitals, with the largest gains at
small or isolated facilities. The fitted models also yield a
data-derived transferability network and clinically coherent risk
factors, including prior self-harm, mood disorders, substance-use
disorders, and anxiety.

The rest of the paper is organized as follows. Section~\ref{sec:chime}
describes the CHIME cohort, outcome and predictor definitions, and
exploratory evidence of cross-facility heterogeneity and cluster
structure. Section~\ref{sec:method} formalizes the transfer-learning
problem, presents Trans-GLMC, and develops the coefficient-distance
and cluster-detection procedure. Section~\ref{sec:theory} gives the
main theoretical results. Section~\ref{sec:simulation} reports
simulation experiments, and Section~\ref{sec:realdata} presents
facility-level predictive performance, the transferability network,
and estimated risk factors from the CHIME application.
Section~\ref{sec:discussion} concludes. Proofs, transfer-learning
variants, and additional numerical results are provided in the
supplement.

\section{Motivating Application: the CHIME Suicide Study}
\label{sec:chime}

\subsection{Cohort and data}\label{sec:chime:cohort}

The CHIME data, obtained from the Connecticut Department of Public
Health, integrate electronic health records from inpatient and
emergency encounters across $27$ hospitals in the state. Each record
contains patient demographic fields, payer type, and visit-level
International Classification of Diseases (ICD) diagnosis codes in
priority order. 
The records used in this study span January 1, 2012 to December 31,
2017.

We adopt a retrospective follow-up design to assess and predict
one-year suicide risk. We identified $636{,}758$ patients aged 18--64
years during the recruitment window between January 1, 2016 and
December 31, 2016.  Each patient was followed during the period from
January 1, 2017 to December 31, 2017 for the occurrence of a suicide
attempt. ICD-9 codes were converted to ICD-10 with the \texttt{touch}
R package \citep{wang2018touch}, and contacts whose primary ICD-10
code ended in ``D'' or ``S'' (indicating subsequent encounter or
sequela) were excluded.

Suicide attempts in the follow-up window were defined, as in prior EHR
studies \citep{su2020machine,xu2022improving,sacco2023target}, by
ICD-10 codes indicating intentional self-harm or by combinations of
suicide-related mental-disorder and injury codes co-occurring in the
same visit. Patients with no follow-up visit or no visit associated
with a suicide attempt were assigned to the control group.

Table~\ref{app:tab:demo} summarizes patient characteristics aggregated
across the $27$ hospitals. Individual facilities contributed a median
of $19{,}395$ patients (IQR $13{,}651$--$30{,}289$; range $5{,}129$ to
$78{,}288$) and $65$ suicide attempts (IQR $40$--$120$; range $12$ to
$362$). Facility-specific demographics and case counts are reported in
Table~S.3 of the Supplement.

\begin{table}[H]
  \centering
  \caption{Demographics of the patient cohorts across the 27
    hospitals. Values are summarized across facilities.}
  \label{app:tab:demo}
  \begin{tabular}{l r r r r}
    \toprule
    Demographic variable & Median & IQR & Minimum & Maximum\\
    \midrule
    Sample size, N & 19395 & 13651--30289 & 5129 & 78288\\
    Suicide attempts, N & 65 & 40--120 & 12 & 362\\
    Suicide attempts, \% & 0.4 & 0.2--0.5 & 0.1 & 1.1\\
    Age group, years, \% & & & & \\
    \hspace{1em} 18-24 & 16.2 & 15.2--16.7 & 12.3 & 26.5\\
    \hspace{1em} 25-39 & 35.3 & 34.0--36.7 & 30.8 & 41.0\\
    \hspace{1em} 40-54 & 30.1 & 28.8--30.8 & 26.5 & 32.1\\
    \hspace{1em} 55-64 & 18.8 & 17.6--20.1 & 15.8 & 22.6\\
    Male, \% & 44.7 & 43.9--45.4 & 38.6 & 47.3\\
    Non-Hispanic White, \% & 66.9 & 48.9--76.5 & 32.5 & 93.7\\
    Medicaid insurance, \% & 36.6 & 30.1--39.6 & 14.0 & 48.3\\
    \bottomrule
  \end{tabular}
\end{table}

\subsection{Exploratory analysis: heterogeneity and apparent cluster
  structure}\label{sec:chime:eda}

Two features of the CHIME cohort motivate the methodology developed in
the next section.

\emph{Rare-outcome heterogeneity.} Suicide-attempt incidence varies by
nearly an order of magnitude across facilities ($0.07\%$ to $1.07\%$),
and several facilities contribute fewer than $40$ positive cases over
the full follow-up year (Table~S.3 of the Supplement). A target-only
$\ell_1$-penalized logistic regression trained on a single such
facility is statistically unstable; Section~\ref{sec:realdata} shows
that facility 6 (a suburban facility with the lowest observed
suicide-attempt rate) 
and facility 27 (the smallest facility in the
cohort), 
both with fewer than $20$ observed events in the test period, yield
target-only models that are no better than random guessing. This
motivates transferring information from external facilities.

\emph{Source heterogeneity and structure.} At the same time, CHIME
facilities differ substantially in patient demographics: the
Non-Hispanic White fraction ranges from $32.5\%$ to $93.7\%$, Medicaid
enrollment from $14.0\%$ to $48.3\%$, and the distribution of
diagnosis-code prevalences varies accordingly. A pooled regression
over all $27$ facilities ignores this heterogeneity, and our own
analyses in Section~\ref{sec:realdata} show that unclustered transfer
learning only partially mitigates it. Conversely, using the
coefficient-based distance metric developed in
Section~\ref{sec:distance} and hierarchical density-based clustering,
we find that the CHIME facilities separate into several interpretable
communities with similar coefficient signals (see
Figure~\ref{app:fig:transferability} in
Section~\ref{sec:realdata} for community assignments and their
geographic footprint). These community memberships are stable across
random splits of the data, are not fully explained by geographic
proximity, and indicate that a transfer-learning procedure should
treat hospitals in the same community as near-exchangeable while
down-weighting those in other communities. This is precisely the
structure we aim to exploit with our cluster-structured transfer
learning approach.

\section{Cluster-Structured Transfer Learning for GLMs}
\label{sec:method}
\label{sec:tlcluster}

\subsection{Notation and problem setup}\label{sec:setup}

We first collect notation used throughout. For scalars $c_1,c_2$, write
$c_1\vee c_2=\max(c_1,c_2)$ and $c_1\wedge c_2=\min(c_1,c_2)$. For
$\bv=(v_1,\ldots,v_n)^\top$, $\|\bv\|_1=\sum_i|v_i|$,
$\|\bv\|=\bigl(\sum_i v_i^2\bigr)^{1/2}$, and
$\|\bv\|_\infty=\max_i|v_i|$. For a matrix $\bA\in\R^{n\times n}$,
$\|\bA\|_1=\sup_j\sum_i|A_{ij}|$. For sequences $\{a_n\}$, $\{b_n\}$,
$a_n\asymp b_n$ means $a_n=O(b_n)$ and $b_n=O(a_n)$;
$a_n\gg b_n$ means $|b_n/a_n|\to 0$; and $\lesssim$, $\gtrsim$
denote inequality up to a multiplicative constant.

Our main interest is to learn a regression model for the target
facility, with outcome $y^{(0)}\in\R$ and predictors
$\bx^{(0)}\in\R^p$. We assume the canonical exponential-family GLM
\begin{equation}
\label{eq:glmmodel-tg}
y^{(0)}\mid\bx^{(0)} \sim
\rho\bigl(y^{(0)};\phi\bigr)\,
\exp\Bigl\{\bigl[y^{(0)}\bx^{(0)\top}\bbeta^{(0)}
-\psi\bigl(\bx^{(0)\top}\bbeta^{(0)}\bigr)\bigr]/\phi\Bigr\},
\end{equation}
where $\psi$ is the cumulant function (twice continuously
differentiable with $\psi''>0$), $\rho(\cdot;\phi)$ is the base
measure, and $\phi>0$ is a known dispersion parameter. The same
parameterization is assumed for each of the $K$ source facilities,
\begin{equation}
\label{eq:glmmodel-sc}
y^{(k)}\mid\bx^{(k)} \sim
\rho\bigl(y^{(k)};\phi\bigr)\,
\exp\Bigl\{\bigl[y^{(k)}\bx^{(k)\top}\bbeta^{(k)}
-\psi\bigl(\bx^{(k)\top}\bbeta^{(k)}\bigr)\bigr]/\phi\Bigr\},
\end{equation}
with facility-specific coefficients $\bbeta^{(k)}\in\R^p$,
$k=1,\ldots,K$. The target coefficient $\bbeta^{(0)}$ is sparse with
$s_0$ nonzero entries and $\|\bbeta^{(0)}\|_\infty$ bounded. The
dispersion $\phi$ does not vary across facilities and is treated as
fixed throughout; in our CHIME application $\phi\equiv 1$ for the logistic GLM, and more
generally $\phi$ enters only as a global scale that does not affect
the regression-coefficient estimating equations. We therefore suppress it in the
exposition of the loss functions below.

Sources are useful for the target to the extent that their
coefficient vectors $\bbeta^{(k)}$ lie close to $\bbeta^{(0)}$. We
formalize this by assuming a latent partition of facilities into
clusters of mutually similar coefficients. For clarity of exposition
we present the two-cluster case and write
$\{0,1,\ldots,K\}={\cal C}_1\cup{\cal C}_2$,
${\cal C}_1\cap{\cal C}_2=\emptyset$, with $0\in{\cal C}_1$ without
loss of generality; the multi-cluster case is identical at the cost
of additional notation. Define the pairwise population distance
\[
d_{ij} = \bigl\|\bbeta^{(i)}-\bbeta^{(j)}\bigr\|_1,
\qquad i,j\in\{0,1,\ldots,K\}.
\]
The clusters are assumed to be \emph{well-separated} in the sense that
\begin{align*}
\min_{i,j\in{\cal C}_g,\,k\in{\cal C}_{g'}}\,
\max\bigl(d_{ik},d_{jk}\bigr)
\;>\;
\max_{i,j\in{\cal C}_g}\,d_{ij}
\quad\text{for all }g\ne g'\in\{1,2\},
\end{align*}
i.e., between-cluster distances exceed within-cluster distances. We
will repeatedly use three derived quantities, all measured in $\ell_1$
norm of coefficient differences:
\begin{itemize}
\item the \emph{global informativeness level} $d>0$, used to define
the (unknown) informative set
${\cal A}(d)=\{1\le k\le K:d_{k0}\le d\}$ relative to the target;
\item the \emph{within-cluster informativeness}
$d_{\cal C}=\max_{k\in{\cal C}_1\setminus\{0\}}d_{k0}$, the largest
distance from the target to any source in its own cluster; by
construction $d_{\cal C}\le d$ whenever ${\cal C}_1\setminus\{0\}
\subseteq{\cal A}(d)$;
\item facility sample sizes $n_0$ (target), $n_{\cal A}=
\sum_{k\in{\cal A}}n_k$ (sources in the informative set), and
$n_{\cal C}=\sum_{k\in({\cal C}_1\cap{\cal A})\setminus\{0\}}n_k$
(sources in the target's cluster).
\end{itemize}
The pair $(d_{\cal C},n_{\cal C})$ encodes how concentrated and how
populous the target's cluster is, and is the source of the
improvement over unclustered transfer learning that we establish in
Section~\ref{sec:theory}.

Assume we observe $n_k$ i.i.d.\ samples from facility $k$ under
\eqref{eq:glmmodel-tg}--\eqref{eq:glmmodel-sc}, namely
$(\bX^{(k)},\by^{(k)})$ with
$\bX^{(k)}=(\bx_1^{(k)},\ldots,\bx_{n_k}^{(k)})^\top\in\R^{n_k\times p}$
whose rows are i.i.d., zero-mean, sub-Gaussian with covariance
$\bSigma^{(k)}$, and $\by^{(k)}\in\R^{n_k}$, for $k=0,1,\ldots,K$. Our 
task is to estimate $\bbeta^{(0)}$ from
$(\bX^{(0)},\by^{(0)})$ and $\{(\bX^{(k)},\by^{(k)})\}_{k=1}^{K}$,
either with the partition $\{{\cal C}_1,{\cal C}_2\}$ given by
domain knowledge or with the partition inferred from data via the
procedure of Section~\ref{sec:distance}.

\subsection{Conventional Estimators}\label{sec:glmtrans}

We briefly describe two baselines used as references throughout.

Using target data alone, the
$\ell_1$-penalized GLM estimator is
\begin{equation}
\label{eq:glm-classic}
\hbbeta^{(0)} = \mathop{\arg\min}_{\bbeta\in\R^p}
\frac{1}{n_0}\sum_{i=1}^{n_0}
\Bigl[-y_i^{(0)}\bx_i^{(0)\top}\bbeta
+\psi\bigl(\bx_i^{(0)\top}\bbeta\bigr)\Bigr]
+\lambda_0\|\bbeta\|_1,
\end{equation}
with regularization parameter $\lambda_0$. We refer to this as the
\emph{target-only} method.

Given an informative set
${\cal A}\subseteq\{1,\ldots,K\}$, the oracle Trans-GLM
\citep{tian2023transfer} fuses target and informative-source data and
then debiases against the target:
\begin{align*}
\hbbeta^{\cal A} &= \mathop{\arg\min}_{\bbeta\in\R^p}
\frac{1}{n_{\cal A}+n_0}
\sum_{k\in\{0\}\cup{\cal A}}
\Bigl[-(\by^{(k)})^\top\bX^{(k)}\bbeta
+\textstyle\sum_{i=1}^{n_k}\psi\bigl(\bbeta^\top\bx_i^{(k)}\bigr)\Bigr]
+\lambda_1\|\bbeta\|_1,\\
\hbdelta^{\cal A} &= \mathop{\arg\min}_{\bdelta\in\R^p}
\frac{1}{n_0}
\Bigl[-(\by^{(0)})^\top\bX^{(0)}(\hbbeta^{\cal A}+\bdelta)
+\textstyle\sum_{i=1}^{n_0}\psi\bigl((\hbbeta^{\cal A}+\bdelta)^\top\bx_i^{(0)}\bigr)\Bigr]
+\lambda_2\|\bdelta\|_1,
\end{align*}
yielding $\hbbeta^{(0)}_{TGLM}=\hbbeta^{\cal A}+\hbdelta^{\cal A}$.
In practice ${\cal A}$ is unknown and is estimated by a cross-validated
data-driven detection procedure adapted from \citet{tian2023transfer}; the
full algorithm is presented in Section~S.1 of the Supplement. We
denote its output by $\widehat{\cal A}$ and use it both as a baseline
competitor (``Trans-GLM'') and as the screening set fed into our
cluster-structured procedure below.

Crucially, Trans-GLM treats all members of $\widehat{\cal A}$ as
exchangeable, with no mechanism for distinguishing a tightly
concentrated sub-cluster around the target from a loose collection of
informative-but-distinct sources. The proposal in
Section~\ref{sec:glmtransc} generalizes this construction and
recovers it as a special case when the detected cluster structure is
trivial.

\subsection{Trans-GLMC}\label{sec:glmtransc}

Trans-GLM enrolls every source in $\widehat{\cal A}$ on equal footing,
which discards the differential informativeness implied by the cluster
structure. Trans-GLMC remedies this by adding an intermediate
within-cluster fusion step. The schematic in
Figure~\ref{fig:schematic1} illustrates the resulting three-stage
estimator: a global fusion across all detected informative sources, a
within-cluster refinement using only sources in the target's cluster,
and a debiasing step on the target.

\begin{figure}[H]
\centering
\includegraphics[width=0.7\linewidth]{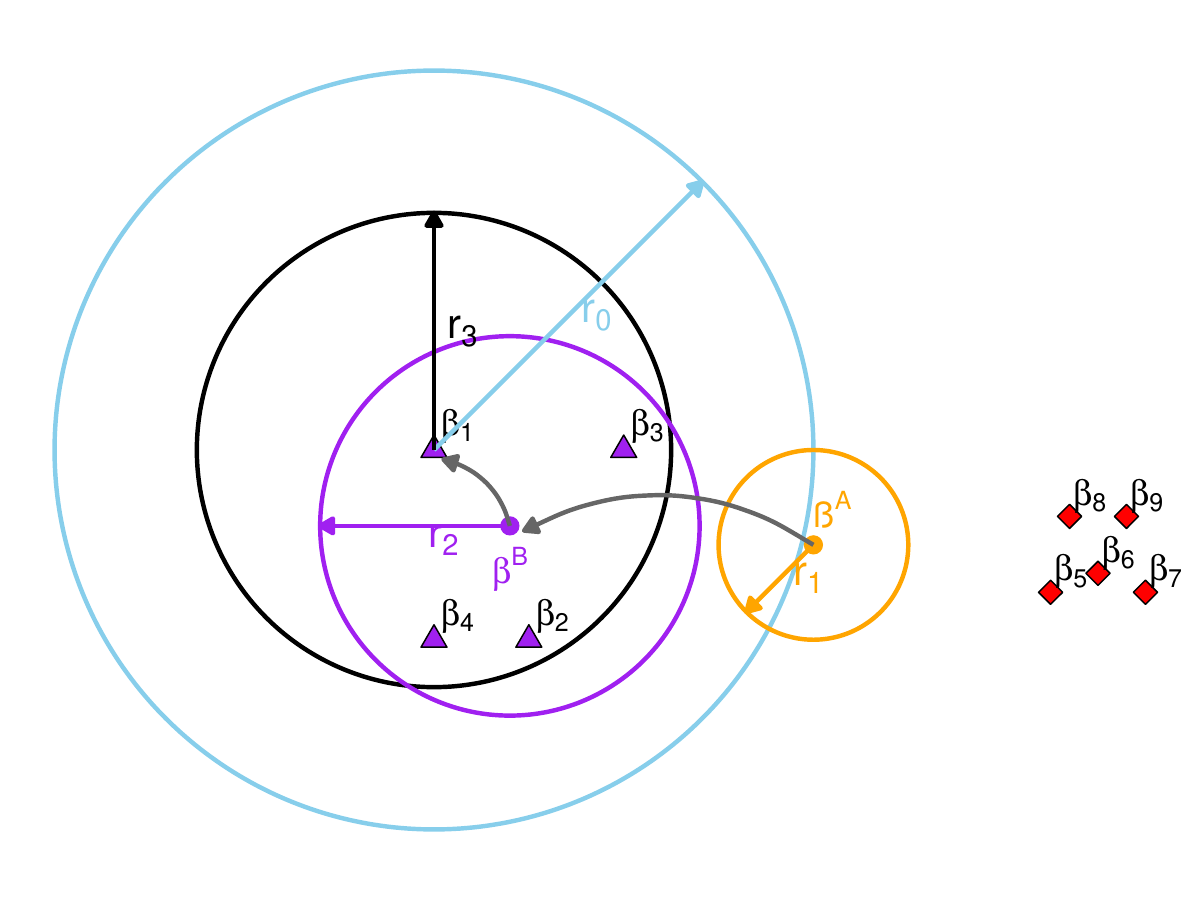}
\caption{Schematic of Trans-GLMC. Step~1 produces an estimate
$\hbbeta^{\cal A}$ of the population aggregate
$\bbeta^{\cal A}$ (denoted $\bbeta^{A}$ in the figure) from all
detected informative sources; Step~2 estimates the cluster offset
$\balpha$ that takes $\bbeta^{\cal A}$ to the cluster aggregate
$\bbeta^{B}=\bbeta^{\cal A}+\balpha$; Step~3 estimates the
target-specific offset $\bgamma$ that takes the cluster aggregate to
$\bbeta^{(0)}=\bbeta^{\cal A}+\balpha+\bgamma$.}
\label{fig:schematic1}
\end{figure}

Adopting the bias-regularization viewpoint of 
\citet{li2022transfer,tian2023transfer}, we represent
$\bbeta^{(0)}=\bbeta^{\cal A}+\balpha+\bgamma$, where
$\bbeta^{\cal A}$ is the population aggregate over all informative
sources, $\balpha$ is the cluster offset, and $\bgamma$ is the
target-specific offset. The three terms enjoy progressively smaller
sample sizes but progressively smaller magnitudes, and the
resulting bias--variance trade-off is favorable whenever a non-trivial
cluster exists around the target. Concretely, given the screened
informative set $\widehat{\cal A}$ from Section~\ref{sec:glmtrans} and a
cluster partition (known or estimated as in
Section~\ref{sec:distance}), let
$n_{\cal A}=\sum_{k\in\widehat{\cal A}}n_k$ and
$n_{\cal C}=\sum_{k\in{\cal C}_1\cap\widehat{\cal A}\setminus\{0\}}n_k$.
Trans-GLMC solves:
\begin{itemize}
\item[\bf Step 1.] (Fuse all informative sources)
\begin{align*}
\hbbeta^{\cal A}
=\mathop{\arg\min}_{\bbeta}
\frac{1}{n_{\cal A}+n_0}
\sum_{k\in\{0\}\cup\widehat{\cal A}}\sum_{i=1}^{n_k}
\Bigl\{-y_i^{(k)}\bx_i^{(k)\top}\bbeta
+\psi\bigl(\bx_i^{(k)\top}\bbeta\bigr)\Bigr\}
+\lambda_1\|\bbeta\|_1.
\end{align*}
\item[\bf Step 2.] (Within-cluster refinement)
\begin{align*}
\hbalpha
= \mathop{\arg\min}_{\balpha}
& \frac{1}{n_{\cal C}+n_0}
\sum_{k\in\{0\}\cup({\cal C}_1\cap\widehat{\cal A})}\sum_{i=1}^{n_k}
\Bigl\{-y_i^{(k)}\bx_i^{(k)\top}(\hbbeta^{\cal A}+\balpha)
+\psi\bigl(\bx_i^{(k)\top}(\hbbeta^{\cal A}+\balpha)\bigr)\Bigr\}\\
& +\lambda_2\|\balpha\|_1.
\end{align*}
\item[\bf Step 3.] (Target debiasing)
\begin{align*}
\hbgamma
=\mathop{\arg\min}_{\bgamma}
\frac{1}{n_0}\sum_{i=1}^{n_0}
\Bigl\{-y_i^{(0)}\bx_i^{(0)\top}(\hbbeta^{\cal A}+\hbalpha+\bgamma)
+\psi\bigl(\bx_i^{(0)\top}(\hbbeta^{\cal A}+\hbalpha+\bgamma)\bigr)\Bigr\}
+\lambda_3\|\bgamma\|_1.
\end{align*}
\item[\bf Output.] $\hbbeta^{(0)}_{TGLMC}
=\hbbeta^{\cal A}+\hbalpha+\hbgamma$.
\end{itemize}

When the detected partition is trivial (one cluster containing the
target), Step~2 reduces to Step~1 and Trans-GLMC collapses exactly to
Trans-GLM. The theoretical properties of $\hbbeta^{(0)}_{TGLMC}$ are
summarized in Section~\ref{sec:theory}.

\subsection{Cluster Structure Detection}\label{sec:distance}

In practice, neither the latent partition $\{{\cal C}_1,{\cal C}_2\}$ nor
the population coefficients $\{\bbeta^{(k)}\}$ are known. We estimate
the partition in two stages: a coefficient-based pairwise distance
matrix among facilities, followed by hierarchical density-based
clustering on that matrix.

Following the federated-learning idea of
\citet{li2024federated}, we use the target-only-style estimator
\eqref{eq:glm-classic} fitted facility-by-facility to plug
$\hbbeta^{(k)}$ into the population distance
$d_{ij}=\|\bbeta^{(i)}-\bbeta^{(j)}\|_1$. To stabilize the plug-in
against fold-level fluctuations, we use a cross-validated average:
\begin{itemize}
\item[\bf Step 1.] For each $k\in\{0,1,\ldots,K\}$, randomly split the
data from facility $k$ into $M$ equal-size folds
$(\bX^{(k)[r]},\by^{(k)[r]})$, $r=1,\ldots,M$.
\item[\bf Step 2.] For each $r$, fit \eqref{eq:glm-classic} on
$\cup_{m\ne r}(\bX^{(k)[m]},\by^{(k)[m]})$ to get per-fold estimator
$\hbbeta^{(k)[r]}$, and let
$\bbbeta^{(k)}=M^{-1}\sum_{r=1}^M\hbbeta^{(k)[r]}$.
\item[\bf Step 3.] Set
${\cal D}_{ij}=\bigl\|\bbbeta^{(i)}-\bbbeta^{(j)}\bigr\|_1$ for
$i,j\in\{0,1,\ldots,K\}$.
\end{itemize}
Theorem~\ref{thm:detect} in Section~\ref{sec:theory} shows that
$({\cal D}_{ij})$ preserves the order of $(d_{ij})$ with high
probability, so any cluster-detection rule based on pairwise distances
will recover the true partition under a minimum-gap separation
condition.

Given $({\cal D}_{ij})$, several pairwise-distance clustering
algorithms apply, including DBSCAN, HDBSCAN, and hierarchical
clustering. We use HDBSCAN because it is robust to outliers and does
not require a pre-specified number of clusters. A final merging step
absorbs cluster fragments that are still mutually informative for the
target.

\begin{itemize}
\item[\bf Step 1.] (Initial screening.) Apply the Trans-GLM data-driven
informative-source detection of Section~\ref{sec:glmtrans}
(detailed in Section~S.1 of the Supplement) to obtain
$\widehat{\cal A}$.
\item[\bf Step 2.] (Distance matrix.) Compute $({\cal D}_{ij})$ on
$\{0\}\cup\widehat{\cal A}$ as above.
\item[\bf Step 3.] (Clustering.) Run HDBSCAN on $({\cal D}_{ij})$ to
obtain a partition $\{\widehat{\cal A}_1,\ldots,\widehat{\cal A}_G\}$
of $\{0\}\cup\widehat{\cal A}$, with the target index $0$ assigned to
$\widehat{\cal A}_1$.
\item[\bf Step 4.] (Merging.) Pool all samples in $\widehat{\cal A}_1$
as the target and treat the pooled samples in each remaining cluster
as a single source; rerun the Trans-GLM detection procedure to merge
informative clusters with $\widehat{\cal A}_1$, yielding the final
partition $\{\widehat{\cal A}_1,\ldots,\widehat{\cal A}_{G'}\}$.
\end{itemize}
The final estimator is obtained by running Trans-GLMC of
Section~\ref{sec:glmtransc} with $\widehat{\cal A}_1$ playing the role
of ${\cal C}_1\cap\widehat{\cal A}$.

\section{Theoretical Results}\label{sec:theory}

This section states the two main theoretical results for the
cluster-structured procedure: an $\ell_2$-estimation error bound for
Trans-GLMC, and a consistency guarantee for the distance-based cluster
detection step. Theorem statements are restated, and proofs and
technical lemmas given, in Sections~S.2--S.3 of the Supplement.
Throughout, we write $s$ for the
sparsity of $\bbeta^{(0)}$, $n_0$ for the target sample size,
$n_{\cal A}=\sum_{k\in\hat{\cal A}}n_k$ for the total sample size
across detected informative sources, and $n_{\cal C}=
\sum_{k\in {\cal C}_1\cap\hat{\cal A}}n_k$ for the total sample size
within the target cluster.

\subsection{Assumptions}

We adopt the following standard assumptions on the GLM family and on
the covariate distribution. Detailed discussion is given in
Section~S.2 of the Supplement.

\begin{assumption}[Strict convexity of the cumulant]\label{assum:A1}
$\psi$ is twice continuously differentiable with $\psi''(x)>0$ for all
$x\in\R$.
\end{assumption}

\begin{assumption}[Sub-Gaussian covariates and well-conditioned
design]\label{assum:A2}
For every $k=0,1,\ldots,K$ and every $\ba\in\R^p$, $\ba^\top\bx_i^{(k)}$
is $\kappa_u\|\ba\|_2^2$-sub-Gaussian with zero mean, and
$\inf_k\lambda_{\min}(\bSigma^{(k)})\ge\kappa_l>0$.
\end{assumption}

\begin{assumption}[Bounded second derivative]\label{assum:A3}
At least one of the following holds: (i) $\|\psi''\|_\infty\le
\kappa_\psi<\infty$; or (ii) $\sup_k\|\bx^{(k)}\|_\infty\le U$ a.s.\
and $\sup_{|z|\le V}\psi''(\bx^{(k)\top}\bbeta^{(k)}+z)\le
\kappa_\psi<\infty$.
\end{assumption}

\begin{assumption}[Bounded weighted covariance shift]\label{assum:A4}
Let ${\cal A}\subseteq\{1,\ldots,K\}$ be the (population) informative set with $d=\max_{k\in{\cal A}}d_{k0}$, and let $\widetilde\bSigma$, $\widetilde\bSigma_k^{(k)}$ be the weighted-design matrices defined in equations~(S.2.1)--(S.2.2) of the Supplement. Then
$\sup_{k\in\{0\}\cup{\cal A}}\|\widetilde\bSigma^{-1}\widetilde\bSigma_k^{(k)}\|_1<\infty$.
\end{assumption}

Assumption~\ref{assum:A1} is satisfied by Gaussian, Poisson, and
binomial GLMs. Assumptions~\ref{assum:A2}--\ref{assum:A3} are
standard in high-dimensional regression \citep{negahban2009unified}
and are inherited from \citet{tian2023transfer}. Assumption~\ref
{assum:A4} limits the weighted covariance shift between target and
source designs \citep[cf.][Condition~4]{li2022transfer}.

\subsection{Estimation error bound}

Recall from Section~\ref{sec:setup} that, for a generic informative
set ${\cal A}\subseteq\{1,\ldots,K\}$, $d=\max_{k\in{\cal A}}d_{k0}$
measures its global informativeness and
$d_{\cal C}=\max_{k\in({\cal C}_1\cap{\cal A})\setminus\{0\}}d_{k0}$
measures the within-cluster informativeness in the target's cluster,
with $d_{\cal C}\le d$ by construction. The following theorem, proved
in Section~S.3 of the Supplement, gives a non-asymptotic $\ell_2$
error bound for the Trans-GLMC estimator $\hbbeta_{TGLMC}^{(0)}$ when
both ${\cal A}$ and the cluster ${\cal C}_1$ are known. Combined with
Theorem~\ref{thm:detect} below and the screening guarantee of
\citet{tian2023transfer} for $\hat{\cal A}$, the same bound holds with
high probability when ${\cal A}$ is replaced by the data-driven
$\hat{\cal A}$ and ${\cal C}_1$ by the estimated cluster.

\begin{theorem}[$\ell_2$-estimation error bound of
Trans-GLMC]\label{thm:rate}
Suppose Assumptions~\ref{assum:A1}--\ref{assum:A4} hold for the
informative set ${\cal A}$, with $d\lesssim U^{-1}V$ in
Assumption~\ref{assum:A3}. Assume the sample-size conditions
$d\ll\sqrt{n_0/\log p}$, $n_0\gtrsim \log p$,
$n_{\cal A}\gtrsim s\log p$, and $n_{\cal C}\gtrsim s\log p$, and
take penalty parameters $\lambda_1\asymp
\sqrt{\log p/(n_{\cal A}+n_0)}$,
$\lambda_2\asymp\sqrt{\log p/(n_{\cal C}+n_0)}$, and
$\lambda_3\asymp\sqrt{\log p/n_0}$. Then, with probability at
least $1-n_0^{-1}$,
\begin{align}\label{eq:tglmc-rate}
\bigl\|\hbbeta_{TGLMC}^{(0)}-\bbeta^{(0)}\bigr\|
\lesssim \, &
\Bigl(\tfrac{s\log p}{n_{\cal A}+n_0}\Bigr)^{1/2}
+ \Bigl\{\bigl[\bigl(\tfrac{\log p}{n_{\cal C}+n_0}\bigr)^{1/4}
d^{1/2}\bigr]\wedge d\Bigr\}
\\
\nonumber
\phantom{=} \, &
+ \Bigl\{\bigl[\bigl(\tfrac{\log p}{n_0}\bigr)^{1/4}
d_{\cal C}^{1/2}\bigr]\wedge d_{\cal C}\Bigr\}.
\end{align}
\end{theorem}

Two important consequences deserve discussion. First, under the same assumptions, the unclustered Trans-GLM estimator of
\citet{tian2023transfer} satisfies
$\|\hbbeta_{TGLM}^{(0)}-\bbeta^{(0)}\|\lesssim
(s\log p/(n_{\cal A}+n_0))^{1/2}+T$, where
$T:=\bigl[(\log p/n_0)^{1/4}d^{1/2}\bigr]\wedge d$
(Theorem~S.2.2 of the Supplement). The first (variance) term of
\eqref{eq:tglmc-rate} coincides with that of Trans-GLM. For the two
remaining (bias) terms, the second term of \eqref{eq:tglmc-rate} is at
most $T$ because $n_{\cal C}+n_0\ge n_0$, and the third term is at most
$T$ because $d_{\cal C}\le d$; hence the Trans-GLMC bound is no larger
than the Trans-GLM bound up to a constant factor. The improvement is
strict and substantial when a cluster structure concentrates mass
around the target: if $n_{\cal C}\asymp n_{\cal A}\gg n_0$ the second
term is of order $(\log p/n_{\cal A})^{1/4}d^{1/2}\ll T$, and if
$d_{\cal C}\ll d$ the third term is of order
$(\log p/n_0)^{1/4}d_{\cal C}^{1/2}\ll T$, so both bias terms fall
strictly below the Trans-GLM rate. Second, when there is no cluster structure, so that the detected cluster
containing the target satisfies $d_{\cal C}=d$ and
$n_{\cal C}=n_{\cal A}$, the second and third terms of
\eqref{eq:tglmc-rate} are both of order $T$, so the bound matches the
Trans-GLM rate up to a constant factor. In this sense Trans-GLMC is
never worse than Trans-GLM up to constants, and dominates it whenever a
non-trivial cluster exists.

\subsection{Consistency of cluster detection}

The second main result shows that the sample distance matrix
$({\cal D}_{ij})$ computed from facility-wise cross-validated
estimators of Section~\ref{sec:distance} preserves the ordering of the
population distances $(d_{ij})$, so that any cluster-detection rule
based on pairwise distances consistently recovers the true cluster
partition under a minimum-gap condition.

\begin{theorem}[Order preservation of estimated distances and cluster
detection]\label{thm:detect}
Suppose that for every facility $\ell\in\{0,1,\ldots,K\}$ and every
$m\in\{1,\ldots,M\}$, the per-fold facility-specific estimator
$\hbbeta^{(\ell)[m]}$ satisfies
$\P\bigl(\|\hbbeta^{(\ell)[m]}-\bbeta^{(\ell)}\|_1\le r\bigr)
\ge 1-c_1\exp(-c_2 p)$ for some constants $c_1,c_2>0$.
\begin{enumerate}
\item[(i)] \emph{(Order preservation.)} For any four indices
$i,j,k,\ell$ with $d_{ij}<d_{k\ell}-4r$,
\begin{align*}
\P\bigl({\cal D}_{ij}<{\cal D}_{k\ell}\bigr)
\ge 1-c_3\exp(-c_4 p).
\end{align*}
\item[(ii)] \emph{(Consistent cluster detection under a minimum-gap
condition.)} Suppose the true clusters ${\cal C}_1,{\cal C}_2$
satisfy the separation condition
\begin{align}\label{eq:gap}
\min_{i\in{\cal C}_g,\,j\in{\cal C}_{g'},\,g\ne g'} d_{ij}
> \max_{i,j\in{\cal C}_g} d_{ij} + 4r,
\end{align}
for every pair $g,g'$. Then the hierarchical density-based clustering
rule of Section~\ref{sec:distance} applied to $({\cal D}_{ij})$
recovers the true partition $\{{\cal C}_1,{\cal C}_2\}$ with
probability at least $1-c_5\exp(-c_6 p)$ for some constants
$c_5,c_6>0$.
\end{enumerate}
\end{theorem}

Part~(i) follows from the triangle inequality applied to the facility-specific
estimators; see Theorem~S.2.1 of the Supplement. Part~(ii) is a direct consequence of part~(i): under
the gap condition~\eqref{eq:gap}, all within-cluster sample distances
are strictly smaller than all between-cluster sample distances, which
is precisely the regime in which any single-linkage-type rule 
recovers the correct partition. The condition~\eqref{eq:gap}
tightens the informal separation requirement stated in
Section~\ref{sec:setup} by the concentration-induced slack $4r$,
which decays to zero as facility sample sizes grow.

\section{Simulation}\label{sec:simulation}

We evaluate Trans-GLMC against multiple competing methods over a range of cluster 
designs that vary the size of the target's cluster and
the within-cluster informativeness.

\subsection{Data-generating process}\label{sec:simu:dgp}

We generate one target ($k=0$) and $K=20$ source facilities, each with
$n_k=n=150$ observations and $p=150$ covariates (including an
intercept). For each $k$, the rows of $\bX^{(k)}$ are i.i.d.\
${\cal N}(0,\bSigma^{(k)})$ with $\bSigma^{(k)}=\bSigma+\bvar^{(k)}
\bvar^{(k)\top}$, $\bvar^{(k)}\sim {\cal N}(0_p,0.3^2 I_p)$, and
$\bSigma_{ij}=0.5^{|i-j|}$, so that designs differ across facilities
through the rank-one perturbation $\bvar^{(k)}$.

The facilities form two clusters indexed by $g\in\{1,2\}$, with
cluster~1 (the target's cluster) of size $|{\cal A}|$ varied over
$\{0,1,2,4,6,\ldots,20\}$ and cluster~2 of size $K-|{\cal A}|$. Two
distance parameters control between- and within-cluster informativeness:
\begin{itemize}
\item Set the population-target coefficient
$\bbeta^{(0)}=\bbeta=(0,\underbrace{1,\ldots,1}_{s=20},\underbrace{0,\ldots,0}_{p-1-s})$
and define the cluster centers $\bbeta+\balpha^{g}$ via
\[
\balpha^{1}=\mathbf{0},\qquad
\balpha^{2}=\bigl(0,\underbrace{-1,\ldots,-1}_{d_1/2},\underbrace{0,\ldots,0}_{s},\underbrace{1,\ldots,1}_{d_1/2},\underbrace{0,\ldots,0}_{p-1-s-d_1}\bigr),
\]
so that cluster~2's center is at $\ell_1$-distance $d_1$ from the target.
\item For each $k=1,\ldots,K$, draw a random subset $H_k\subset\{1,\ldots,p\}$ with $|H_k|=s$ and set $\bgamma^{k}_j=\xi_j\,\mathbf 1(j\in H_k)$ with $\xi_j\stackrel{\text{i.i.d.}}{\sim}\mathrm{Laplace}(0,d_2/s)$, so $\E\|\bgamma^{k}\|_1=d_2$.
\item Set $\bbeta^{(k)}=\bbeta+\balpha^{g(k)}+\bgamma^{k}$ for $k=1,\ldots,K$, where $g(k)=1$ if $k\in{\cal A}$ and $g(k)=2$ otherwise.
\end{itemize}
The target belongs to the cluster~1. In the notation of
Section~\ref{sec:setup}, $d_1$ is the between-cluster distance, while
$d_2$ governs the within-cluster informativeness $d_{\cal C}$ for
cluster~1 (since $\E\|\bbeta^{(k)}-\bbeta^{(0)}\|_1 = d_2$ for $k\in{\cal A}$). The binary outcome
$y_i^{(k)}$ is drawn from the logistic GLM in
\eqref{eq:glmmodel-tg}--\eqref{eq:glmmodel-sc}. We fix $d_1=10$ and
consider $d_2\in\{1,3\}$, corresponding to strong (concentrated) and
weak (diffuse) within-cluster signal, respectively. Each configuration
is replicated $100$ times.

\subsection{Competing methods and evaluation}\label{sec:simu:methods}

We compare six estimators, all $\ell_1$-penalized with cross-validated
tuning parameters:
\begin{enumerate}
\item[(1)] \emph{Target-GLM}: facility-specific $\ell_1$-penalized
GLM with target data.
\item[(2)] \emph{Trans-GLM}: data-driven Trans-GLM by \citet{tian2023transfer} (Section~\ref{sec:glmtrans}).
\item[(3)] \emph{Trans-GLM-Q}: Q-aggregation variant of
\citet{li2022transfer} adapted to GLMs (Section~S.1.2 of the Supplement).
\item[(4)] \emph{Trans-GLM-IDW}: inverse-distance-weighted-loss variant of Trans-GLM (Section~S.1.3 of the Supplement).
\item[(5)] \emph{Trans-GLM-SPH}: spherical-weighted-loss variant of Trans-GLM (Section~S.1.4 of the Supplement).
\item[(6)] \emph{Trans-GLMC}: the proposed cluster-structured estimator (Section~\ref{sec:glmtransc}), with the cluster partition recovered by the data-driven HDBSCAN procedure of Section~\ref{sec:distance}.
\end{enumerate}
Performance is summarized by the mean squared error
$\|\hbbeta^{(0)}-\bbeta^{(0)}\|^2$ averaged over the $100$
replications, with a normal-approximation confidence interval
$\bar x\pm 1.96\,\hat\sigma/\sqrt{100}$.

\subsection{Results}\label{sec:simu:results}

Figure~\ref{fig:tglmc-simu} summarizes the results; full tables, broken down by $|{\cal A}|$, are reported in Tables~S.1
and~S.2 of the Supplement.

\begin{figure}[H]
\centering
\includegraphics[width=\linewidth]{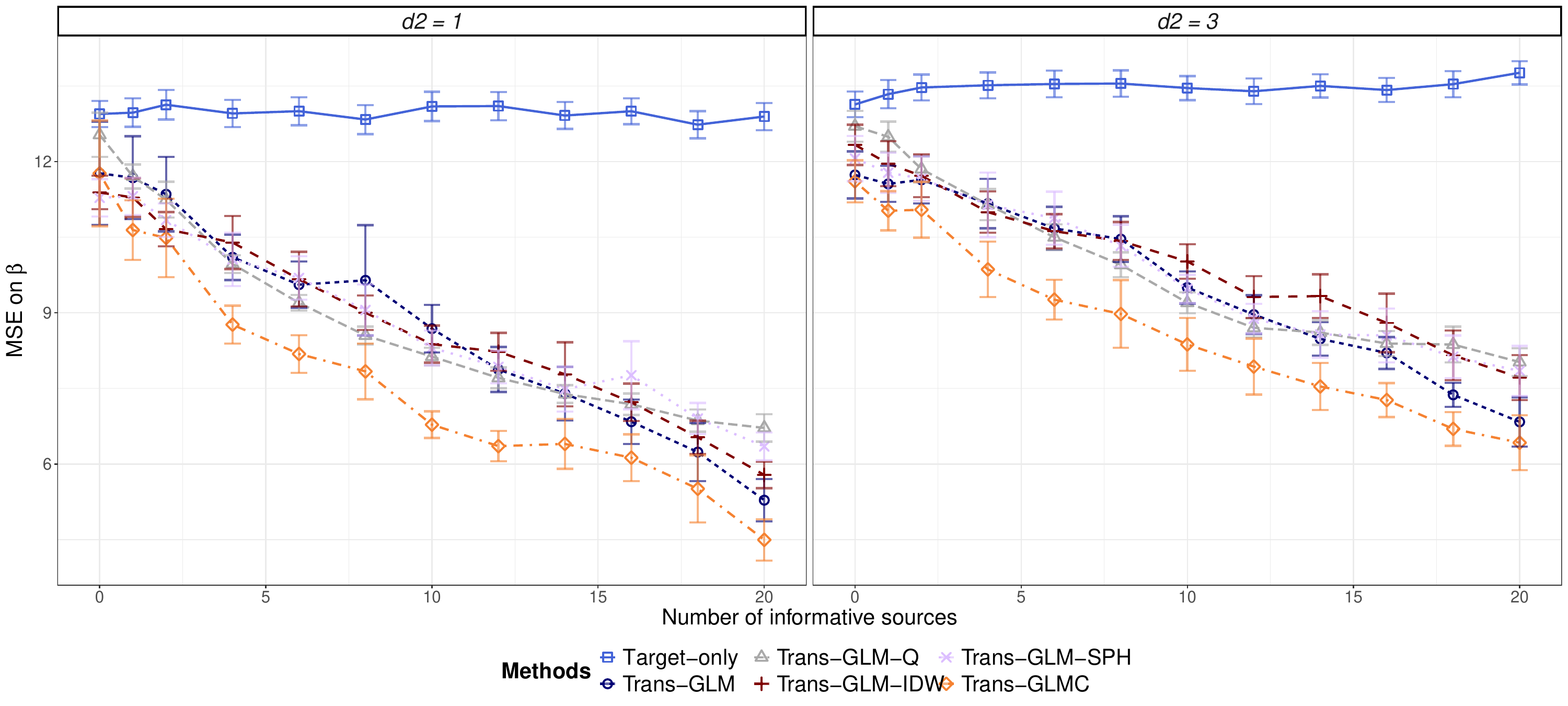}
\caption{Average MSE of $\hbbeta^{(0)}$ across 100 replications as a
function of $|{\cal A}|$, for strong ($d_2=1$, left) and weak
($d_2=3$, right) within-cluster informativeness. Error bars are
$\bar x\pm 1.96\,\hat\sigma/\sqrt{100}$.}
\label{fig:tglmc-simu}
\end{figure}

All five transfer-learning methods
improve substantially on the target-only estimator across the range of $|{\cal A}|$. 
Trans-GLMC delivers the lowest MSE in essentially every
configuration with a non-trivial cluster ($|{\cal A}|\ge 2$); the
margin over the unclustered Trans-GLM and the three weighted
variants is largest for moderate cluster sizes ($6\le|{\cal A}|\le
16$), where the cluster pattern is well-balanced and the
within-cluster information is plentiful. The gain from
Trans-GLMC is more pronounced in the strong-cluster regime
($d_2=1$) than in the weak-cluster regime ($d_2=3$), in agreement
with Theorem~\ref{thm:rate}: $d_2$ plays the role of the
within-cluster informativeness $d_{\cal C}$, and the second/third
terms of the bound~\eqref{eq:tglmc-rate} shrink as $d_{\cal C}$ does.

When $|{\cal A}|$ is very small ($\le 2$) the target's cluster
contributes too few samples to gain over Trans-GLM, while when
$|{\cal A}|$ is very large ($\ge 18$) the unclustered Trans-GLM is
already nearly optimal because almost all sources are informative.
In both extremes Trans-GLMC remains at least as accurate as
Trans-GLM, consistent with the no-worse-than guarantee discussed
after Theorem~\ref{thm:rate}.

\section{Facility-Specific Suicide Risk Modeling}\label{sec:realdata}

Suicide is among the leading causes of preventable death in the United
States, and the emergency department is a critical, and often the
only, point of contact for individuals at imminent risk
\citep{ahmedani2014health}. Predictive models built from electronic
health records can flag at-risk patients in real time, but their
clinical utility hinges on facility-specific calibration: the patient
mix, documentation practices, and base rates of self-harm vary
markedly across hospitals \citep{simon2018predicting}. 

Connecticut's
27 acute-care hospitals span the entire spectrum, from large academic
medical
centers 
to small community and rural
facilities 
(Section~\ref{sec:chime}; the full facility-level summary, with
demographic characteristics and transferability-community assignment,
is given in Table~S.3 of the Supplement). As such, a model trained on
one site rarely transfers cleanly to another. At the same time, some
facilities have too few suicide-attempt cases to reliably fit a
high-dimensional model on their own; for the smallest
hospitals, the Target-only AUROC sits near chance
(Table~\ref{app:tab:metric_overall}). This is precisely the regime in
which a clustered transfer learning approach should help, by borrowing
strength from facilities whose risk profiles are genuinely
similar to the target. The remainder of this section reports the
facility-specific predictive performance, the estimated
transferability network among the 27 hospitals, and the risk factors
identified by Trans-GLMC. Facility indices are used throughout the
main text to focus interpretation on transferability and prediction,
rather than on hospital-level comparisons.

\subsection{Study Design and Setup}

We adopt a retrospective follow-up design to assess one-year suicide
risk across the 27 Connecticut hospitals in the CHIME cohort
(Section~\ref{sec:chime}). Patients aged 18--64 years were recruited
during the calendar year 2016 and followed through December 31, 2017
for the occurrence of a suicide attempt, with controls defined as
patients who had no follow-up visit associated with a suicide attempt.
Information collected from historical visits between January 1, 2012
and December 31, 2016 was used to construct predictor variables.
Demographic variables comprised age (stratified into 18--24, 25--39,
40--54, 55--64), biological sex, race/ethnicity (Non-Hispanic White
versus other), and insurance type (Medicaid versus other). Following
established suicide-risk modeling pipelines \citep{su2020machine,
  xu2022improving, sacco2023target}, historical ICD-10 diagnosis codes
were first reduced to their 3-digit categories and screened against
the suicide-attempt outcome via Fisher's exact tests; codes with
BH-adjusted $p$-values below $0.05$ \citep{benjamini1995controlling}
were retained. The retained diagnosis-code features were combined with
the demographic variables to form a common predictor set, which was
applied to all 27 facilities and standardized before fitting.

We compared Trans-GLMC against the same five benchmarks used in
Section~\ref{sec:simulation}: Target-only, Trans-GLM, Trans-GLM-Q,
Trans-GLM-IDW, and Trans-GLM-SPH. 
For each of the 27 facilities (treated as the target in turn), we
performed an 80/20 random train/test split, repeated 10 times, and
reported mean AUROC, AUPRC, and sensitivity and PPV at 90\% and 95\%
specificity, together with the standard deviation across repetitions.

\subsection{Predictive Performance}

Table~\ref{app:tab:metric_overall} summarizes the cross-facility
distribution of each metric; additional results including complete facility-by-facility 
results are provided in Section~S.5 of the Supplement. Trans-GLMC attains the highest median
AUROC (79.58\%), AUPRC (6.73\%), sensitivity at 90\% specificity
(52.62\%), PPV at 90\% specificity (1.79\%, tied with Trans-GLM-SPH),
sensitivity at 95\% specificity (38.00\%), and PPV at 95\% specificity
(2.69\%). Across every metric the gain over Target-only is
substantial. 
Trans-GLMC also dominates Trans-GLM and the three weighted
variants on most summaries regarding the median across facilities. The
improvements are most pronounced at the lower end of the distribution.
At the smallest facility (facility~$27$), 
Target-only AUROC of $44.21\%$ rises to $75.93\%$ under Trans-GLMC,
with a roughly fourfold gain in AUPRC. At
facility~$6$, 
the site with the lowest observed suicide-attempt rate, Target-only
AUROC is $50.00\%$ (constant predictions across all $10$ splits,
SD~$=0$). The model extracts no signal from local data alone;
Trans-GLMC raises AUROC to $56.81\%$. Although this represents a
recovery of some signal where Target-only had none, $56.81\%$ remains
below the threshold of clinical viability, and we therefore view
Trans-GLMC's role at the most data-poor sites as stabilizing otherwise
unreliable estimates rather than producing deployable models.

\begin{table}[H]
  \scriptsize\centering \setlength{\tabcolsep}{4pt}
  \caption{Performance summary across methods. For each facility we
    first averaged each metric over 10 replicates; we then summarized
    across the 27 facilities by reporting the median, IQR, and
    minimum--maximum (Min--Max).}
  \label{app:tab:metric_overall}
  \begin{tabular}{l r r r r | r r r r}
    \toprule
    Method & Median & IQR & Min & Max & Median & IQR & Min & Max\\
    \midrule
    \addlinespace[0.3em]
           & \multicolumn{4}{l}{\textbf{AUROC, \%}} & \multicolumn{4}{l}{\textbf{AUPRC, \%}}\\
    Target-only   & 75.25 & 70.10 - 78.13 & 44.21 & 85.05 & 4.03 & 2.46 - 5.63 & 0.07 & 9.09\\
    Trans-GLM     & 78.19 & 76.16 - 80.10 & 59.78 & 85.19 & 6.67 & 4.78 - 8.27 & 2.27 & 10.87\\
    Trans-GLM-Q   & 79.48 & 77.11 - 80.72 & 59.55 & 85.18 & 6.33 & 3.40 - 7.65 & 0.91 & 13.18\\
    Trans-GLM-IDW & 77.73 & 75.33 - 79.95 & 59.31 & 85.73 & 6.28 & 4.79 - 8.08 & 0.64 & 11.91\\
    Trans-GLM-SPH & 77.08 & 72.52 - 80.03 & 54.55 & 85.68 & 5.88 & 4.25 - 7.93 & 0.72 & 11.36\\
    Trans-GLMC    & 79.58 & 76.99 - 81.17 & 56.81 & 86.41 & 6.73 & 5.27 - 8.99 & 1.21 & 11.88\\
    \midrule
    \addlinespace[0.3em]
           & \multicolumn{4}{l}{\textbf{Sensitivity, \% (90\% specificity)}} & \multicolumn{4}{l}{\textbf{PPV, \% (90\% specificity)}}\\
    Target-only   & 45.00 & 38.17 - 47.32 &  0.00 & 65.55 & 1.42 & 1.05 - 2.22 & 0.00 & 4.87\\
    Trans-GLM     & 50.50 & 46.97 - 54.58 & 30.00 & 66.16 & 1.76 & 1.16 - 2.62 & 0.23 & 4.93\\
    Trans-GLM-Q   & 51.67 & 47.94 - 55.00 & 25.00 & 66.47 & 1.73 & 1.18 - 2.58 & 0.23 & 5.21\\
    Trans-GLM-IDW & 51.19 & 47.61 - 54.31 & 10.00 & 66.48 & 1.74 & 1.20 - 2.60 & 0.19 & 5.27\\
    Trans-GLM-SPH & 49.52 & 45.36 - 52.98 & 20.00 & 66.48 & 1.79 & 1.19 - 2.52 & 0.17 & 5.30\\
    Trans-GLMC    & 52.62 & 48.72 - 57.05 & 30.00 & 66.02 & 1.79 & 1.18 - 2.63 & 0.23 & 5.17\\
    \midrule
    \addlinespace[0.3em]
           & \multicolumn{4}{l}{\textbf{Sensitivity, \% (95\% specificity)}} & \multicolumn{4}{l}{\textbf{PPV, \% (95\% specificity)}}\\
    Target-only   & 32.38 & 25.86 - 37.07 &  0.00 & 51.38 & 2.09 & 1.56 - 3.22 & 0.00 & 6.73\\
    Trans-GLM     & 35.00 & 33.20 - 42.76 & 20.00 & 52.28 & 2.49 & 1.69 - 3.82 & 0.34 & 6.73\\
    Trans-GLM-Q   & 36.39 & 34.24 - 41.89 & 15.00 & 55.71 & 2.46 & 1.82 - 3.73 & 0.34 & 6.70\\
    Trans-GLM-IDW & 35.73 & 33.86 - 41.67 & 10.00 & 53.24 & 2.64 & 1.89 - 3.88 & 0.34 & 6.65\\
    Trans-GLM-SPH & 36.00 & 32.26 - 43.78 & 15.00 & 53.24 & 2.57 & 1.82 - 3.82 & 0.23 & 6.71\\
    Trans-GLMC    & 38.00 & 34.88 - 43.63 & 20.00 & 53.54 & 2.69 & 1.82 - 3.99 & 0.34 & 6.59\\
    \bottomrule
  \end{tabular}
\end{table}

Figure~\ref{app:fig:metric_diff} compares Trans-GLMC against Trans-GLM and against Target-only across facilities for AUROC and sensitivity at $90\%$ specificity. Trans-GLMC improved AUROC at $25$ of $27$ facilities over Target-only and at $21$ of $27$ over Trans-GLM, with the largest gains concentrated at data-poor sites. At facility $18$ ($29$ attempts; Table S.4), 
for instance, AUROC rises from $67.39\%$ (Target-only) to $75.90\%$ (Trans-GLM) and to $80.49\%$ (Trans-GLMC), and sensitivity at $90\%$ specificity rises in step from $23.33\%$ to $48.33\%$ and to $60.00\%$, so Trans-GLMC improves over both baselines on both metrics. As expected, not all metrics improve simultaneously at every facility: at $7$ facilities an AUROC gain over Trans-GLM is not accompanied by a sensitivity gain. At facility $27$ ($15$ attempts), 
for example, AUROC rises from $44.21\%$ to $71.95\%$ and to $75.93\%$ across the three methods, but Trans-GLMC's sensitivity at $90\%$ specificity is slightly below Trans-GLM's. This reflects the intrinsic difficulty in dealing with rare outcomes. Overall, however, Trans-GLMC retains an improvement over Trans-GLM at most facilities. The results support our transfer learning strategy of adaptively selecting informative source facilities.

\begin{figure}[H]
  \centering
  \begin{subfigure}[t]{0.48\linewidth}
    \centering
    \includegraphics[width=\linewidth]{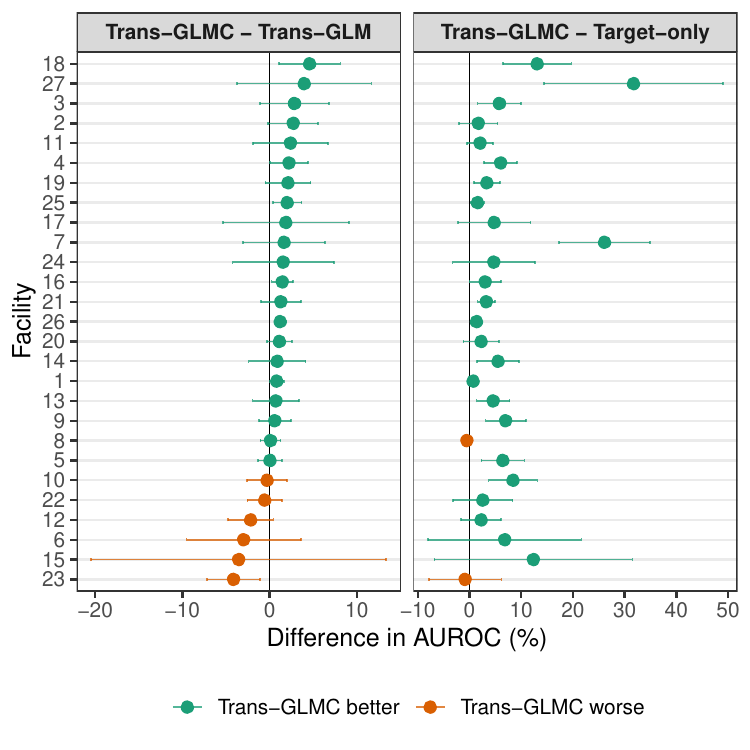}
    \caption{AUROC.}
    \label{app:fig:metric_diff_auroc}
  \end{subfigure}
  \begin{subfigure}[t]{0.48\linewidth}
    \centering
    \includegraphics[width=\linewidth]{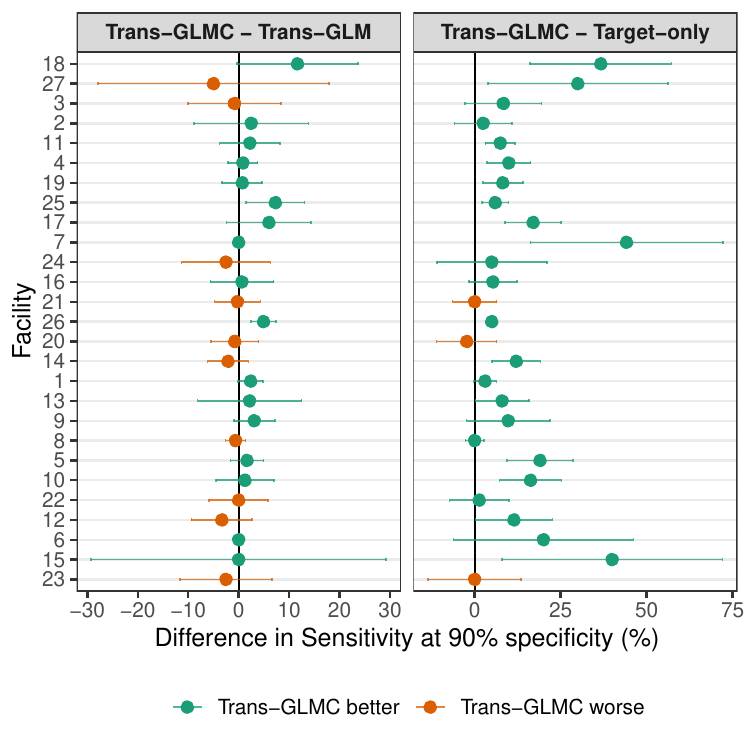}
    \caption{Sensitivity at 90\% specificity.}
    \label{app:fig:metric_diff_sens90}
  \end{subfigure}
    \caption{Facility-specific differences for Trans-GLMC relative to Trans-GLM and to the Target-only baseline across CHIME facilities, shown for (a)~AUROC, (b)~sensitivity at 90\% specificity. Facilities are ordered by the mean AUROC difference between Trans-GLMC and Trans-GLM. Points show mean differences across 10 splits; error bars are normal-approximation $95\%$ confidence intervals.}
  \label{app:fig:metric_diff}
\end{figure}

\subsubsection{Transferability}

To characterize how information flows across hospitals, we construct a directed transferability network from the Trans-GLMC fits. For each ordered pair $(a,b)$ we record an edge $a\!\rightarrow\! b$ whenever facility $b$ is selected as an informative source when facility $a$ is the target. Directed edges are then collapsed into undirected ones with weight 1 (one-way transfer) or weight 2 (mutual transfer), and community structure in the resulting weighted network is detected by greedy modularity maximization
\citep{clauset2004finding}.

Figure~\ref{app:fig:transferability} displays the resulting
transferability communities: panel~(a) shows the pairwise
transferability states with facilities reordered by community
assignment, and panel~(b) maps the communities onto Connecticut.
Table~S.3 in the Supplement lists each facility's community assignment
alongside its demographic characteristics. Two clear communities
emerge, broadly aligned with geography and observed suicide-attempt
incidence.

Community~1 ($14$ facilities) covers the Hartford corridor and central, northern, and eastern Connecticut, with suicide-attempt rates ranging from $0.21\%$ to $1.07\%$. Community~2 ($13$ facilities) is concentrated in the southwestern coastal corridor (Fairfield County) and extends to several mid-state and rural facilities, with uniformly lower rates ranging from $0.07\%$ to $0.39\%$. This stratification by incidence is consistent with regional heterogeneity in suicide risk across Connecticut, where socioeconomic and demographic profiles correlated with geography are established risk correlates \citep{franklin2017risk, harris1997suicide}. Notably, the recovered community structure does not reduce to hospital system membership: the state's two largest health systems 
each have facilities in both communities. This indicates that the learned transferability structure captures information about shared suicide-risk profiles beyond administrative groupings, supporting the value of data-driven clustering over pre-specified partitions.

\begin{figure}[H]
  \centering
  \begin{subfigure}[t]{0.4\linewidth}
    \centering
    \includegraphics[width=\linewidth]{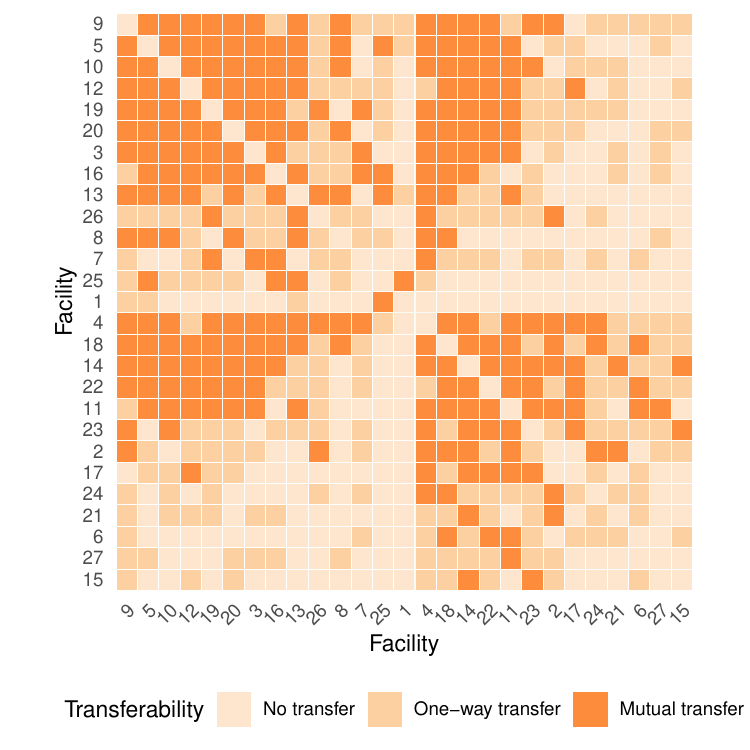}
    \caption{Pairwise transferability}
    \label{app:fig:transfer_heatmap}
  \end{subfigure}
  \hfill
  \begin{subfigure}[t]{0.58\linewidth}
    \centering
    \includegraphics[width=\linewidth]{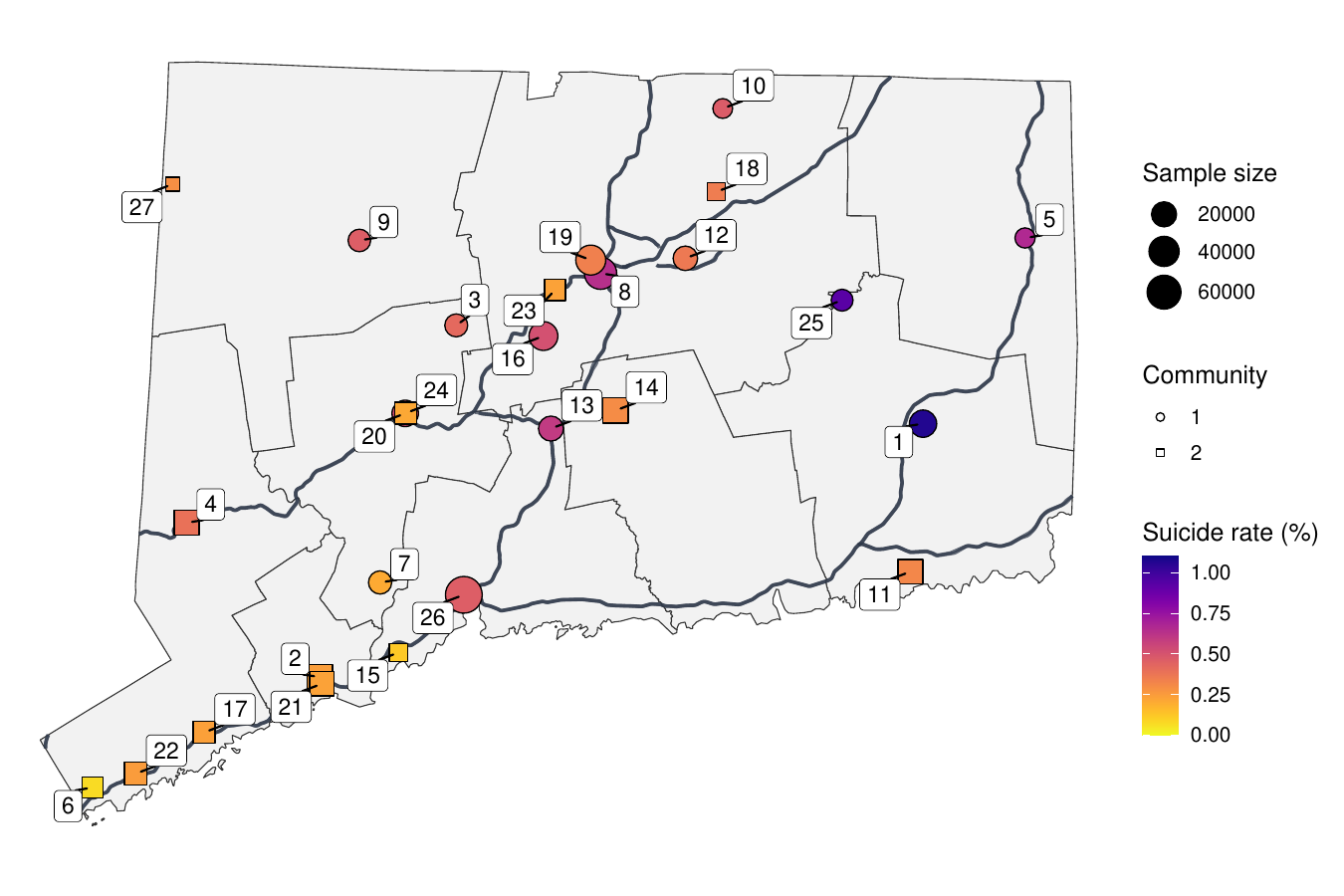}
    \caption{Facility map}
    \label{app:fig:map}
  \end{subfigure}
  \caption{Transferability communities in the CHIME study. Panel~(a)
    shows pairwise transferability states after reordering facilities
    by detected community. Panel~(b) shows the map of the facilities across
    Connecticut; node size reflects facility sample size, opacity
    reflects the observed suicide-attempt rate, and color indicates the
    transferability community detected by greedy modularity
    maximization.}
  \label{app:fig:transferability}
\end{figure}

\subsubsection{Risk Factors}
To gain clinical insight into the predictors driving suicide risk, we
examine the coefficients from the Trans-GLMC models fitted across all
27 facilities, each treated as the target in turn, with no splits. The
predictors are binary indicators of a documented prior suicide attempt
or self-harm event (SA) and 3-digit ICD-10 diagnosis codes. For each
predictor, we report the average coefficient across all 27 facilities
and the sign breakdown $(n_+, n_-)$, denoting how many facilities
select it with a positive versus negative sign. Predictors are ranked
by the magnitude of the average coefficient, and the top 15 risk
factors are reported.

\begin{table}[H]
  \centering
  \small
  \setlength{\tabcolsep}{3pt}
  \caption{Top 15 risk factors for suicide attempt identified by
    Trans-GLMC, ranked by mean coefficient across all 27 CHIME
    facilities. The Coefficient column reports the mean coefficient
    across all 27 facilities, treating unselected predictors as zero;
    parentheses show the minimum and maximum among nonzero
    coefficients. The $(n_+, n_-)$ column reports the number of
    facilities selecting a predictor with a positive and negative
    coefficient, respectively.}
  \label{tab:risk_factors}
  \begin{tabular}{lp{0.44\linewidth}rr}
    \toprule
    ICD-10 & Description & Coefficient (Min, Max) & $(n_+, n_-)$\\
    \midrule
    SA & Suicide attempt / self-harm indicator used in this project & 0.83 (0.40, 1.64) & (27, 0)\\
    F31 & Bipolar disorder & 0.76 (0.58, 0.87) & (27, 0)\\
    F32 & Depressive episode & 0.71 (0.46, 0.92) & (27, 0)\\
    F41 & Other anxiety disorders & 0.66 (-0.07, 0.94) & (26, 1)\\
    F10 & Mental and behavioral disorders due to use of alcohol & 0.61 (0.49, 0.85) & (27, 0)\\
    F60 & Specific personality disorders & 0.45 (0.18, 1.62) & (24, 0)\\
    F17 & Mental and behavioral disorders due to use of tobacco & 0.42 (0.09, 0.55) & (27, 0)\\
    R45 & Symptoms and signs involving emotional state & 0.35 (0.06, 0.82) & (23, 0)\\
    F43 & Reaction to severe stress, and adjustment disorders & 0.29 (0.04, 0.52) & (25, 0)\\
    F33 & Recurrent depressive disorder & 0.27 (0.07, 1.00) & (22, 0)\\
    F25 & Schizoaffective disorders & 0.18 (0.03, 1.29) & (18, 0)\\
    Z87 & Personal history of other diseases and conditions & 0.08 (0.02, 0.52) & (23, 0)\\
    Y92 & Place of occurrence of the external cause & 0.07 (-0.17, 0.42) & (11, 1)\\
    Z79 & Long term current drug therapy & 0.06 (0.01, 0.33) & (21, 0)\\
    M54 & Dorsalgia & 0.06 (0.01, 0.42) & (16, 0)\\
    \bottomrule
  \end{tabular}
\end{table}

Table~\ref{tab:risk_factors} summarizes the top risk factors
identified by the Trans-GLMC fits. The pattern is clinically coherent
and consistent with the wider suicide-prediction literature.  SA
carries the largest average coefficient ($+0.83$) and is selected with
a positive sign at \emph{every} one of the 27 facilities, mirroring
the near-universal finding that past suicidal behavior is the single
strongest short-term predictor of subsequent attempt and death by
suicide \citep{ribeiro2016self,franklin2017risk}. The next tier of
risk factors -- bipolar disorder (F31), depressive episode (F32),
anxiety disorders (F41), recurrent depressive disorder (F33),
alcohol-use disorder (F10), tobacco-use disorder (F17), personality
disorders (F60), reaction to severe stress (F43), and emotional-state
symptoms (R45) -- are likewise selected with a positive sign in 22--27
facilities and reproduce the well-documented role of mood, anxiety,
substance-use, and stress-related comorbidities in elevating suicide
risk \citep{harris1997suicide,nock2008suicide,bertolote2004suicide}.

These effects emerge simultaneously across hospitals as diverse as
small rural community sites 
and large urban academic medical
centers, 
serving patient populations of varying demographic and
socioeconomic status. 
Given this diversity, the consistency of these predictors highlights
the value of clustered transfer learning: the model lets each target
hospital sharpen its own coefficients while pooling evidence on
predictors that are stable across the system.

\begin{figure}[H]
  \centering
  \includegraphics[width=\linewidth]{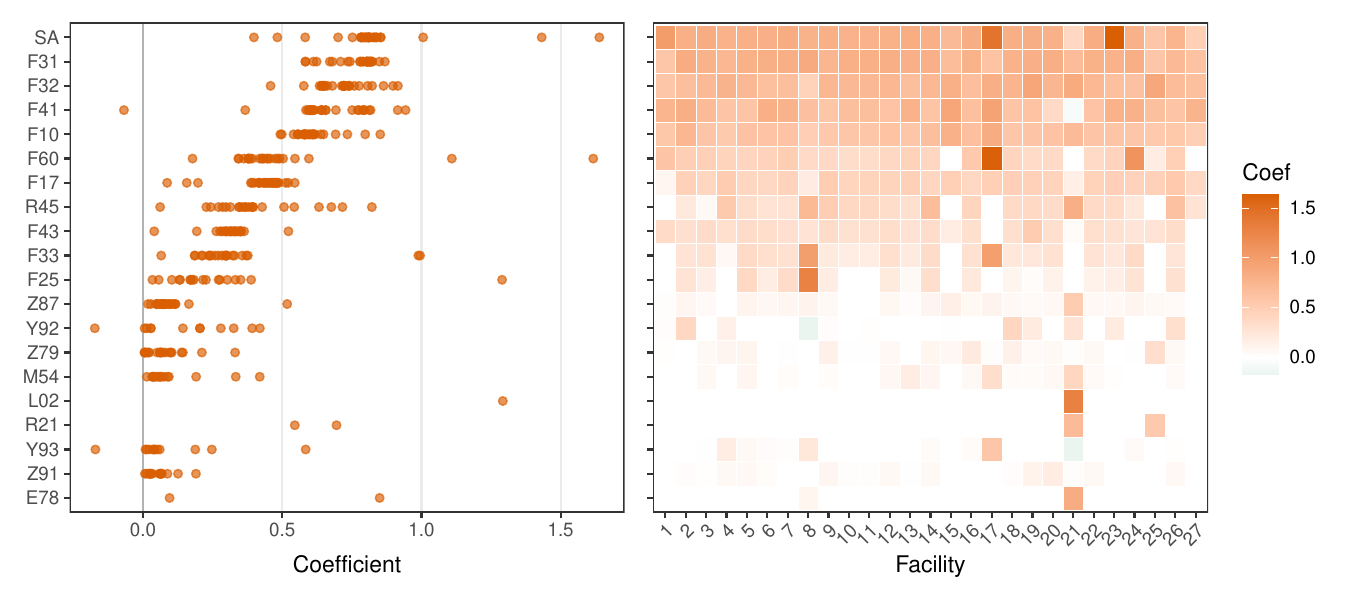}
  \caption{Trans-GLMC coefficients for the top 20 risk factors across
    27 facilities.}
  \label{fig:coef_beeswarm_heatmap}
\end{figure}

Figure~\ref{fig:coef_beeswarm_heatmap} confirms these findings: the
leading risk factors are tightly concentrated across facilities and
form densely red rows across all 27 columns of the heatmap, whereas
lower-ranked predictors become increasingly sparse and are limited to
isolated cells in only a few columns. Heterogeneity is largely
confined to effect magnitudes rather than signs, with one notable
exception. Facility 21 
exhibits a distinct coefficient pattern and receives transferable
information from only two source facilities, suggesting limited
similarity with the remaining facilities. As a result, information
borrowing is substantially reduced, leading to more facility-specific
estimates, including the negative effect for anxiety disorders and
zero effects for several leading predictors.

\section{Discussion}\label{sec:discussion}

We have developed Trans-GLMC, a cluster-structured transfer-learning
approach for high-dimensional GLMs with heterogeneous source
populations. The method uses coefficient-based distances to detect
latent source structure, then combines global fusion, within-cluster
refinement, and target debiasing to improve estimation at a target
population. The theory shows that this additional structure can
strictly improve the estimation rate when a target-relevant cluster is
present, while matching the unclustered transfer rate up to constants
when no meaningful cluster structure is available.

The CHIME suicide-risk study illustrates the practical value of this
perspective. Facility-specific models are difficult to fit because
suicide attempts are rare within individual hospitals, yet hospitals
differ enough that indiscriminate pooling is not ideal. Trans-GLMC
addresses this tension by learning which facilities are mutually
transferable and by borrowing most strongly within the target's
estimated community. In both simulations and the CHIME analysis, this
strategy improves prediction for many facilities and yields clinically
coherent risk-factor patterns.

Several extensions are worth pursuing. One direction is multi-task
learning with clustered populations, where the goal is to improve all
datasets simultaneously rather than one target at a time
\citep{suk2014clustering,suresh2018learning}. Another is a federated
implementation that preserves the cluster-aware borrowing strategy
while reducing the need to share raw patient-level data, in the spirit
of privacy-preserving transfer-learning methods such as
\citet{li2023targeting}. Such extensions would broaden the use of
cluster-structured transfer learning in real-world healthcare systems
where heterogeneity, rarity, and privacy constraints arise together.

\section*{Supplementary Materials}
\addcontentsline{toc}{section}{Supplementary Materials}

\setcounter{section}{0}
\setcounter{subsection}{0}
\setcounter{subsubsection}{0}
\setcounter{equation}{0}
\setcounter{table}{0}
\setcounter{figure}{0}
\setcounter{axiom}{0}
\setcounter{theorem}{0}
\setcounter{corollary}{0}
\setcounter{prop}{0}
\setcounter{assumption}{0}
\renewcommand{\thesection}{S.\arabic{section}}
\renewcommand{\thesubsection}{\thesection.\arabic{subsection}}
\renewcommand{\thesubsubsection}{\thesubsection.\arabic{subsubsection}}
\renewcommand{\thetable}{S.\arabic{table}}
\renewcommand{\thefigure}{S.\arabic{figure}}
\renewcommand{\theHsection}{supp.\arabic{section}}
\renewcommand{\theHsubsection}{supp.\arabic{section}.\arabic{subsection}}
\renewcommand{\theHsubsubsection}{supp.\arabic{section}.\arabic{subsection}.\arabic{subsubsection}}
\renewcommand{\theHtable}{supp.\arabic{table}}
\renewcommand{\theHfigure}{supp.\arabic{figure}}
\renewcommand{\theHequation}{supp.\arabic{section}.\arabic{equation}}
\numberwithin{equation}{section}

\section{Algorithmic Details for Trans-GLM and Its
  Variants}\label{sec:glmtransv}

This section contains two pieces of algorithmic content used in the
main paper but deferred for space. Section~\ref{sec:glmdetect} gives
the data-driven informative-source detection procedure for Trans-GLM
that is referenced in Sections~3.2 and~3.4 of the main
paper. Section~\ref{sec:glmqagg} onwards introduces three competing
weighted transfer-learning variants (Q-aggregation,
inverse-distance-weighted loss, and spherical-weighted loss) used as
benchmarks in the simulation and CHIME studies.

\subsection{Data-driven informative-source detection for
  Trans-GLM}\label{sec:glmdetect}

\citet{tian2023transfer} proposed a cross-validation rule that
compares, for each candidate source $k$, the fold-averaged loss
$\widehat Q_0^{(k)}$ obtained when source $k$ is fused into the target
estimator against the target-only loss $\widehat Q_0^{(0)}$, returning
$\widehat{\cal A}=\{k\ne 0:\widehat Q_0^{(k)}-\widehat Q_0^{(0)} \le
C_0(\widehat\sigma\vee 0.01)\}$ for a user-specified strictness
constant $C_0$. The constant $C_0$ is difficult to calibrate in
practice; we adopt the following data-driven variant as the
informative-source detector for Trans-GLM and as the screening step in
the cluster-detection procedure of the main paper.

\begin{itemize}
\item[\textbf{Step 1.}] Split the target data into three folds
  $(\bX^{(0)[r]},\by^{(0)[r]})$, $r=1,2,3$.
\item[\textbf{Step 2.}] For each $r$, compute the target-only
  estimator $\hbbeta^{(0)[r]}$ on
  $\cup_{i\ne r}(\bX^{(0)[i]},\by^{(0)[i]})$ and, for each $k$, the
  single-source-augmented estimator $\hbbeta^{(0,k)[r]}$ by fusing the
  held-in target folds with source $k$ via the Trans-GLM transferring
  step described in Section~3.2 of the main paper.
\item[\textbf{Step 3.}] Compute the averaged cross-validated losses
  $\widehat Q_0^{(k)}$ and $\widehat Q_0^{(0)}$. Initialize
  $\widehat{\cal A}=\{k\ne 0:\widehat Q_0^{(k)}\le\widehat
  Q_0^{(0)}\}$ and order the remaining indices by ascending
  $\widehat Q_0^{(k)}$.
\item[\textbf{Step 4.}] Greedily augment $\widehat{\cal A}$ by the
  next candidate index, recompute the fold-averaged fused loss, and
  stop as soon as adding the candidate no longer improves the loss
  relative to the target-only value.
\end{itemize}

The output $\widehat{\cal A}$ is used as the unclustered Trans-GLM
informative set in the simulations and as the input to the
cluster-detection step in Section~3.4 of the main paper.

\subsection{Transfer Learning Generalized Linear Models with
  Q-aggregation}\label{sec:glmqagg}

Although the modified data-driven Trans-GLM mentioned in the main
paper provides a mechanism for identifying informative sources, it
assigns equal weight to each selected source without differentiation
and may perform poorly when any non-informative sources are included. To
make the procedure more flexible and adaptive, we propose an extension
inspired by the work of \citet{li2022transfer}, where Q-aggregation is
used to assign adaptive weights to transfer learning estimators
generated from a progressively expanding candidate informative set.

With a half-and-half split of the target samples into $\cal I$ and
${\cal I}^c$, the loss function is defined as:
\begin{align*}
  \widehat Q\left( {{{\cal I}^c},\bbeta } \right) =  - \frac{2}{{{n_0}}}\sum\limits_{i \in {{\cal I}^c}} {\log \rho \left( {y_i^{\left( 0 \right)}} \right)}  - \frac{2}{{{n_0}}}{\left( {\by_{{{\cal I}^c}}^{\left( 0 \right)}} \right)^ \top }\bX_{{{\cal I}^c}}^{\left( 0 \right)}\bbeta  + \frac{2}{{{n_0}}}\sum\limits_{i \in {{\cal I}^c}} {\psi \left( {{\bbeta ^ \top }{\bx}_i^{\left( 0 \right)}} \right)}.
\end{align*}
The method starts by constructing inclusive candidate sets from the
data-driven informative-source detection procedure. Let
$\widehat {\cal A}$ be the detected informative set
from Trans-GLM and assume $| {\widehat {\cal A}} | = L$. We can then
construct the expanding candidate informative set
$\{ {{{\widehat G}_0},{{\widehat G}_1}, \ldots ,{ {\widehat G}_L}}\}$
where ${{\widehat G}_0} = \emptyset$,
${\widehat G}_L = \widehat {\cal A}$, and
\begin{align*} {{\widehat G}_l} = \left\{
  \begin{array}{*{20}{c}}
    {k \in \widehat {\cal A}:\widehat Q_0^{\left( k \right)}}&{is}&{among}&l&{smallest}
  \end{array}
  \right\},
\end{align*}
where $\widehat Q_0^{\left( k \right)}$ is defined as in the
data-driven Trans-GLM screening step. Then, with these constructed
candidate informative sets, let the output from the oracle Trans-GLM
algorithm applied to sample $\cal I$ with ${ {\widehat G}_l}$ as
$\hbbeta ( {{{\widehat G}_l}} )$ for all $0 \le l \le L$.  Then the
Q-aggregation algorithm assumes the final output in a weighted format
${\hbbeta^{\btheta} } = \sum\nolimits_{l = 0}^{L} { {\theta _l}\hbbeta
  ( {{{\widehat G}_l}} )}$, and the weighting parameter $\btheta$,
which is defined on $L+1$-dimension space
${\Lambda ^{L + 1}} = \{ {\bv \in {\R^{L + 1}}:{v_l} \ge
  0,\sum\nolimits_{l = 0}^L {{v_l} = 1} } \}$, is trained by
\begin{align}\label{Qaggloss}
  \hat \btheta  = \mathop {\arg \min }\limits_{{\Lambda ^{L + 1}}} \left\{ {\widehat Q\left( {{{\cal I}^c},\sum\limits_{l = 0}^L {{\theta _l}\hbbeta \left( {{{\widehat G}_l}} \right)} } \right) + \sum\limits_{l = 0}^L {{\theta _l}\widehat Q\left( {{{\cal I}^c},\hbbeta \left( {{{\widehat G}_l}} \right)} \right)}  + \frac{{{\lambda _\theta }}}{{{n_0}}}\sum\limits_{l = 0}^L {{\theta _l}\log {\theta _l}} } \right\},
\end{align}
where $\lambda_{\btheta}$ is a tuning parameter determined by
minimizing~\eqref{Qaggloss} without the penalty term through
cross-validation.

\subsection{Inverse-Distance-Weighted Transfer GLM}

With the distance ${\cal D}$ defined in Section~3.4 of the main
paper, it is natural to
consider weighting each source in the fusion step based on this
distance, giving more weight to sources that are closer to the
target. We start with the conventional inverse distance weight
proposed by~\citet{shepard1968two}, which has recently been extended
to weighted loss functions in areas such as active learning~\citep
{bemporad2023active} and ``weighted attention'' in deep
learning~\citep{mccarter2023inverse}. Similar to
\citet{shepard1968two}, the weight for $k$th dataset
($k \in \{0\} \cup {\cal A}$) is assigned as
${w_k^{IDW}}(q) = {s_k} (q)/\sum\nolimits_{i\in\{0\}\cup{\cal A}} {{s_i}(q)}
$ where $s_i(q)$ is the inverse of distance between the target and the
$i$th source such that ${s_i}(q) = {\cal D}_ {0i}^{ - q}$ and the
power $q$ is a tuning parameter. For ${\cal D}_{00}$, we define
${{\cal D}_{00}} = {\max _{i,j}}{\| {{\hbbeta^{( 0 )[ i ]}} -
    {\hbbeta^{( 0 )[ j ]}}} \|_1}$.

Using these defined weights, the proposed Trans-GLM with loss weighted
by inverse distance weight modifies the loss function in the
transferring step used by Trans-GLM in the main paper:

\noindent {\bf Step 1* (Transferring step for Trans-GLM-IDW):} Compute
\begin{align}\label{eq:glmfuseidw}
  {\hbbeta^{\cal A}} = \, & \mathop {\arg \min }\limits_{\bbeta  \in {\R^p}} \frac{1}{{{n_{\cal A}} + {n_0}}}\sum\limits_{k \in \left\{ 0 \right\} \cup {\cal A}} {w_k^{IDW}\left( q\right) \left[ -{{{\left( {{\by^{\left( k \right)}}} \right)}^\top}{\bX^{\left( k \right)}}\bbeta  + \sum\limits_{i = 1}^{{n_k}} {\psi \left( {{\bbeta ^\top}\bx_i^{\left( k \right)}} \right)} } \right]} \nonumber \\
  \phantom{=} \, & + {\lambda _1}{\left\| \bbeta  \right\|_1},
\end{align} 
where ${n_{\cal A}}$ is the total sample size of sources in $\cal A$,
$q$ and ${\lambda _1}$ are tuned by cross-validation. The
informative set $\cal A$ is estimated using our proposed data-driven
procedure.

\subsection{Transfer Learning Generalized Linear Models with Spherical
  Weighted Loss}

Similar to the loss weighting idea using inverse distance weight, we
can also assign weights based on the spherical correlation \citep
{cressie2015statistics} between the target and specific sources. This
method, a case of the semi-variogram models, provides a correlation
structure based on computed distance. In our case, the weight for
$k$th dataset ($k \in \{0\} \cup {\cal A}$) is assigned as
${w_k^{SPH}}(r) = {s_k}(r)/\sum\nolimits_{i\in\{0\}\cup{\cal A}} {{s_i}(r)} $
where ${s_i}(r) = 1 - 1.5({{\cal D}_{0i}}/r) + 0.5({{\cal D}_{0i}}/r)^3$ with $s_i(r)$ truncated at zero when ${\cal D}_{0i}>r$,
and the range parameter $r$ is the tuning parameter.

Using the defined weights, the proposed Trans-GLM with loss weighted
by spherical correlation modifies the loss function in the
transferring step used by Trans-GLM in the main paper:

\noindent {\bf Step 1** (Transferring step for Trans-GLM-SPH):}
Compute
\begin{align}\label{eq:glmfusesph}
  {\hbbeta^{\cal A}} = \, & \mathop {\arg \min }\limits_{\bbeta  \in {\R^p}} \frac{1}{{{n_{\cal A}} + {n_0}}}\sum\limits_{k \in \left\{ 0 \right\} \cup {\cal A}} {w_k^{SPH}\left( r\right) \left[ -{{{\left( {{\by^{\left( k \right)}}} \right)}^\top}{\bX^{\left( k \right)}}\bbeta  + \sum\limits_{i = 1}^{{n_k}} {\psi \left( {{\bbeta ^\top}\bx_i^{\left( k \right)}} \right)} } \right]} \nonumber \\
  \phantom{=} \, & + {\lambda _1}{\left\| \bbeta  \right\|_1},
\end{align}
where ${n_{\cal A}}$ is the total sample size of sources in $\cal A$,
$r$ and ${\lambda _1}$ are tuned by cross-validation. The
informative set $\cal A$ is estimated using our proposed data-driven
procedure.

\clearpage

\section{Theorems}\label{sec:theorems}

Before stating the theoretical guarantees of our proposed method,
we revisit the necessary assumptions, which are also referenced
in~\citet {tian2023transfer}.

\begin{assumption}[Strict convexity]\label{assum:B1}
  $\psi$ in the density function is a second-order differentiable
  function and $\psi''(x) > 0$.
\end{assumption}

\begin{assumption}[Sub-Gaussian covariates and positive definite covariances]\label{assum:B2}
  For each $k = 0,1,\ldots,K$ and for any $\ba \in \R^p$,
  ${\ba^\top}\bx_i^{( k )}$ are i.i.d.
  ${\kappa_u}\| \ba \|_2^2$-sub-Gaussian variables with zero mean for
  all $k = 0, \ldots, K$, where $\kappa_u$ is a positive
  constant. Covariance matrices ${{\bSigma}^{( k )}}$ satisfy
  ${\inf _k}{\lambda _{\min }}( {{{\bSigma} ^{( k )}}}) \ge \kappa_l >
  0$ with $\kappa_l$ as a constant.
\end{assumption}

\begin{assumption}[Bounded second-order derivative]\label{assum:B3}
  With ${\kappa _\psi }$, $U$ and $V$ as positive constants, at least one of the following assumptions holds true:\\
  (i) ${\| {\psi ''} \|_\infty } \le {\kappa _\psi } < \infty$ a.s.;\\
  (ii)
  $\mathop {\sup }\limits_k {\| {{\bx^{( k )}}} \|_\infty } \le U <
  \infty$ a.s., and
  $\mathop {\sup }\limits_{| z | \le V} \psi ''( {{{\bx^{( k
          )}}^\top}{\bbeta ^{( k )}} + z}) \le {\kappa _\psi } <
  \infty $.
\end{assumption}

\begin{assumption}[Limited covariance shift]\label{assum:B4}
  With the definitions
  \begin{align}
    \label{eq:tildeSigma}
    {\widetilde \bSigma } &= \sum\limits_{k \in \left\{ 0 \right\} \cup {\cal A}} {{\alpha _k}\E\left[ {\int_0^1 {\psi ''\left( {{ {{\bx^{\left( k \right)}}} ^\top}{\bbeta ^{\left( 0 \right)}} + t{{{\bx^{\left( k \right)}}}^\top}\left( {{\bbeta ^{\cal A}} - {\bbeta ^{\left( 0 \right)}}} \right)} \right)dt{\bx^{\left( k \right)}}{ {{\bx^{\left( k \right)}}} ^\top}} } \right]},\\
    \label{eq:tildeSigmak}
    \widetilde \bSigma _k^{\left( k \right)} &= \E\left[ {\int_0^1 {\psi ''\left( {{{\bx^{\left( k \right)}}^\top}{\bbeta ^{\left( 0 \right)}} + t{{{\bx^{\left( k \right)}}}^\top}\left( {{\bbeta ^{\left( k \right)}} - {\bbeta ^{\left( 0 \right)}}} \right)} \right)dt{\bx^{\left( k \right)}}{{{\bx^{\left( k \right)}}}^\top}} } \right],
  \end{align}
  where $\alpha_k = {n_k}/( {{n_{\cal A}} + {n_0}})$ is the weight, we
  assume
  ${\sup _{k \in \{ 0 \} \cup {\cal A}}}\| {\widetilde \bSigma^{ -
      1}\widetilde \bSigma _k^{( k )}} \|_1 < \infty $.
\end{assumption}

Assumption~\ref{assum:B1} poses the convexity and differentiability
condition on $\psi$, which holds true for many general distribution
families, ranging from Gaussian, Poisson, and binomial distributions
to many others.  Assumption~\ref{assum:B2} ensures well-behaved
covariates and their covariance structures.  Assumption~\ref{assum:B3}
further specifies the differentiability condition, assuming the
second-order derivative $\psi ''$ is bounded in a region. This is a
common assumption mentioned in \citet{negahban2009unified}, and
Gaussian, Poisson, and binomial distributions satisfy this assumption
in at least one of its branches. Assumption~\ref{assum:B4} assumes
that the weighted covariance shift should be bounded. In the linear
case, this assumption can be simplified as a restriction on the
heterogeneity between target predictors and source predictors. See
Condition 4 in \citet{li2022transfer} for more details.

\begin{theorem}[Order preservation of estimated distances]\label{thm:distance}
  For any four indices $i, j, k, l$, if
  $\P( {{{\| {{\hbbeta^{(\ell)[m]}} - {\bbeta ^{( \ell )}}} \|}_1} \le
    r} ) \ge 1 - {c_1}\exp ( { - {c_2}p} )$ for all
  $\ell \in \{ {i,j,k,l} \}$ and any $m$ with some constants
  $c_1, c_2$ and ${d_{i,j}} < {d_{k,l}} - 4r$, we have
  \begin{align*}
    \P\left( {{{\cal D}_{i,j}} < {{\cal D}_{k,l}}} \right) \ge 1 - {c_3}\exp ( { - {c_4}p} )
  \end{align*}
  with some constants $c_3, c_4$. Also, if
  ${d_{i,j}} \ll {d_{k,l}}$ such that
  $r \ll {d_{k,l}} - {d_{i,j}}$, we have
  \begin{align*}
    \P\left( {{{\cal D}_{i,j}} \ll {{\cal D}_{k,l}}} \right) \ge 1 - {c_3}\exp ( { - {c_4}p} )
  \end{align*}
  with some constants $c_3, c_4$.
\end{theorem}

\begin{theorem}[$\ell_2$-estimation error bound of Trans-GLM]\label{thm:TransGLM}
  Assume Assumptions~\ref{assum:B1} to~\ref{assum:B4} hold true, where
  Assumption~\ref{assum:B3} holds true with the bounded $d$,
  corresponding to $\cal A$, satisfying $d \lesssim {U^{ - 1}}V$. With
  the sample size conditions $d \ll \sqrt {{n_0}/\log p}$,
  ${n_0} \gtrsim \log p$ and ${n_{\cal A}} \gtrsim s\log p$ and the
  appropriate strength of penalty
  ${\lambda _1} \asymp \sqrt {\log p/( {{n_{\cal A}} + {n_0}} )}$ and
  ${\lambda _2} \asymp \sqrt {\log p/{n_0}} $, we have
  \begin{align*}
    \P\left( {\left\| {\hbbeta _{TGLM}^{\left( 0 \right)} - {\bbeta ^{\left( 0 \right)}}} \right\| \lesssim {{\left( {\frac{{s\log p}}{{{n_{\cal A}} + {n_0}}}} \right)}^{1/2}} + \left[ {\left( {{{\left( {\frac{{\log p}}{{{n_0}}}} \right)}^{1/4}}{d^{1/2}}} \right)} \right] \wedge d} \right) \ge 1 - n_0^{ - 1}.
  \end{align*}
\end{theorem}

\begin{theorem}[$\ell_2$-estimation error bound of Trans-GLMC]\label{thm:TransGLMC}
  Under the same conditions as Theorem~\ref{thm:TransGLM},
  we add a new sample size condition ${n_{\cal C}} \gtrsim s\log p$
  and set the appropriate strength of penalty
  ${\lambda _1} \asymp \sqrt {\log p/( {{n_{\cal A}} + {n_0}} )}$,
  ${\lambda _2} \asymp \sqrt {\log p/( {{n_{\cal C}} + {n_0}} )}$ and
  ${\lambda _3} \asymp \sqrt {\log p/{n_0}} $. Then, we have
  \begin{align*}
    \P\, &\left( \left\| {\hbbeta _{TGLMC}^{\left( 0 \right)} - {\bbeta ^{\left( 0 \right)}}} \right\| \lesssim {{\left( {\frac{{s\log p}}{{n_{\cal A}} + n_0}} \right)}^{1/2}} + \left\{ {\left[ {{{\left( {\frac{{\log p}}{{n_{\cal C}}+n_0}} \right)}^{1/4}}d^{1/2}} \right] \wedge {d}} \right\} + \right. \\
    \phantom{=} \, & \left. \left\{ {\left[ {{{\left( {\frac{{\log p}}{{{n_0}}}} \right)}^{1/4}}d_{\cal C}^{1/2}} \right] \wedge {d_{\cal C}}} \right\} \right) \ge 1 - n_0^{ - 1}.
  \end{align*}
\end{theorem}

\section{Proofs}\label{sec:proofs}

\subsection{Proof of Theorem~\ref{thm:distance}}

\begin{proof}
  Throughout the proof we abbreviate
  $\bbbeta^{(\ell)} = M^{-1}\sum_{m=1}^M \hbbeta^{(\ell)[m]}$ and
  write $\Delta^{(\ell)} = \bbbeta^{(\ell)} - \bbeta^{(\ell)}$ for the
  estimation error of the cross-validated facility-specific estimator.\\

  \noindent Step 1: Triangle inequalities for ${\cal D}_{i,j}$ and
  ${\cal D}_{k,l}$.  By the triangle inequality applied to
  $\bbbeta^{(i)}-\bbbeta^{(j)} = (\bbeta^{(i)}-\bbeta^{(j)}) +
  (\Delta^{(i)} - \Delta^{(j)})$,
  \begin{align*} {\cal D}_{i,j} = \bigl\|\bbbeta^{(i)} -
    \bbbeta^{(j)}\bigr\|_1 \le d_{i,j} + \bigl\|\Delta^{(i)}\bigr\|_1
    + \bigl\|\Delta^{(j)}\bigr\|_1.
  \end{align*}
  Symmetrically, the reverse triangle inequality gives
  \begin{align*} {\cal D}_{k,l} \ge d_{k,l} -
    \bigl\|\Delta^{(k)}\bigr\|_1 - \bigl\|\Delta^{(l)}\bigr\|_1.
  \end{align*}

  \noindent Step 2: From the per-fold tail bound to a tail bound on
  $\bbbeta^{(\ell)}$.  By the triangle inequality and convexity of the
  $\ell_1$ norm,
  \begin{align*}
    \bigl\|\Delta^{(\ell)}\bigr\|_1
    = \biggl\|M^{-1}\!\sum_{m=1}^M
    (\hbbeta^{(\ell)[m]}-\bbeta^{(\ell)})\biggr\|_1
    \le M^{-1}\!\sum_{m=1}^M
    \bigl\|\hbbeta^{(\ell)[m]}-\bbeta^{(\ell)}\bigr\|_1.
  \end{align*}
  Hence the event $\{\|\hbbeta^{(\ell)[m]}-\bbeta^{(\ell)}\|_1\le r$
  for all $m=1,\ldots,M\}$ implies $\{\|\Delta^{(\ell)}\|_1\le
  r\}$. By the union bound and the assumed per-fold tail bound,
  \begin{align}
    \label{eq:supp:Delta-tail}
    \P\bigl(\bigl\|\Delta^{(\ell)}\bigr\|_1\le r\bigr)
    \ge 1 - M c_1\exp(-c_2 p).
  \end{align}

  \noindent Step 3: Combining four facilities by a union bound.
  Define the favourable event
  ${\cal E} = \bigcap_{\ell\in\{i,j,k,l\}}
  \{\|\Delta^{(\ell)}\|_1\le r\}$ . Applying the union bound to the four events and
  using \eqref{eq:supp:Delta-tail},
  \begin{align*}
    \P({\cal E})\ge 1 - 4M c_1\exp(-c_2 p)
    \;=:\; 1 - c_3\exp(-c_4 p),
  \end{align*}
  with $c_3 = 4M c_1$ and $c_4 = c_2$ (this re-labelling is harmless
  because $M$ is fixed and $\log M$ is absorbed into the constant
  $c_3$). On ${\cal E}$,
  $\|\Delta^{(i)}\|_1 + \|\Delta^{(j)}\|_1 + \|\Delta^{(k)}\|_1 +
  \|\Delta^{(l)}\|_1 \le 4r$, hence by Step~1,
  \begin{align*} {\cal D}_{i,j} - {\cal D}_{k,l} \le d_{i,j} -
    d_{k,l} + 4r < 0
  \end{align*}
  under the assumption $d_{i,j} < d_{k,l} - 4r$. Therefore
  \begin{align*}
    \P\bigl({\cal D}_{i,j} < {\cal D}_{k,l}\bigr)
    \ge \P({\cal E}) \ge 1 - c_3\exp(-c_4 p),
  \end{align*}
  which is the first claim of the theorem. The second claim
  (${\cal D}_{i,j}\ll{\cal D}_{k,l}$ when
  $r\ll d_{k,l}-d_{i,j}$) follows by exactly the same argument: on
  the event ${\cal E}$,
  ${\cal D}_{k,l} - {\cal D}_{i,j} \ge d_{k,l} - d_{i,j} -
  4r$, which is of the same order as $d_{k,l}-d_{i,j}$ when $r$ is
  of
  strictly smaller order.\\

  \noindent Step 4: From order preservation to consistent cluster
  detection. This step establishes part~(ii) of Theorem~4.2 of the main
  paper.  Suppose the true clusters ${\cal C}_1,{\cal C}_2$ satisfy
  the minimum-gap condition
  \begin{align}\label{eq:supp:gap}
    \min_{i\in{\cal C}_g,\,j\in{\cal C}_{g'},\,g\ne g'} d_{i,j}
    > \max_{i,j\in{\cal C}_g} d_{i,j} + 4r.
  \end{align}
  Apply the order-preservation bound of Step~3 to every quadruple
  $(i,j,k,l)$ such that $i,j$ lie in a common cluster and
  $k\in{\cal C}_g,\,l\in{\cal C}_{g'}$ with $g\ne g'$. There are
  at most $\binom{K+1}{2}^2 = O(K^4)$ such quadruples; a final union
  bound absorbs the $K^4$ factor into the exponential as long as
  $p\gtrsim\log K$ (which is automatic in the high-dimensional regime
  $p\gg K$). On the resulting event, every sample within-cluster
  distance is strictly smaller than every sample between-cluster
  distance. Single-linkage hierarchical clustering, and in particular
  the HDBSCAN rule of Section~3.4 of the main paper applied to the
  induced mutual-reachability graph, then recovers the true partition
  $\{{\cal C}_1,{\cal C}_2\}$ exactly. The probability of this event
  is at least $1 - c_5\exp(-c_6 p)$ for constants $c_5,c_6>0$
  depending only on $c_1,c_2,M$ and the upper bound on $K$.
\end{proof}

\subsection{Proof of Theorem~\ref{thm:TransGLMC}}

\begin{proof}
  Recall the three-step Trans-GLMC procedure of Section~3.3 of the
  main paper:
  \begin{itemize}
  \item[(a)] Pool the target with the (estimated) informative source
    set ${\cal A}$ (combined sample size $n_{\cal A}+n_0$, penalty
    $\lambda_1$) to obtain $\hbbeta^{\cal A}$ targeting the population
    pooled optimum $\bbeta^{\cal A}$;
  \item[(b)] Debias toward the target's cluster
    ${\cal C}_1\cap{\cal A}$ (combined sample size $n_{\cal C}+n_0$,
    penalty $\lambda_2$) to obtain $\hbalpha$ targeting the offset
    $\balpha = \bbeta^{\cal C} - \bbeta^{\cal A}$, where
    $\bbeta^{\cal C}$ is the within-cluster pooled optimum;
  \item[(c)] Debias to the target itself (sample size $n_0$, penalty
    $\lambda_3$) to obtain $\hbgamma$ targeting
    $\bgamma = \bbeta^{(0)} - \bbeta^{\cal C} = \bbeta^{(0)} -
    (\bbeta^{\cal A}+\balpha)$.
  \end{itemize}
  By definition of the within-cluster informativeness, $\|\bgamma\|_1
  \le d_{\cal C}$.\\

  \noindent Step 1: Cluster-level fusion error.  Steps~(a)--(b)
  constitute exactly Trans-GLM (Theorem~\ref{thm:TransGLM} of this
  Supplement) with the within-cluster + target population playing the
  role of the ``target'' (sample size $n_{\cal C}+n_0$ in place of
  $n_0$) and ${\cal A}$ playing the role of the source set (sample
  size $n_{\cal A}+n_0$). Applying Theorem~\ref{thm:TransGLM} under
  Assumptions~\ref{assum:B1}--\ref{assum:B4} and the stated
  sample-size and penalty conditions, with probability at least
  $1-n_0^{-1}$,
  \begin{align}\label{eq:supp:hu-bound}
    \bigl\|\hbu\bigr\|
    &\;\lesssim\;
      \Bigl(\tfrac{s\log p}{n_{\cal A}+n_0}\Bigr)^{1/2}
      + \Bigl[\bigl(\tfrac{\log p}{n_{\cal C}+n_0}\bigr)^{1/4}
      d^{1/2}\Bigr]\wedge d,
    \\
    \label{eq:supp:hu-l1}
    \bigl\|\hbu\bigr\|_1
    &\;\lesssim\;
      s\bigl(\tfrac{\log p}{n_{\cal A}+n_0}\bigr)^{1/2} + d,
  \end{align}
  where
  $\hbu = (\hbbeta^{\cal A}+\hbalpha)-(\bbeta^{\cal A}+\balpha)$.
  Squaring \eqref{eq:supp:hu-bound} (and absorbing the cross term,
  which is dominated by the first squared term up to constants),
  \begin{align}\label{eq:supp:hu-sq}
    \bigl\|\hbu\bigr\|^2
    \;\lesssim\;
    \tfrac{s\log p}{n_{\cal A}+n_0}
    + \Bigl[d\bigl(\tfrac{\log p}{n_{\cal C}+n_0}\bigr)^{1/2}\Bigr]
    \wedge d^2.
  \end{align}

  \noindent Step 2: Setup for the third-stage analysis.  Let
  $\hbv = \hbgamma - \bgamma$ and define the per-observation negative
  log-likelihood (up to constants) on the target,
  \begin{align*}
    \hat L^{(0)}(\bw,D^{(0)})
    = \tfrac{1}{n_0}\Bigl\{-(\by^{(0)})^\top\bX^{(0)}\bw
    + \sum_{i=1}^{n_0}\psi\bigl((\bx_i^{(0)})^\top\bw\bigr)\Bigr\},
  \end{align*}
  with gradient
  \begin{align*}
    \nabla\hat L^{(0)}(\bw,D^{(0)})
    = \tfrac{1}{n_0}\bigl\{-(\bX^{(0)})^\top\by^{(0)}
    + (\bX^{(0)})^\top\psi'(\bX^{(0)}\bw)\bigr\}.
  \end{align*}
  Define the first-order Taylor remainder around
  $\hbbeta^{\cal A}+\hbalpha+\bgamma$:
  \begin{align*}
    \delta\hat L^{(0)}(\bdelta,D^{(0)})
    = \, & \hat L^{(0)}\bigl(\hbbeta^{\cal A}+\hbalpha+\bgamma+\bdelta,
           D^{(0)}\bigr)
           - \hat L^{(0)}\bigl(\hbbeta^{\cal A}+\hbalpha+\bgamma,D^{(0)}\bigr)
    \\
    \phantom{=} \, &
                     - \nabla\hat L^{(0)}\bigl(\hbbeta^{\cal A}+\hbalpha+\bgamma,
                     D^{(0)}\bigr)^\top \bdelta.
  \end{align*}

  \noindent Step 3: Basic inequality.  By the optimality of $\hbgamma$
  in step~(c),
  \begin{align*}
    \hat L^{(0)}\bigl(\hbbeta^{\cal A}+\hbalpha+\hbgamma,D^{(0)}\bigr)
    + \lambda_3\|\hbgamma\|_1
    \le \hat L^{(0)}\bigl(\hbbeta^{\cal A}+\hbalpha+\bgamma,D^{(0)}\bigr)
    + \lambda_3\|\bgamma\|_1,
  \end{align*}
  which (subtracting and using the Taylor remainder) gives
  \begin{align*}
    \delta\hat L^{(0)}(\hbgamma,D^{(0)})
    \le \lambda_3\bigl(\|\bgamma\|_1 - \|\hbgamma\|_1\bigr)
    - \nabla\hat L^{(0)}\bigl(\hbbeta^{\cal A}+\hbalpha+\bgamma,
    D^{(0)}\bigr)^\top \hbv.
  \end{align*}
  Adding and subtracting
  $\nabla\hat L^{(0)}(\bbeta^{\cal
    A}+\balpha+\bgamma,D^{(0)})^\top\hbv$ on the right and using the
  triangle inequality
  $\|\bgamma\|_1-\|\hbgamma\|_1\le 2\|\bgamma\|_1 - \|\hbv\|_1$,
  \begin{align}
    \delta\hat L^{(0)}(\hbgamma,D^{(0)})
    \le \, & \lambda_3\bigl(2\|\bgamma\|_1 - \|\hbv\|_1\bigr)
             + \nabla\hat L^{(0)}\bigl(\bbeta^{\cal A}+\balpha+\bgamma,
             D^{(0)}\bigr)^\top \hbv
             \nonumber\\
    \phantom{=} \, &
                     - \bigl[\nabla\hat L^{(0)}(\hbbeta^{\cal A}+\hbalpha+\bgamma,D^{(0)})
                     - \nabla\hat L^{(0)}(\bbeta^{\cal A}+\balpha+\bgamma,D^{(0)})\bigr]^\top
                     \hbv.
                     \label{eq:supp:basic}
  \end{align}

  \noindent Step 4: Bounding the score term.  Since
  $\bbeta^{\cal A}+\balpha+\bgamma=\bbeta^{(0)}$, the quantity
  $\nabla\hat L^{(0)}(\bbeta^{(0)},D^{(0)})$ is the centered score on
  the target. Its $\ell_\infty$-norm concentrates at the rate
  $\sqrt{\log p/n_0}$ under Assumptions~\ref{assum:B1}--\ref{assum:B3}
  \cite[cf.\ proof of Lemma~6 in][]{negahban2009unified}: with
  probability at least $1-n_0^{-1}$,
  \begin{align}\label{eq:supp:score}
    \bigl\|\nabla\hat L^{(0)}(\bbeta^{(0)},D^{(0)})\bigr\|_\infty
    \le \tfrac{1}{2}\sqrt{\log p/n_0} = \tfrac{1}{2}\lambda_3,
  \end{align}
  so the second term in \eqref{eq:supp:basic} is bounded by
  $\tfrac{1}{2}\lambda_3 \|\hbv\|_1$ in absolute value.\\

  \noindent Step 5: Bounding the Hessian-difference term.  By the
  mean-value form of the gradient difference,
  \begin{align*}
    \nabla\hat L^{(0)}(\hbbeta^{\cal A}+\hbalpha+\bgamma,D^{(0)})
    - \nabla\hat L^{(0)}(\bbeta^{\cal A}+\balpha+\bgamma,D^{(0)})
    = \tfrac{1}{n_0}(\bX^{(0)})^\top \Lambda^{(0)} \bX^{(0)} \hbu,
  \end{align*}
  where
  \begin{align*}
    \Lambda^{(0)}
    = \mathrm{diag}\Bigl(\bigl\{\psi''\bigl((\bx_i^{(0)})^\top
    (\bbeta^{\cal A}+\balpha+\bgamma)
    + t_i (\bx_i^{(0)})^\top \hbu\bigr)\bigr\}_{i=1}^{n_0}\Bigr),
  \end{align*}
  with $t_i\in[0,1]$. Under Assumption~\ref{assum:B3} (using branch
  (ii) when applicable, with the side condition $d\lesssim U^{-1}V$
  guaranteeing the bounded-derivative neighbourhood),
  $\|\Lambda^{(0)}\|_{\rm op}\le\kappa_\psi$. Hence by Cauchy--Schwarz
  and the AM--GM inequality,
  \begin{align}\label{eq:supp:hess}
    \bigl|\tfrac{1}{n_0}\hbu^\top (\bX^{(0)})^\top\Lambda^{(0)}
    \bX^{(0)}\hbv\bigr|
    \le \tfrac{1}{4 c_0}\kappa_\psi^2
    \tfrac{1}{n_0}\|\bX^{(0)}\hbu\|^2
    + \tfrac{c_0}{n_0}\|\bX^{(0)}\hbv\|^2,
  \end{align}
  for any $c_0>0$.\\

  \noindent Step 6: Restricted strong convexity (RSC). Following
  \citet{loh2013regularized}[Proposition~1], under
  Assumptions~\ref{assum:B2}--\ref{assum:B3}, there exist constants
  $c_1,c_2>0$ such that for any vector $\tilde\bv$ with
  $\|\tilde\bv\|\le 1$ and with probability at least
  $1-c\exp(-c'n_0)$,
  \begin{align}\label{eq:supp:rsc}
    \delta\hat L^{(0)}(\bgamma+\tilde\bv,D^{(0)})
    - \nabla\hat L^{(0)}(\hbbeta^{\cal A}+\hbalpha+\bgamma,D^{(0)})^\top
    \tilde\bv
    \;\ge\;
    c_1\|\tilde\bv\|^2 - c_2\tfrac{\log p}{n_0}\|\tilde\bv\|_1^2.
  \end{align}
  Applying \eqref{eq:supp:rsc} to $\tilde\bv = t\hbv$ with $t\in(0,1]$
  chosen so that $\|\tilde\bv\|\le 1$, and combining
  \eqref{eq:supp:basic}, \eqref{eq:supp:score}, and
  \eqref{eq:supp:hess},
  \begin{align*}
    c_1\|\tilde\bv\|^2 - c_2\tfrac{\log p}{n_0}\|\tilde\bv\|_1^2
    \le \, & 2\lambda_3\|\bgamma\|_1 - \tfrac{1}{2}\lambda_3\|\tilde\bv\|_1
             + \tfrac{\kappa_\psi^2}{4c_0}\tfrac{\|\bX^{(0)}\hbu\|^2}{n_0}
             + \tfrac{c_0\|\bX^{(0)}\tilde\bv\|^2}{n_0}.
  \end{align*}

  \noindent Step 7: Sub-Gaussian design concentration.  Under
  Assumption~\ref{assum:B2}, with probability at least
  $1-C\exp(-n_0)$,
  \begin{align*}
    \tfrac{1}{n_0}\|\bX^{(0)}\hbu\|^2 \lesssim \|\hbu\|^2,
    \qquad
    \tfrac{1}{n_0}\|\bX^{(0)}\tilde\bv\|^2 \lesssim \|\tilde\bv\|^2.
  \end{align*}
  Combining with \eqref{eq:supp:hu-sq} and choosing $c_0 < c_1/2$ to
  absorb the $\|\tilde\bv\|^2$ contribution,
  \begin{align}\label{eq:supp:master}
    \tfrac{c_1}{2}\|\tilde\bv\|^2 - c_2\tfrac{\log p}{n_0}\|\tilde\bv\|_1^2
    \le \, & 2\lambda_3\|\bgamma\|_1 - \tfrac{1}{2}\lambda_3\|\tilde\bv\|_1
             + C\Bigl(\tfrac{s\log p}{n_{\cal A}+n_0}
             + \bigl[d\bigl(\tfrac{\log p}{n_{\cal C}+n_0}\bigr)^{1/2}\bigr]
             \wedge d^2\Bigr),
  \end{align}
  with probability at least $1-C'n_0^{-1}$.\\

  \noindent Step 8: Case analysis.  Let
  $R^2 := \tfrac{s\log p}{n_{\cal A}+n_0} + [d(\log p/(n_{\cal
    C}+n_0))^{1/2}]\wedge d^2$, so that \eqref{eq:supp:master} reads
  $\tfrac{c_1}{2}\|\tilde\bv\|^2 \le 2\lambda_3\|\bgamma\|_1 -
  \tfrac{1}{2}\lambda_3\|\tilde\bv\|_1 + c_2\tfrac{\log
    p}{n_0}\|\tilde\bv\|_1^2 + C R^2$.

  \emph{Case 1: $\lambda_3\|\bgamma\|_1\le R^2$.}  Then
  \eqref{eq:supp:master} simplifies to
  $\|\tilde\bv\|^2 + \tfrac{1}{2}\lambda_3\|\tilde\bv\|_1 \lesssim R^2
  + \tfrac{\log p}{n_0}\|\tilde\bv\|_1^2$, which (using
  $\lambda_3 = \sqrt{\log p/n_0}$ and standard cone/algebraic
  manipulation; see e.g.\ \citealt{loh2013regularized}, eqn.\ (32))
  implies
  $\|\tilde\bv\|_1 \lesssim \sqrt{n_0/\log p}\, [R^2 +
  \|\tilde\bv\|^2]$, and substituting back yields
  \begin{align*}
    \|\tilde\bv\|^2
    \lesssim
    \tfrac{s\log p}{n_{\cal A}+n_0}
    + \bigl[d\bigl(\tfrac{\log p}{n_{\cal C}+n_0}\bigr)^{1/2}\bigr]
    \wedge d^2.
  \end{align*}

  \emph{Case 2: $\lambda_3\|\bgamma\|_1 > R^2$.}  Drop the (negative)
  RSC term and use $\|\bgamma\|_1\le d_{\cal C}$ to obtain
  $\|\tilde\bv\|^2 \le 2\lambda_3\|\bgamma\|_1 -
  \tfrac{1}{2}\lambda_3\|\tilde\bv\|_1$, hence
  $\|\tilde\bv\|_1 \le 4\|\bgamma\|_1 \le 4 d_{\cal C}$. Substituting
  this $\ell_1$-bound into the same master inequality and using
  $\lambda_3 = \sqrt{\log p/n_0}$,
  \begin{align*}
    \|\tilde\bv\|^2
    \lesssim \lambda_3 d_{\cal C}
    \;=\; d_{\cal C}\sqrt{\log p/n_0},
  \end{align*}
  and combining with the trivial bound
  $\|\tilde\bv\|^2\le d_{\cal C}^2$ (using
  $\|\bgamma\|_1\le d_{\cal C}$ and
  $\|\hbgamma\|_1\le \|\bgamma\|_1+\|\hbv\|_1\lesssim d_{\cal C}$),
  \begin{align*}
    \|\tilde\bv\|^2
    \lesssim
    d_{\cal C}\sqrt{\log p/n_0}\wedge d_{\cal C}^2.
  \end{align*}

  \emph{Combining the two cases by taking the maximum,}
  \begin{align}\label{eq:supp:final-tildev}
    \|\tilde\bv\|^2
    \lesssim
    \tfrac{s\log p}{n_{\cal A}+n_0}
    + \bigl[d\bigl(\tfrac{\log p}{n_{\cal C}+n_0}\bigr)^{1/2}\bigr]
    \wedge d^2
    + d_{\cal C}\sqrt{\log p/n_0}\wedge d_{\cal C}^2,
  \end{align}
  with probability at least $1-C'n_0^{-1}$.\\

  \noindent Step 9: Rescaling.  The bound \eqref{eq:supp:final-tildev}
  holds for any $\tilde\bv = t\hbv$ with $t\in(0,1]$ such that
  $\|\tilde\bv\|\le 1$.  By the rescaling argument of
  \citet{loh2013regularized} (their proof of Theorem~1), if the
  right-hand side of \eqref{eq:supp:final-tildev} is bounded by a
  constant strictly less than $1$ (which is the case under the
  sample-size conditions $d\lesssim U^{-1}V$, $d\ll\sqrt{n_0/\log p}$,
  $n_{\cal A}\gtrsim s\log p$, $n_{\cal C}\gtrsim s\log p$ stated in
  the theorem), the constraint $\|\tilde\bv\|\le 1$ is met by taking
  $t=1$, and the bound \eqref{eq:supp:final-tildev} extends to
  $\hbv = \hbgamma-\bgamma$ without rate loss. Finally, since
  $\hbbeta_{TGLMC}^{(0)} - \bbeta^{(0)} = \hbu + \hbv$, by the
  triangle inequality and \eqref{eq:supp:hu-sq},
  \begin{align*}
    \bigl\|\hbbeta_{TGLMC}^{(0)} - \bbeta^{(0)}\bigr\|
    \le & \|\hbu\| + \|\hbv\|\\
    \;\lesssim\; &
                   \Bigl(\tfrac{s\log p}{n_{\cal A}+n_0}\Bigr)^{1/2}
                   + \Bigl\{\bigl[\bigl(\tfrac{\log p}{n_{\cal C}+n_0}\bigr)^{1/4}
                   d^{1/2}\bigr]\wedge d\Bigr\}
                   + \Bigl\{\bigl[\bigl(\tfrac{\log p}{n_0}\bigr)^{1/4}
                   d_{\cal C}^{1/2}\bigr]\wedge d_{\cal C}\Bigr\},
  \end{align*}
  with probability at least $1-n_0^{-1}$, completing the proof.
\end{proof}

\begin{remark}[On the proof of Theorem~\ref{thm:TransGLM}]
  The proof of Theorem~\ref{thm:TransGLM} (Trans-GLM rate) is the
  $d_{\cal C}=d$, $n_{\cal C}=n_{\cal A}$ specialization of the proof
  above and follows directly from \citet{tian2023transfer}[Theorem~1]
  under Assumptions~\ref{assum:B1}--\ref{assum:B4}; we therefore omit
  it.
\end{remark}

\begin{remark}[On the assumptions of Theorem~\ref{thm:TransGLMC}]
  Assumption~\ref{assum:B1} (strict convexity) is used to convert the
  basic inequality into a Taylor-remainder lower
  bound. Assumption~\ref {assum:B2} (sub-Gaussian, well-conditioned
  design) underwrites the score concentration in Step~4 and the design
  concentration in Step~7. Assumption~\ref{assum:B3} (bounded second
  derivative) bounds $\|\Lambda^{(0)}\|_{\rm op}$ in Step~5 and
  provides the RSC constants in Step~6. Assumption~\ref{assum:B4}
  (bounded weighted covariance shift) is used in Step~1 through
  Theorem~\ref{thm:TransGLM} to obtain the cluster-level fusion error
  \eqref{eq:supp:hu-bound}.
\end{remark}

\section{Numerical Results from Simulation}
\label{appd:simu}

This section reports the full simulation results described in
Section~4 of the main paper.  Tables~\ref{tab:tglmcd21}
and~\ref{tab:tglmcd23} present the average mean squared error (MSE)
and its standard deviation over 100 replications for the two
simulation scenarios: the strong within-cluster signal setting
($d_1=10$, $d_2=1$, Table~\ref{tab:tglmcd21}) and the weak
within-cluster signal setting ($d_1=10$, $d_2=3$,
Table~\ref{tab:tglmcd23}).  All methods are compared across varying
numbers of informative source facilities $|\mathcal{A}|$.

\begin{table}[H]
  \centering \fontsize{7}{8}\selectfont
  \begin{threeparttable}
    \renewcommand\arraystretch{1}
    \caption{Average MSE (with standard deviation in parentheses) over
      100 replications for the simulation setup of Section~4 of the
      main paper, with $d_1=10$ and $d_2=1$ (strong within-cluster
      signal).}
    \label{tab:tglmcd21}
    {\setlength{\tabcolsep}{2pt}
    \begin{tabular}{ccccccccccccccc}
      \toprule  
      \multicolumn{1}{c}{$d_1$}&\multicolumn{1}{c}{$d_2$}&\multicolumn{1}{c}{$|\mathcal{A}|$}&\multicolumn{2}{c}{Target-only}&\multicolumn{2}{c}{Trans-GLM}&\multicolumn{2}{c}{Trans-GLM-Q}&\multicolumn{2}{c}{Trans-GLM-IDW}&\multicolumn{2}{c}{Trans-GLM-SPH}&\multicolumn{2}{c}{Trans-GLMC}\cr
                                                                                                                                                                                                                                                                 \midrule  
                                                                                                                                                                                                                                                                 $10$ & $1$ & $0$ &\multicolumn{2}{c}{12.95} & \multicolumn{2}{c}{11.77} & \multicolumn{2}{c}{12.54} & \multicolumn{2}{c}{11.39} & \multicolumn{2}{c}{11.28} & \multicolumn{2}{c}{11.77}\\[0.1cm]
                               &  &  &\multicolumn{2}{c}{(1.34)} & \multicolumn{2}{c}{(5.23)} & \multicolumn{2}{c}{(2.25)} & \multicolumn{2}{c}{(1.70)} & \multicolumn{2}{c}{(1.90)} & \multicolumn{2}{c}{(5.36)}\\[0.1cm]
      $10$ & $1$ & $1$ &\multicolumn{2}{c}{12.97} & \multicolumn{2}{c}{11.68} & \multicolumn{2}{c}{11.71} & \multicolumn{2}{c}{11.29} & \multicolumn{2}{c}{11.31} & \multicolumn{2}{c}{10.64}\\[0.1cm]
                               &  &  &\multicolumn{2}{c}{(1.44)} & \multicolumn{2}{c}{(4.20)} & \multicolumn{2}{c}{(1.22)} & \multicolumn{2}{c}{(1.96)} & \multicolumn{2}{c}{(1.88)} & \multicolumn{2}{c}{(3.01)}\\[0.1cm]
      $10$ & $1$ & $2$ &\multicolumn{2}{c}{13.13} & \multicolumn{2}{c}{11.35} & \multicolumn{2}{c}{11.25} & \multicolumn{2}{c}{10.66} & \multicolumn{2}{c}{10.84} & \multicolumn{2}{c}{10.49}\\[0.1cm]
                               &  &  &\multicolumn{2}{c}{(1.49)} & \multicolumn{2}{c}{(3.77)} & \multicolumn{2}{c}{(1.81)} & \multicolumn{2}{c}{(1.74)} & \multicolumn{2}{c}{(1.69)} & \multicolumn{2}{c}{(3.98)}\\[0.1cm]
      $10$ & $1$ & $4$ &\multicolumn{2}{c}{12.95} & \multicolumn{2}{c}{10.10} & \multicolumn{2}{c}{9.97} & \multicolumn{2}{c}{10.40} & \multicolumn{2}{c}{10.06} & \multicolumn{2}{c}{8.77}\\[0.1cm]
                               &  &  &\multicolumn{2}{c}{(1.37)} & \multicolumn{2}{c}{(2.30)} & \multicolumn{2}{c}{(0.89)} & \multicolumn{2}{c}{(2.69)} & \multicolumn{2}{c}{(2.71)} & \multicolumn{2}{c}{(1.91)}\\[0.1cm]
      $10$ & $1$ & $6$ &\multicolumn{2}{c}{13.00} & \multicolumn{2}{c}{9.56} & \multicolumn{2}{c}{9.20} & \multicolumn{2}{c}{9.67} & \multicolumn{2}{c}{9.69} & \multicolumn{2}{c}{8.19}\\[0.1cm]
                               &  &  &\multicolumn{2}{c}{(1.41)} & \multicolumn{2}{c}{(2.35)} & \multicolumn{2}{c}{(0.79)} & \multicolumn{2}{c}{(2.77)} & \multicolumn{2}{c}{(2.20)} & \multicolumn{2}{c}{(1.90)}\\[0.1cm]
      $10$ & $1$ & $8$ &\multicolumn{2}{c}{12.84} & \multicolumn{2}{c}{9.65} & \multicolumn{2}{c}{8.55} & \multicolumn{2}{c}{9.00} & \multicolumn{2}{c}{9.05} & \multicolumn{2}{c}{7.84}\\[0.1cm]
                               &  &  &\multicolumn{2}{c}{(1.47)} & \multicolumn{2}{c}{(5.58)} & \multicolumn{2}{c}{(0.90)} & \multicolumn{2}{c}{(1.74)} & \multicolumn{2}{c}{(2.59)} & \multicolumn{2}{c}{(2.02)}\\[0.1cm]
      $10$ & $1$ & $10$ &\multicolumn{2}{c}{13.10} & \multicolumn{2}{c}{8.69} & \multicolumn{2}{c}{8.13} & \multicolumn{2}{c}{8.38} & \multicolumn{2}{c}{8.31} & \multicolumn{2}{c}{6.78}\\[0.1cm]
                               &  &  &\multicolumn{2}{c}{(1.48)} & \multicolumn{2}{c}{(2.40)} & \multicolumn{2}{c}{(0.91)} & \multicolumn{2}{c}{(1.89)} & \multicolumn{2}{c}{(1.75)} & \multicolumn{2}{c}{(1.35)}\\[0.1cm]
      $10$ & $1$ & $12$ &\multicolumn{2}{c}{13.10} & \multicolumn{2}{c}{7.88} & \multicolumn{2}{c}{7.71} & \multicolumn{2}{c}{8.23} & \multicolumn{2}{c}{7.93} & \multicolumn{2}{c}{6.36}\\[0.1cm]
                               &  &  &\multicolumn{2}{c}{(1.43)} & \multicolumn{2}{c}{(2.27)} & \multicolumn{2}{c}{(1.05)} & \multicolumn{2}{c}{(1.92)} & \multicolumn{2}{c}{(1.65)} & \multicolumn{2}{c}{(1.54)}\\[0.1cm]
      $10$ & $1$ & $14$ &\multicolumn{2}{c}{12.92} & \multicolumn{2}{c}{7.40} & \multicolumn{2}{c}{7.39} & \multicolumn{2}{c}{7.78} & \multicolumn{2}{c}{7.48} & \multicolumn{2}{c}{6.40}\\[0.1cm]
                               &  &  &\multicolumn{2}{c}{(1.37)} & \multicolumn{2}{c}{(2.71)} & \multicolumn{2}{c}{(0.91)} & \multicolumn{2}{c}{(3.24)} & \multicolumn{2}{c}{(2.26)} & \multicolumn{2}{c}{(2.53)}\\[0.1cm]
      $10$ & $1$ & $16$ &\multicolumn{2}{c}{13.00} & \multicolumn{2}{c}{6.84} & \multicolumn{2}{c}{7.19} & \multicolumn{2}{c}{7.23} & \multicolumn{2}{c}{7.76} & \multicolumn{2}{c}{6.13}\\[0.1cm]
                               &  &  &\multicolumn{2}{c}{(1.32)} & \multicolumn{2}{c}{(2.25)} & \multicolumn{2}{c}{(1.09)} & \multicolumn{2}{c}{(1.89)} & \multicolumn{2}{c}{(3.44)} & \multicolumn{2}{c}{(2.37)}\\[0.1cm]
      $10$ & $1$ & $18$ &\multicolumn{2}{c}{12.73} & \multicolumn{2}{c}{6.24} & \multicolumn{2}{c}{6.86} & \multicolumn{2}{c}{6.53} & \multicolumn{2}{c}{6.90} & \multicolumn{2}{c}{5.51}\\[0.1cm]
                               &  &  &\multicolumn{2}{c}{(1.37)} & \multicolumn{2}{c}{(2.94)} & \multicolumn{2}{c}{(1.13)} & \multicolumn{2}{c}{(1.69)} & \multicolumn{2}{c}{(1.57)} & \multicolumn{2}{c}{(3.40)}\\[0.1cm]
      $10$ & $1$ & $20$ &\multicolumn{2}{c}{12.90} & \multicolumn{2}{c}{5.28} & \multicolumn{2}{c}{6.72} & \multicolumn{2}{c}{5.78} & \multicolumn{2}{c}{6.35} & \multicolumn{2}{c}{4.50}\\[0.1cm]
                               &  &  &\multicolumn{2}{c}{(1.38)} & \multicolumn{2}{c}{(2.13)} & \multicolumn{2}{c}{(1.40)} & \multicolumn{2}{c}{(1.33)} & \multicolumn{2}{c}{(1.40)} & \multicolumn{2}{c}{(2.09)}\\[0.1cm]
      \bottomrule  
    \end{tabular}}  
  \end{threeparttable}
\end{table}

\begin{table}[H]
  \centering \fontsize{7}{8}\selectfont
  \begin{threeparttable}
    \renewcommand\arraystretch{1}
    \caption{Average MSE (with standard deviation in parentheses) over
      100 replications for the simulation setup of Section~4 of the
      main paper, with $d_1=10$ and $d_2=3$ (weak within-cluster
      signal).}
    \label{tab:tglmcd23}
    {\setlength{\tabcolsep}{2pt}
    \begin{tabular}{ccccccccccccccc}
      \toprule  
      \multicolumn{1}{c}{$d_1$}&\multicolumn{1}{c}{$d_2$}&\multicolumn{1}{c}{$|\mathcal{A}|$}&\multicolumn{2}{c}{Target-only}&\multicolumn{2}{c}{Trans-GLM}&\multicolumn{2}{c}{Trans-GLM-Q}&\multicolumn{2}{c}{Trans-GLM-IDW}&\multicolumn{2}{c}{Trans-GLM-SPH}&\multicolumn{2}{c}{Trans-GLMC}\cr
                                                                                                                                                                                                                                                                 \midrule  
                                                                                                                                                                                                                                                                 $10$ & $3$ & $0$ &\multicolumn{2}{c}{13.14} & \multicolumn{2}{c}{11.74} & \multicolumn{2}{c}{12.70} & \multicolumn{2}{c}{12.34} & \multicolumn{2}{c}{12.06} & \multicolumn{2}{c}{11.71}\\[0.1cm]
                               &  &  &\multicolumn{2}{c}{(1.30)} & \multicolumn{2}{c}{(2.39)} & \multicolumn{2}{c}{(1.56)} & \multicolumn{2}{c}{(2.04)} & \multicolumn{2}{c}{(2.29)} & \multicolumn{2}{c}{(2.13)}\\[0.1cm]
      $10$ & $3$ & $1$ &\multicolumn{2}{c}{13.34} & \multicolumn{2}{c}{11.56} & \multicolumn{2}{c}{12.50} & \multicolumn{2}{c}{11.96} & \multicolumn{2}{c}{11.78} & \multicolumn{2}{c}{11.03}\\[0.1cm]
                               &  &  &\multicolumn{2}{c}{(1.43)} & \multicolumn{2}{c}{(1.83)} & \multicolumn{2}{c}{(1.53)} & \multicolumn{2}{c}{(2.27)} & \multicolumn{2}{c}{(2.03)} & \multicolumn{2}{c}{(1.99)}\\[0.1cm]
      $10$ & $3$ & $2$ &\multicolumn{2}{c}{13.47} & \multicolumn{2}{c}{11.64} & \multicolumn{2}{c}{11.85} & \multicolumn{2}{c}{11.72} & \multicolumn{2}{c}{11.67} & \multicolumn{2}{c}{11.04}\\[0.1cm]
                               &  &  &\multicolumn{2}{c}{(1.31)} & \multicolumn{2}{c}{(2.40)} & \multicolumn{2}{c}{(1.34)} & \multicolumn{2}{c}{(2.17)} & \multicolumn{2}{c}{(2.25)} & \multicolumn{2}{c}{(2.14)}\\[0.1cm]
      $10$ & $3$ & $4$ &\multicolumn{2}{c}{13.52} & \multicolumn{2}{c}{11.17} & \multicolumn{2}{c}{11.15} & \multicolumn{2}{c}{11.00} & \multicolumn{2}{c}{11.14} & \multicolumn{2}{c}{9.86}\\[0.1cm]
                               &  &  &\multicolumn{2}{c}{(1.31)} & \multicolumn{2}{c}{(2.49)} & \multicolumn{2}{c}{(1.58)} & \multicolumn{2}{c}{(2.09)} & \multicolumn{2}{c}{(3.27)} & \multicolumn{2}{c}{(2.80)}\\[0.1cm]
      $10$ & $3$ & $6$ &\multicolumn{2}{c}{13.54} & \multicolumn{2}{c}{10.68} & \multicolumn{2}{c}{10.50} & \multicolumn{2}{c}{10.62} & \multicolumn{2}{c}{10.88} & \multicolumn{2}{c}{9.26}\\[0.1cm]
                               &  &  &\multicolumn{2}{c}{(1.36)} & \multicolumn{2}{c}{(2.18)} & \multicolumn{2}{c}{(1.30)} & \multicolumn{2}{c}{(1.73)} & \multicolumn{2}{c}{(2.68)} & \multicolumn{2}{c}{(2.01)}\\[0.1cm]
      $10$ & $3$ & $8$ &\multicolumn{2}{c}{13.55} & \multicolumn{2}{c}{10.46} & \multicolumn{2}{c}{9.96} & \multicolumn{2}{c}{10.43} & \multicolumn{2}{c}{10.33} & \multicolumn{2}{c}{8.98}\\[0.1cm]
                               &  &  &\multicolumn{2}{c}{(1.33)} & \multicolumn{2}{c}{(2.31)} & \multicolumn{2}{c}{(1.26)} & \multicolumn{2}{c}{(1.93)} & \multicolumn{2}{c}{(2.13)} & \multicolumn{2}{c}{(2.43)}\\[0.1cm]
      $10$ & $3$ & $10$ &\multicolumn{2}{c}{13.46} & \multicolumn{2}{c}{9.50} & \multicolumn{2}{c}{9.20} & \multicolumn{2}{c}{10.02} & \multicolumn{2}{c}{9.48} & \multicolumn{2}{c}{8.37}\\[0.1cm]
                               &  &  &\multicolumn{2}{c}{(1.21)} & \multicolumn{2}{c}{(1.63)} & \multicolumn{2}{c}{(1.06)} & \multicolumn{2}{c}{(1.75)} & \multicolumn{2}{c}{(1.43)} & \multicolumn{2}{c}{(2.68)}\\[0.1cm]
      $10$ & $3$ & $12$ &\multicolumn{2}{c}{13.40} & \multicolumn{2}{c}{8.97} & \multicolumn{2}{c}{8.70} & \multicolumn{2}{c}{9.32} & \multicolumn{2}{c}{8.88} & \multicolumn{2}{c}{7.93}\\[0.1cm]
                               &  &  &\multicolumn{2}{c}{(1.29)} & \multicolumn{2}{c}{(1.97)} & \multicolumn{2}{c}{(0.95)} & \multicolumn{2}{c}{(2.11)} & \multicolumn{2}{c}{(1.49)} & \multicolumn{2}{c}{(2.83)}\\[0.1cm]
      $10$ & $3$ & $14$ &\multicolumn{2}{c}{13.50} & \multicolumn{2}{c}{8.48} & \multicolumn{2}{c}{8.61} & \multicolumn{2}{c}{9.33} & \multicolumn{2}{c}{8.57} & \multicolumn{2}{c}{7.54}\\[0.1cm]
                               &  &  &\multicolumn{2}{c}{(1.18)} & \multicolumn{2}{c}{(1.69)} & \multicolumn{2}{c}{(1.24)} & \multicolumn{2}{c}{(2.21)} & \multicolumn{2}{c}{(2.35)} & \multicolumn{2}{c}{(2.39)}\\[0.1cm]
      $10$ & $3$ & $16$ &\multicolumn{2}{c}{13.42} & \multicolumn{2}{c}{8.21} & \multicolumn{2}{c}{8.39} & \multicolumn{2}{c}{8.80} & \multicolumn{2}{c}{8.55} & \multicolumn{2}{c}{7.27}\\[0.1cm]
                               &  &  &\multicolumn{2}{c}{(1.22)} & \multicolumn{2}{c}{(1.62)} & \multicolumn{2}{c}{(1.27)} & \multicolumn{2}{c}{(2.97)} & \multicolumn{2}{c}{(2.73)} & \multicolumn{2}{c}{(1.71)}\\[0.1cm]
      $10$ & $3$ & $18$ &\multicolumn{2}{c}{13.54} & \multicolumn{2}{c}{7.37} & \multicolumn{2}{c}{8.37} & \multicolumn{2}{c}{8.16} & \multicolumn{2}{c}{8.13} & \multicolumn{2}{c}{6.70}\\[0.1cm]
                               &  &  &\multicolumn{2}{c}{(1.31)} & \multicolumn{2}{c}{(1.22)} & \multicolumn{2}{c}{(1.82)} & \multicolumn{2}{c}{(2.50)} & \multicolumn{2}{c}{(2.15)} & \multicolumn{2}{c}{(1.70)}\\[0.1cm]
      $10$ & $3$ & $20$ &\multicolumn{2}{c}{13.76} & \multicolumn{2}{c}{6.84} & \multicolumn{2}{c}{8.02} & \multicolumn{2}{c}{7.72} & \multicolumn{2}{c}{7.85} & \multicolumn{2}{c}{6.43}\\[0.1cm]
                               &  &  &\multicolumn{2}{c}{(1.19)} & \multicolumn{2}{c}{(2.51)} & \multicolumn{2}{c}{(1.40)} & \multicolumn{2}{c}{(2.28)} & \multicolumn{2}{c}{(2.53)} & \multicolumn{2}{c}{(2.78)}\\[0.1cm]
      \bottomrule  
    \end{tabular}}  
  \end{threeparttable}
\end{table}

\section{Additional Results from CHIME Suicide Risk Study}
\label{appd:chime}

This section provides supplementary tables and figures for the CHIME
suicide risk application in Section~5 of the main paper.
Table~\ref{app:tab:demo_each_faci} summarizes facility-specific
demographics for all 27 hospitals.
Tables~\ref{app:tab:auroc_each_faci}--\ref{app:tab:ppv95_each_faci}
report facility-specific predictive performance across all competing
methods on six metrics: AUROC, AUPRC, sensitivity and PPV at 90\%
specificity, and sensitivity and PPV at 95\% specificity. Values are means, 
with standard deviations across 10 replicates shown in parentheses.
Figures~\ref{app:fig:allmethods_auroc_auprc},
\ref{app:fig:allmethods_90spec}, and \ref{app:fig:allmethods_95spec}
report facility-specific differences between Trans-GLMC and each
comparison model (Target-only, Trans-GLM, Trans-GLM-Q, Trans-GLM-IDW,
Trans-GLM-SPH) across all six metrics: AUROC and AUPRC
(Figure~\ref{app:fig:allmethods_auroc_auprc}); sensitivity and PPV at
90\% specificity (Figure~\ref{app:fig:allmethods_90spec}); and
sensitivity and PPV at 95\% specificity
(Figure~\ref{app:fig:allmethods_95spec}). Points are colored by the
sign of the mean difference (Trans-GLMC better vs.\ Trans-GLMC worse);
error bars are normal-approximation 95\% confidence intervals across
the 10 replicates.

\begin{landscape}
  \begin{table}[H]
    \centering
    \caption{Facility-specific demographic characteristics of the patient cohorts across the 27 study hospitals in Connecticut. Facilities are grouped by transferability community and, within each community, ordered by suicide-attempt rate from highest to lowest. The facility index column corresponds to the labels used throughout the manuscript and remaining supplementary tables.}
    \label{app:tab:demo_each_faci}
    \tiny
    \setlength{\tabcolsep}{2pt}
    \begin{tabular}{cclrrrrrrrrr}
      \toprule
      Index & Community & Hospital & Sample size & Suicide Attempts & \multicolumn{4}{c}{Age group (\%)} & Male & Non-Hispanic White & Medicaid insurance \\
      \cmidrule(lr){6-9}
            & & & $N$ & $N$ (\%) & 18--24 & 25--39 & 40--54 & 55--64 & (\%) & (\%) & (\%) \\
      \midrule
      1 & 1 & William W. Backus Hospital & 27767 & 296 (1.07) & 15.5 & 34.8 & 30.9 & 18.9 & 45.1 & 78.8 & 36.2\\
      25 & 1 & Windham Hospital & 11763 & 109 (0.93) & 26.5 & 31.1 & 26.5 & 15.8 & 46.2 & 64.2 & 37.7\\
      5 & 1 & Day Kimball Hospital & 8721 & 58 (0.67) & 16.2 & 34.6 & 29.5 & 19.8 & 45.4 & 93.7 & 34.5\\
      8 & 1 & Hartford Hospital & 50808 & 325 (0.64) & 16.5 & 36.8 & 28.1 & 18.5 & 44.0 & 38.6 & 42.2\\
      13 & 1 & MidState Medical Center & 19974 & 119 (0.60) & 16.2 & 35.4 & 30.6 & 17.8 & 44.4 & 57.2 & 38.2\\
      20 & 1 & Saint Mary's Hospital & 24698 & 130 (0.53) & 17.9 & 36.6 & 29.2 & 16.4 & 44.8 & 44.6 & 48.3\\
      16 & 1 & Hospital of Central Connecticut & 34219 & 176 (0.51) & 17.1 & 37.7 & 28.6 & 16.6 & 44.0 & 48.9 & 45.0\\
      10 & 1 & Johnson Memorial Hospital & 8100 & 38 (0.47) & 15.1 & 34.6 & 31.1 & 19.2 & 47.1 & 86.7 & 31.4\\
      9 & 1 & Charlotte Hungerford Hospital & 13167 & 61 (0.46) & 15.4 & 32.8 & 30.7 & 21.1 & 46.3 & 86.6 & 36.2\\
      26 & 1 & Yale-New Haven Hospital & 78288 & 362 (0.46) & 17.0 & 35.6 & 28.0 & 19.4 & 43.4 & 49.0 & 36.6\\
      3 & 1 & Bristol Hospital & 14135 & 60 (0.42) & 16.3 & 35.7 & 30.5 & 17.4 & 44.7 & 75.4 & 39.9\\
      12 & 1 & Manchester Memorial Hospital & 18371 & 67 (0.36) & 16.2 & 39.1 & 27.8 & 16.8 & 41.7 & 56.9 & 39.2\\
      19 & 1 & Saint Francis Hospital \& Medical Ctr & 39116 & 134 (0.34) & 15.4 & 35.5 & 28.8 & 20.3 & 43.2 & 34.8 & 43.9\\
      7 & 1 & Griffin Hospital & 14844 & 31 (0.21) & 14.3 & 34.7 & 31.0 & 19.9 & 44.7 & 73.2 & 38.7\\
      \addlinespace
      4 & 2 & Danbury Hospital & 30713 & 121 (0.39) & 15.9 & 32.9 & 30.5 & 20.7 & 44.8 & 67.7 & 24.2\\
      18 & 2 & Rockville General Hospital & 8245 & 29 (0.35) & 16.3 & 35.0 & 29.9 & 18.8 & 45.3 & 77.7 & 37.0\\
      11 & 2 & Lawrence \& Memorial Hospital & 29865 & 95 (0.32) & 16.8 & 36.6 & 28.0 & 18.5 & 44.3 & 70.3 & 33.8\\
      14 & 2 & Middlesex Hospital & 32674 & 98 (0.30) & 15.3 & 30.8 & 31.4 & 22.6 & 45.7 & 82.5 & 26.7\\
      27 & 2 & Sharon Hospital & 5129 & 15 (0.29) & 16.6 & 31.4 & 30.1 & 21.9 & 45.8 & 85.2 & 14.0\\
      2 & 2 & Bridgeport Hospital & 34604 & 85 (0.25) & 17.0 & 37.7 & 29.5 & 15.8 & 41.6 & 32.5 & 45.4\\
      22 & 2 & The Stamford Hospital & 21204 & 53 (0.25) & 14.1 & 41.0 & 28.8 & 16.2 & 41.7 & 40.8 & 28.5\\
      17 & 2 & Norwalk Hospital & 19264 & 46 (0.24) & 14.1 & 35.3 & 30.8 & 19.9 & 43.9 & 66.9 & 24.8\\
      21 & 2 & St. Vincent's Medical Center & 28105 & 65 (0.23) & 16.8 & 34.2 & 30.8 & 18.2 & 45.4 & 40.2 & 38.9\\
      23 & 2 & UConn John Dempsey Hospital & 16441 & 38 (0.23) & 13.2 & 33.8 & 31.6 & 21.4 & 45.0 & 73.8 & 32.5\\
      24 & 2 & Waterbury Hospital & 19395 & 42 (0.22) & 16.6 & 36.8 & 29.0 & 17.7 & 44.3 & 51.3 & 45.8\\
      15 & 2 & Milford Hospital & 9809 & 12 (0.12) & 14.2 & 32.3 & 32.1 & 21.3 & 47.3 & 74.9 & 28.9\\
      6 & 2 & Greenwich Hospital & 17339 & 13 (0.07) & 12.3 & 39.0 & 30.6 & 18.0 & 38.6 & 58.1 & 16.2\\
      \bottomrule
    \end{tabular}
  \end{table}
\end{landscape}

\begin{table}[H]
  \tiny\centering
  \caption{Facility-specific AUROC performance across different
    models.}
  \label{app:tab:auroc_each_faci}
  \begin{tabular}{l r r r r r r}
    \toprule
    Facility & Target-only & Trans-GLM & Trans-GLM-Q & Trans-GLM-IDW & Trans-GLM-SPH & Trans-GLMC\\
    \midrule
    1 & 81.90 (4.22) & 81.78 (4.37) & 81.70 (4.22) & 82.34 (4.34) & 82.33 (4.31) & 82.60 (4.55)\\
    2 & 77.48 (7.87) & 76.48 (8.68) & 76.46 (8.20) & 78.88 (4.76) & 79.21 (6.37) & 79.19 (7.20)\\
    3 & 70.71 (13.59) & 73.64 (11.78) & 77.16 (9.47) & 74.44 (9.73) & 70.94 (8.17) & 76.51 (10.58)\\
    4 & 78.78 (7.70) & 82.58 (8.61) & 83.82 (7.58) & 81.58 (8.63) & 81.29 (8.74) & 84.81 (7.55)\\
    5 & 73.79 (10.24) & 80.20 (9.92) & 79.80 (10.81) & 77.73 (9.21) & 78.25 (8.86) & 80.25 (10.34)\\
    6 & 50.00 (0.00) & 59.78 (21.03) & 59.55 (21.11) & 59.31 (21.01) & 54.55 (15.86) & 56.81 (23.92)\\
    7 & 53.15 (7.87) & 77.61 (11.07) & 80.34 (8.90) & 75.89 (9.29) & 70.21 (12.03) & 79.28 (9.55)\\
    8 & 85.00 (3.42) & 84.40 (3.26) & 84.51 (3.09) & 83.72 (3.87) & 84.52 (3.77) & 84.51 (3.38)\\
    9 & 73.61 (8.37) & 80.00 (7.54) & 81.10 (7.63) & 80.97 (6.74) & 79.60 (6.75) & 80.61 (6.74)\\
    \addlinespace
    10 & 69.49 (14.44) & 78.19 (12.66) & 76.95 (14.88) & 76.15 (12.91) & 73.37 (12.41) & 77.91 (12.52)\\
    11 & 79.37 (6.51) & 79.03 (7.35) & 79.86 (6.54) & 79.98 (7.27) & 80.66 (6.71) & 81.42 (7.46)\\
    12 & 79.86 (6.53) & 84.29 (7.70) & 83.22 (8.09) & 79.79 (6.89) & 81.45 (6.72) & 82.12 (8.04)\\
    13 & 74.80 (8.73) & 78.66 (8.63) & 79.49 (8.41) & 78.89 (8.82) & 77.08 (10.10) & 79.40 (7.95)\\
    14 & 76.53 (7.95) & 81.20 (6.26) & 81.58 (2.93) & 79.72 (6.06) & 78.79 (5.59) & 82.07 (4.23)\\
    15 & 52.73 (8.62) & 68.63 (24.80) & 72.95 (25.50) & 68.58 (25.09) & 66.88 (24.84) & 65.08 (29.75)\\
    16 & 77.87 (7.91) & 79.47 (6.67) & 79.48 (6.76) & 81.26 (6.32) & 80.45 (6.43) & 80.92 (5.33)\\
    17 & 72.68 (5.86) & 75.61 (9.98) & 77.96 (10.77) & 74.48 (8.68) & 74.09 (8.96) & 77.47 (12.62)\\
    18 & 67.39 (10.81) & 75.90 (11.86) & 77.45 (11.40) & 76.31 (12.63) & 75.45 (12.79) & 80.49 (12.25)\\
    \addlinespace
    19 & 76.22 (7.40) & 77.47 (8.16) & 77.44 (6.43) & 76.69 (8.09) & 77.55 (8.55) & 79.58 (6.91)\\
    20 & 78.11 (6.99) & 79.23 (4.03) & 79.74 (3.00) & 77.59 (5.03) & 76.84 (6.99) & 80.36 (4.00)\\
    21 & 75.39 (9.51) & 77.37 (9.51) & 77.29 (9.53) & 77.46 (10.26) & 76.96 (10.47) & 78.66 (8.74)\\
    22 & 73.28 (10.61) & 76.43 (14.62) & 76.46 (11.68) & 74.77 (12.14) & 71.67 (10.53) & 75.87 (15.12)\\
    23 & 75.25 (8.71) & 78.55 (4.94) & 80.23 (9.34) & 79.93 (7.90) & 77.01 (9.59) & 74.42 (6.31)\\
    24 & 67.88 (10.99) & 71.02 (8.60) & 77.06 (9.82) & 69.10 (9.65) & 70.64 (10.80) & 72.59 (10.15)\\
    25 & 78.15 (7.13) & 77.67 (6.58) & 77.21 (7.57) & 78.67 (7.67) & 78.71 (7.50) & 79.70 (7.72)\\
    26 & 85.05 (3.62) & 85.19 (3.88) & 85.18 (3.43) & 85.73 (3.77) & 85.68 (3.72) & 86.41 (3.11)\\
    27 & 44.21 (13.49) & 71.95 (25.17) & 70.16 (22.17) & 61.57 (18.95) & 65.07 (23.35) & 75.93 (16.93)\\
    \bottomrule
  \end{tabular}
\end{table}

\begin{table}[H]
  \tiny\centering
  \caption{Facility-specific AUPRC performance across different
    models.}
  \label{app:tab:auprc_each_faci}
  \begin{tabular}{l r r r r r r}
    \toprule
    Facility & Target-only & Trans-GLM & Trans-GLM-Q & Trans-GLM-IDW & Trans-GLM-SPH & Trans-GLMC\\
    \midrule
    1 & 8.31 (4.59) & 8.48 (4.56) & 7.77 (3.20) & 8.28 (4.69) & 8.18 (4.68) & 8.32 (4.52)\\
    2 & 2.07 (1.21) & 2.27 (1.23) & 3.31 (3.64) & 3.18 (3.55) & 3.20 (3.54) & 1.75 (0.58)\\
    3 & 3.86 (5.85) & 3.18 (5.90) & 1.92 (0.77) & 1.69 (0.82) & 1.41 (0.67) & 2.00 (1.22)\\
    4 & 3.73 (2.11) & 8.97 (6.03) & 7.54 (5.61) & 7.88 (5.19) & 8.16 (6.07) & 6.73 (4.83)\\
    5 & 3.57 (1.96) & 10.87 (9.25) & 11.27 (8.51) & 11.91 (8.27) & 10.46 (7.12) & 10.31 (6.07)\\
    6 & 0.07 (0.03) & 10.49 (31.48) & 1.46 (3.05) & 0.86 (1.85) & 0.72 (2.04) & 10.12 (31.58)\\
    7 & 0.27 (0.16) & 2.60 (3.02) & 2.79 (3.69) & 2.86 (3.74) & 1.79 (2.07) & 7.72 (14.00)\\
    8 & 9.09 (3.98) & 9.31 (4.18) & 9.47 (4.43) & 9.65 (4.69) & 9.99 (4.47) & 9.87 (4.42)\\
    9 & 4.11 (5.47) & 4.71 (2.26) & 7.40 (5.96) & 6.74 (5.01) & 6.29 (5.15) & 8.95 (6.69)\\
    \addlinespace
    10 & 7.40 (10.65) & 10.84 (12.68) & 10.97 (13.08) & 11.41 (13.59) & 11.36 (16.90) & 11.88 (14.63)\\
    11 & 3.73 (3.05) & 4.00 (2.39) & 3.15 (1.92) & 4.16 (4.14) & 3.97 (3.91) & 3.93 (4.42)\\
    12 & 2.52 (1.81) & 7.25 (8.44) & 6.33 (6.52) & 6.48 (6.42) & 4.91 (4.95) & 4.82 (2.77)\\
    13 & 5.83 (6.08) & 5.57 (5.73) & 6.68 (6.29) & 6.28 (4.96) & 5.07 (4.20) & 6.62 (6.05)\\
    14 & 4.03 (4.00) & 6.50 (4.38) & 7.19 (4.94) & 6.25 (3.82) & 5.88 (3.58) & 6.04 (3.30)\\
    15 & 0.18 (0.21) & 5.83 (15.63) & 6.01 (15.57) & 5.96 (15.61) & 0.85 (1.25) & 6.28 (15.53)\\
    16 & 5.41 (3.49) & 4.25 (1.96) & 5.50 (3.55) & 5.51 (2.88) & 6.39 (3.80) & 5.52 (3.64)\\
    17 & 5.11 (3.63) & 7.99 (5.52) & 13.18 (12.23) & 10.39 (9.53) & 8.77 (9.00) & 8.74 (6.32)\\
    18 & 4.48 (10.43) & 8.07 (13.76) & 2.77 (3.01) & 5.72 (11.08) & 5.78 (10.66) & 2.86 (3.26)\\
    \addlinespace
    19 & 5.44 (5.08) & 6.67 (6.15) & 6.12 (5.29) & 5.41 (5.21) & 6.38 (6.46) & 6.42 (6.76)\\
    20 & 6.17 (5.33) & 8.93 (6.68) & 7.45 (6.28) & 7.24 (4.63) & 7.20 (4.79) & 7.04 (4.24)\\
    21 & 4.43 (6.27) & 4.85 (7.06) & 3.48 (4.91) & 6.60 (9.11) & 5.05 (7.05) & 5.02 (7.21)\\
    22 & 2.41 (2.18) & 6.06 (8.28) & 5.51 (7.83) & 5.92 (7.53) & 5.29 (7.57) & 5.98 (8.54)\\
    23 & 3.81 (3.96) & 7.03 (9.40) & 5.42 (8.04) & 8.78 (11.26) & 7.62 (8.18) & 9.59 (10.74)\\
    24 & 1.67 (1.42) & 3.86 (4.10) & 6.70 (9.07) & 3.43 (3.47) & 4.52 (4.96) & 9.06 (11.33)\\
    25 & 6.34 (3.34) & 7.32 (3.96) & 8.29 (4.27) & 8.89 (3.94) & 9.52 (4.46) & 9.02 (4.95)\\
    26 & 6.79 (2.53) & 7.09 (2.61) & 8.39 (3.77) & 7.46 (2.34) & 7.70 (2.60) & 7.72 (3.02)\\
    27 & 0.29 (0.19) & 6.28 (15.76) & 0.91 (0.83) & 0.64 (0.78) & 1.01 (1.18) & 1.21 (1.19)\\
    \bottomrule
  \end{tabular}
\end{table}

\begin{table}[H]
  \tiny\centering
  \caption{Facility-specific sensitivity at 90\% specificity across
    different models.}
  \label{app:tab:sens90_each_faci}
  \begin{tabular}{l r r r r r r}
    \toprule
    Facility & Target-only & Trans-GLM & Trans-GLM-Q & Trans-GLM-IDW & Trans-GLM-SPH & Trans-GLMC\\
    \midrule
    1 & 47.69 (10.13) & 48.32 (8.43) & 51.00 (7.90) & 51.69 (9.22) & 52.05 (9.12) & 50.70 (10.43)\\
    2 & 46.94 (13.64) & 46.94 (14.86) & 49.58 (19.63) & 49.31 (10.17) & 48.89 (13.22) & 49.44 (12.18)\\
    3 & 28.33 (15.81) & 37.50 (13.18) & 43.33 (19.56) & 45.00 (20.86) & 44.17 (20.05) & 36.67 (18.92)\\
    4 & 51.28 (15.21) & 60.38 (18.98) & 61.22 (17.81) & 61.22 (19.86) & 57.95 (18.24) & 61.22 (20.62)\\
    5 & 36.33 (18.42) & 53.67 (23.65) & 53.67 (19.34) & 54.00 (21.53) & 47.33 (23.14) & 55.33 (22.62)\\
    6 & 10.00 (21.08) & 30.00 (42.16) & 30.00 (42.16) & 30.00 (42.16) & 20.00 (34.96) & 30.00 (42.16)\\
    7 & 10.00 (22.50) & 54.17 (28.67) & 47.50 (24.23) & 50.83 (24.36) & 44.17 (28.34) & 54.17 (28.67)\\
    8 & 65.55 (7.76) & 66.16 (7.70) & 66.47 (8.02) & 66.48 (9.42) & 66.48 (8.83) & 65.56 (8.01)\\
    9 & 42.86 (21.38) & 49.52 (26.44) & 52.86 (23.80) & 51.19 (26.99) & 49.52 (26.44) & 52.62 (26.02)\\
    \addlinespace
    10 & 40.00 (19.16) & 55.00 (21.23) & 52.50 (26.07) & 52.50 (16.22) & 52.50 (26.07) & 56.25 (19.57)\\
    11 & 45.78 (18.59) & 51.00 (20.20) & 52.22 (24.56) & 48.22 (23.20) & 47.89 (19.41) & 53.22 (17.57)\\
    12 & 46.43 (14.94) & 61.19 (16.73) & 58.10 (17.75) & 62.62 (15.55) & 59.76 (17.69) & 57.86 (20.61)\\
    13 & 41.82 (13.30) & 47.65 (19.40) & 47.88 (10.99) & 49.47 (15.12) & 45.23 (16.00) & 49.81 (15.13)\\
    14 & 46.33 (16.89) & 60.56 (13.86) & 56.33 (15.99) & 56.33 (15.28) & 57.44 (16.31) & 58.44 (16.11)\\
    15 & 5.00 (15.81) & 45.00 (49.72) & 55.00 (49.72) & 45.00 (49.72) & 40.00 (51.64) & 45.00 (49.72)\\
    16 & 41.34 (15.05) & 45.98 (13.82) & 44.77 (13.56) & 46.63 (12.60) & 43.30 (12.49) & 46.63 (11.24)\\
    17 & 42.00 (14.94) & 53.00 (16.87) & 55.00 (16.33) & 53.00 (16.87) & 50.50 (19.07) & 59.00 (14.30)\\
    18 & 23.33 (16.10) & 48.33 (16.57) & 51.67 (22.84) & 50.00 (23.57) & 53.33 (23.31) & 60.00 (26.29)\\
    \addlinespace
    19 & 45.00 (15.62) & 52.47 (16.09) & 50.16 (15.05) & 51.76 (16.02) & 52.36 (15.76) & 53.19 (17.64)\\
    20 & 54.62 (14.71) & 53.08 (13.78) & 53.85 (10.26) & 54.62 (13.78) & 52.31 (14.86) & 52.31 (10.13)\\
    21 & 49.52 (16.02) & 49.76 (20.98) & 50.95 (26.47) & 52.62 (23.93) & 52.62 (21.95) & 49.52 (22.72)\\
    22 & 45.67 (23.15) & 47.00 (20.75) & 45.00 (19.45) & 47.00 (20.75) & 46.67 (24.70) & 47.00 (26.41)\\
    23 & 55.83 (15.24) & 58.33 (16.20) & 57.50 (18.61) & 58.33 (20.03) & 55.83 (19.27) & 55.83 (17.59)\\
    24 & 43.00 (21.76) & 50.50 (17.87) & 48.00 (22.88) & 43.00 (21.76) & 45.50 (20.88) & 48.00 (22.88)\\
    25 & 45.36 (11.97) & 44.09 (9.59) & 48.55 (15.21) & 50.45 (15.52) & 49.45 (14.67) & 51.36 (11.35)\\
    26 & 61.06 (5.20) & 61.08 (6.85) & 63.84 (8.12) & 62.14 (8.12) & 62.70 (7.80) & 66.02 (6.26)\\
    27 & 0.00 (0.00) & 35.00 (47.43) & 25.00 (42.49) & 10.00 (31.62) & 20.00 (42.16) & 30.00 (42.16)\\
    \bottomrule
  \end{tabular}
\end{table}

\begin{table}[H]
  \tiny\centering
  \caption{Facility-specific PPV at 90\% specificity
    across different models.}
  \label{app:tab:ppv90_each_faci}
  \begin{tabular}{l r r r r r r}
    \toprule
    Facility & Target-only & Trans-GLM & Trans-GLM-Q & Trans-GLM-IDW & Trans-GLM-SPH & Trans-GLMC\\
    \midrule
    1 & 4.87 (0.97) & 4.93 (0.81) & 5.21 (0.78) & 5.27 (0.90) & 5.30 (0.87) & 5.17 (1.00)\\
    2 & 1.14 (0.35) & 1.14 (0.38) & 1.20 (0.48) & 1.20 (0.26) & 1.20 (0.37) & 1.20 (0.29)\\
    3 & 1.19 (0.66) & 1.57 (0.55) & 1.81 (0.80) & 1.88 (0.85) & 1.84 (0.82) & 1.53 (0.77)\\
    4 & 1.98 (0.57) & 2.33 (0.71) & 2.36 (0.66) & 2.35 (0.74) & 2.23 (0.68) & 2.35 (0.77)\\
    5 & 2.34 (1.21) & 3.42 (1.45) & 3.43 (1.18) & 3.43 (1.29) & 2.99 (1.36) & 3.53 (1.39)\\
    6 & 0.11 (0.24) & 0.23 (0.30) & 0.23 (0.30) & 0.23 (0.30) & 0.17 (0.28) & 0.23 (0.30)\\
    7 & 0.20 (0.45) & 1.13 (0.63) & 1.00 (0.56) & 1.07 (0.56) & 0.93 (0.64) & 1.13 (0.63)\\
    8 & 4.05 (0.46) & 4.09 (0.46) & 4.10 (0.48) & 4.10 (0.56) & 4.10 (0.53) & 4.04 (0.47)\\
    9 & 1.94 (0.94) & 2.23 (1.15) & 2.38 (1.02) & 2.30 (1.17) & 2.23 (1.15) & 2.37 (1.14)\\
    \addlinespace
    10 & 1.81 (0.85) & 2.52 (1.02) & 2.40 (1.24) & 2.40 (0.78) & 2.40 (1.24) & 2.58 (0.96)\\
    11 & 1.42 (0.53) & 1.58 (0.59) & 1.62 (0.73) & 1.49 (0.66) & 1.49 (0.56) & 1.65 (0.48)\\
    12 & 1.67 (0.53) & 2.19 (0.62) & 2.08 (0.67) & 2.24 (0.59) & 2.14 (0.65) & 2.08 (0.80)\\
    13 & 2.45 (0.78) & 2.78 (1.12) & 2.79 (0.64) & 2.88 (0.88) & 2.64 (0.93) & 2.90 (0.88)\\
    14 & 1.36 (0.45) & 1.78 (0.35) & 1.66 (0.45) & 1.66 (0.43) & 1.69 (0.45) & 1.72 (0.44)\\
    15 & 0.10 (0.32) & 0.51 (0.53) & 0.61 (0.52) & 0.51 (0.53) & 0.40 (0.52) & 0.51 (0.53)\\
    16 & 2.09 (0.75) & 2.32 (0.68) & 2.26 (0.69) & 2.35 (0.60) & 2.18 (0.58) & 2.35 (0.54)\\
    17 & 0.98 (0.29) & 1.23 (0.36) & 1.28 (0.36) & 1.23 (0.36) & 1.18 (0.42) & 1.38 (0.34)\\
    18 & 0.84 (0.58) & 1.68 (0.61) & 1.73 (0.58) & 1.68 (0.61) & 1.79 (0.62) & 2.03 (0.79)\\
    \addlinespace
    19 & 1.52 (0.50) & 1.76 (0.51) & 1.69 (0.48) & 1.74 (0.50) & 1.76 (0.51) & 1.79 (0.57)\\
    20 & 2.80 (0.74) & 2.73 (0.69) & 2.77 (0.51) & 2.80 (0.69) & 2.68 (0.74) & 2.69 (0.50)\\
    21 & 1.13 (0.36) & 1.13 (0.46) & 1.16 (0.59) & 1.20 (0.52) & 1.20 (0.47) & 1.13 (0.52)\\
    22 & 1.12 (0.54) & 1.17 (0.54) & 1.12 (0.50) & 1.17 (0.54) & 1.17 (0.66) & 1.17 (0.66)\\
    23 & 1.26 (0.34) & 1.32 (0.38) & 1.32 (0.47) & 1.32 (0.47) & 1.26 (0.44) & 1.28 (0.44)\\
    24 & 0.92 (0.47) & 1.07 (0.37) & 1.02 (0.48) & 0.92 (0.47) & 0.97 (0.44) & 1.02 (0.48)\\
    25 & 4.06 (1.06) & 3.95 (0.81) & 4.33 (1.34) & 4.49 (1.33) & 4.41 (1.28) & 4.57 (0.99)\\
    26 & 2.76 (0.22) & 2.76 (0.29) & 2.88 (0.34) & 2.80 (0.36) & 2.83 (0.34) & 2.98 (0.27)\\
    27 & 0.00 (0.00) & 0.77 (0.99) & 0.58 (0.93) & 0.19 (0.61) & 0.38 (0.81) & 0.77 (0.99)\\
    \bottomrule
  \end{tabular}
\end{table}

\begin{table}[H]
  \tiny\centering
  \caption{Facility-specific sensitivity at 95\% specificity
    across different models.}
  \label{app:tab:sens95_each_faci}
  \begin{tabular}{l r r r r r r}
    \toprule
    Facility & Target-only & Trans-GLM & Trans-GLM-Q & Trans-GLM-IDW & Trans-GLM-SPH & Trans-GLMC\\
    \midrule
    1 & 33.49 (7.96) & 33.51 (7.52) & 32.48 (7.02) & 33.14 (7.58) & 33.46 (8.08) & 32.80 (6.57)\\
    2 & 34.03 (13.64) & 33.75 (13.55) & 36.39 (14.55) & 34.17 (16.40) & 31.53 (18.01) & 37.50 (14.46)\\
    3 & 20.00 (15.32) & 20.00 (18.92) & 26.67 (8.61) & 26.67 (16.10) & 21.67 (13.72) & 20.00 (13.15)\\
    4 & 39.04 (14.95) & 43.01 (17.98) & 47.12 (18.43) & 46.47 (19.57) & 46.47 (19.57) & 46.35 (20.60)\\
    5 & 26.00 (14.56) & 34.67 (17.79) & 40.50 (17.32) & 33.00 (16.96) & 33.00 (16.96) & 38.00 (17.16)\\
    6 & 5.00 (15.81) & 25.00 (42.49) & 25.00 (42.49) & 25.00 (42.49) & 15.00 (33.75) & 25.00 (42.49)\\
    7 & 3.33 (10.54) & 27.50 (27.79) & 34.17 (28.45) & 34.17 (23.72) & 27.50 (27.79) & 28.33 (15.81)\\
    8 & 51.38 (6.98) & 52.28 (7.39) & 55.71 (7.35) & 53.24 (8.04) & 53.24 (7.77) & 53.54 (8.17)\\
    9 & 34.52 (20.05) & 34.52 (22.93) & 37.26 (24.85) & 38.10 (25.15) & 38.10 (27.49) & 34.76 (23.06)\\
    \addlinespace
    10 & 34.17 (25.59) & 42.50 (22.03) & 41.67 (23.24) & 41.67 (28.05) & 41.67 (28.05) & 44.17 (27.51)\\
    11 & 34.22 (16.84) & 35.33 (22.42) & 33.11 (22.52) & 34.33 (23.86) & 33.11 (19.81) & 35.33 (17.78)\\
    12 & 25.71 (8.16) & 46.19 (15.36) & 46.43 (16.39) & 40.48 (15.55) & 43.33 (12.54) & 43.10 (20.28)\\
    13 & 31.74 (14.46) & 33.41 (21.42) & 34.32 (17.60) & 35.15 (19.41) & 33.48 (13.97) & 35.23 (17.35)\\
    14 & 37.22 (14.82) & 47.33 (15.39) & 42.11 (12.20) & 48.22 (17.25) & 44.22 (14.81) & 46.22 (13.19)\\
    15 & 5.00 (15.81) & 35.00 (47.43) & 35.00 (47.43) & 35.00 (47.43) & 20.00 (42.16) & 35.00 (47.43)\\
    16 & 29.41 (8.82) & 32.94 (11.43) & 35.82 (11.27) & 33.56 (8.42) & 33.59 (7.19) & 34.71 (9.16)\\
    17 & 42.00 (14.94) & 49.00 (17.13) & 48.50 (19.01) & 48.50 (19.01) & 48.50 (19.01) & 48.50 (19.01)\\
    18 & 6.67 (14.05) & 30.00 (18.92) & 31.67 (22.84) & 26.67 (26.29) & 26.67 (26.29) & 36.67 (18.92)\\
    \addlinespace
    19 & 32.25 (12.77) & 37.69 (19.94) & 39.84 (14.16) & 39.84 (15.91) & 39.07 (15.09) & 40.60 (18.58)\\
    20 & 36.92 (14.41) & 40.00 (7.07) & 36.15 (12.05) & 39.23 (12.27) & 39.23 (11.15) & 39.23 (12.79)\\
    21 & 32.38 (22.20) & 32.14 (20.18) & 41.67 (21.68) & 41.67 (18.62) & 44.76 (19.47) & 40.24 (20.22)\\
    22 & 26.33 (14.44) & 39.67 (18.69) & 36.33 (21.11) & 39.33 (21.48) & 36.00 (19.23) & 39.67 (23.12)\\
    23 & 42.50 (24.98) & 47.50 (20.05) & 52.50 (20.05) & 50.83 (20.20) & 50.83 (20.20) & 47.50 (20.05)\\
    24 & 40.50 (22.29) & 40.50 (18.92) & 40.50 (25.22) & 35.50 (22.42) & 43.00 (21.76) & 38.00 (19.18)\\
    25 & 30.36 (10.84) & 33.00 (11.39) & 34.73 (12.46) & 35.73 (13.70) & 34.82 (13.28) & 35.73 (11.52)\\
    26 & 49.46 (8.58) & 50.02 (7.06) & 51.95 (8.36) & 50.02 (7.18) & 51.13 (8.51) & 51.12 (6.40)\\
    27 & 0.00 (0.00) & 35.00 (47.43) & 15.00 (33.75) & 10.00 (31.62) & 20.00 (42.16) & 20.00 (42.16)\\
    \bottomrule
  \end{tabular}
\end{table}

\begin{table}[H]
  \tiny\centering
  \caption{Facility-specific PPV at 95\% specificity
    across different models.}
  \label{app:tab:ppv95_each_faci}
  \begin{tabular}{l r r r r r r}
    \toprule
    Facility & Target-only & Trans-GLM & Trans-GLM-Q & Trans-GLM-IDW & Trans-GLM-SPH & Trans-GLMC\\
    \midrule
    1 & 6.73 (1.45) & 6.73 (1.34) & 6.53 (1.28) & 6.65 (1.41) & 6.71 (1.51) & 6.59 (1.19)\\
    2 & 1.64 (0.67) & 1.65 (0.72) & 1.75 (0.72) & 1.65 (0.81) & 1.53 (0.92) & 1.81 (0.74)\\
    3 & 1.66 (1.25) & 1.65 (1.55) & 2.21 (0.69) & 2.21 (1.31) & 1.80 (1.12) & 1.66 (1.08)\\
    4 & 2.96 (1.09) & 3.28 (1.32) & 3.57 (1.35) & 3.51 (1.41) & 3.51 (1.41) & 3.52 (1.50)\\
    5 & 3.34 (1.85) & 4.40 (2.27) & 5.13 (2.23) & 4.19 (2.15) & 4.19 (2.15) & 4.82 (2.17)\\
    6 & 0.11 (0.36) & 0.34 (0.56) & 0.34 (0.56) & 0.34 (0.56) & 0.23 (0.48) & 0.34 (0.56)\\
    7 & 0.13 (0.42) & 1.18 (1.29) & 1.44 (1.29) & 1.45 (1.13) & 1.18 (1.29) & 1.20 (0.75)\\
    8 & 6.19 (0.80) & 6.31 (0.86) & 6.70 (0.84) & 6.42 (0.91) & 6.41 (0.88) & 6.46 (0.93)\\
    9 & 3.06 (1.71) & 3.05 (1.96) & 3.26 (2.11) & 3.33 (2.14) & 3.33 (2.34) & 3.05 (1.97)\\
    \addlinespace
    10 & 3.10 (2.26) & 3.81 (1.96) & 3.80 (2.23) & 3.79 (2.50) & 3.79 (2.50) & 4.02 (2.47)\\
    11 & 2.09 (0.94) & 2.15 (1.28) & 2.02 (1.30) & 2.09 (1.38) & 2.03 (1.15) & 2.16 (1.02)\\
    12 & 1.82 (0.50) & 3.27 (1.13) & 3.27 (1.13) & 2.86 (1.10) & 3.07 (0.90) & 3.04 (1.46)\\
    13 & 3.67 (1.64) & 3.82 (2.38) & 3.92 (1.96) & 4.01 (2.15) & 3.84 (1.54) & 4.02 (1.92)\\
    14 & 2.15 (0.74) & 2.75 (0.79) & 2.46 (0.65) & 2.80 (0.95) & 2.57 (0.78) & 2.69 (0.69)\\
    15 & 0.20 (0.63) & 0.80 (1.03) & 0.80 (1.03) & 0.80 (1.03) & 0.40 (0.84) & 0.80 (1.03)\\
    16 & 2.96 (0.90) & 3.29 (1.09) & 3.57 (1.07) & 3.35 (0.80) & 3.35 (0.66) & 3.46 (0.88)\\
    17 & 1.92 (0.57) & 2.22 (0.63) & 2.23 (0.79) & 2.23 (0.79) & 2.23 (0.79) & 2.23 (0.79)\\
    18 & 0.48 (1.00) & 2.05 (1.34) & 2.24 (1.59) & 1.88 (1.84) & 1.88 (1.84) & 2.52 (1.35)\\
    \addlinespace
    19 & 2.16 (0.81) & 2.49 (1.23) & 2.64 (0.85) & 2.64 (0.96) & 2.60 (0.94) & 2.69 (1.14)\\
    20 & 3.74 (1.41) & 4.05 (0.69) & 3.65 (1.17) & 3.97 (1.19) & 3.97 (1.07) & 3.97 (1.25)\\
    21 & 1.47 (1.00) & 1.47 (0.95) & 1.89 (0.97) & 1.89 (0.86) & 2.03 (0.88) & 1.82 (0.93)\\
    22 & 1.30 (0.77) & 1.93 (0.90) & 1.75 (1.01) & 1.94 (1.09) & 1.76 (0.90) & 1.93 (1.18)\\
    23 & 1.90 (1.13) & 2.14 (0.91) & 2.37 (0.94) & 2.25 (0.85) & 2.25 (0.85) & 2.14 (0.91)\\
    24 & 1.73 (0.95) & 1.73 (0.82) & 1.72 (1.06) & 1.52 (0.98) & 1.83 (0.92) & 1.63 (0.84)\\
    25 & 5.34 (1.82) & 5.79 (1.93) & 6.10 (2.14) & 6.23 (2.32) & 6.08 (2.24) & 6.23 (1.98)\\
    26 & 4.38 (0.72) & 4.43 (0.59) & 4.60 (0.69) & 4.43 (0.60) & 4.52 (0.70) & 4.53 (0.53)\\
    27 & 0.00 (0.00) & 1.48 (1.91) & 0.74 (1.56) & 0.37 (1.17) & 0.74 (1.56) & 0.74 (1.56)\\
    \bottomrule
  \end{tabular}
\end{table}

\begin{figure}[H]
  \centering
  \begin{subfigure}[t]{\linewidth}
    \centering
    \includegraphics[width=\linewidth]{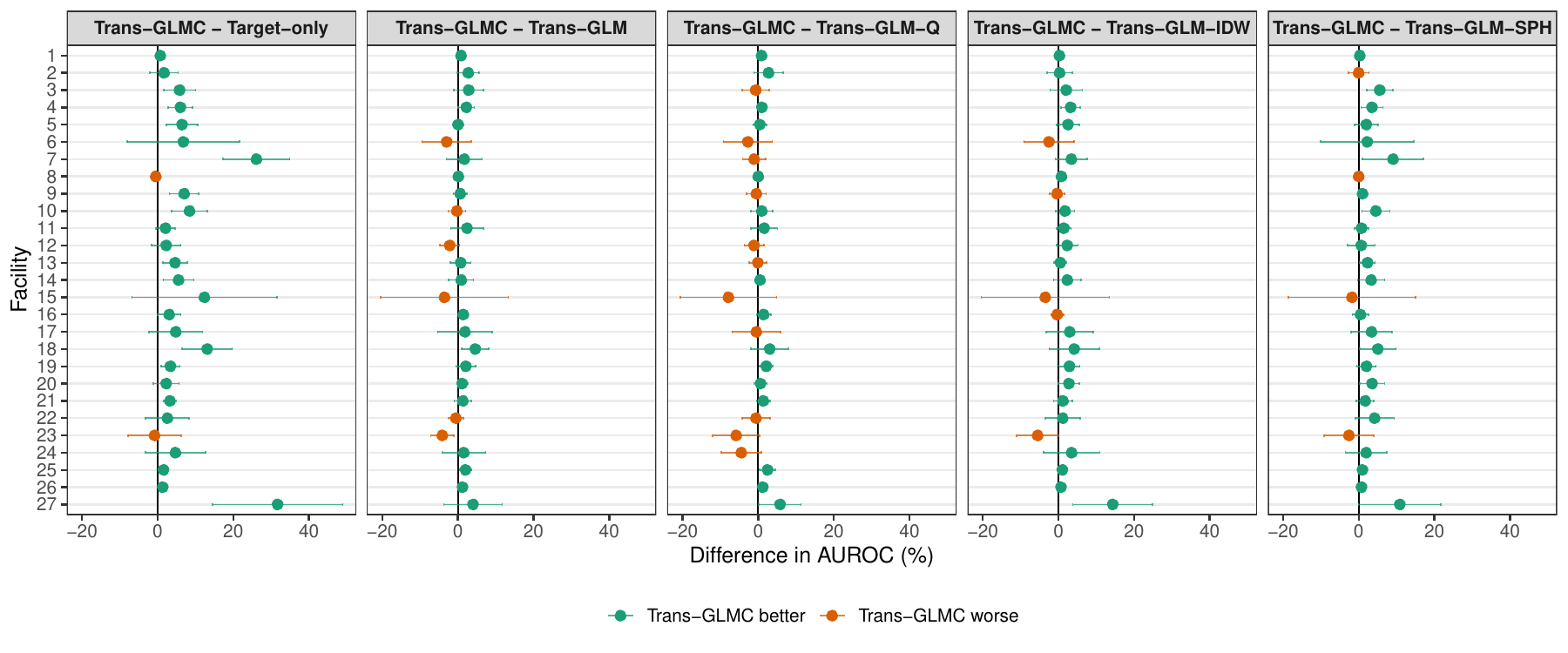}
    \caption{AUROC}
  \end{subfigure}\\[1ex]
  \begin{subfigure}[t]{\linewidth}
    \centering
    \includegraphics[width=\linewidth]{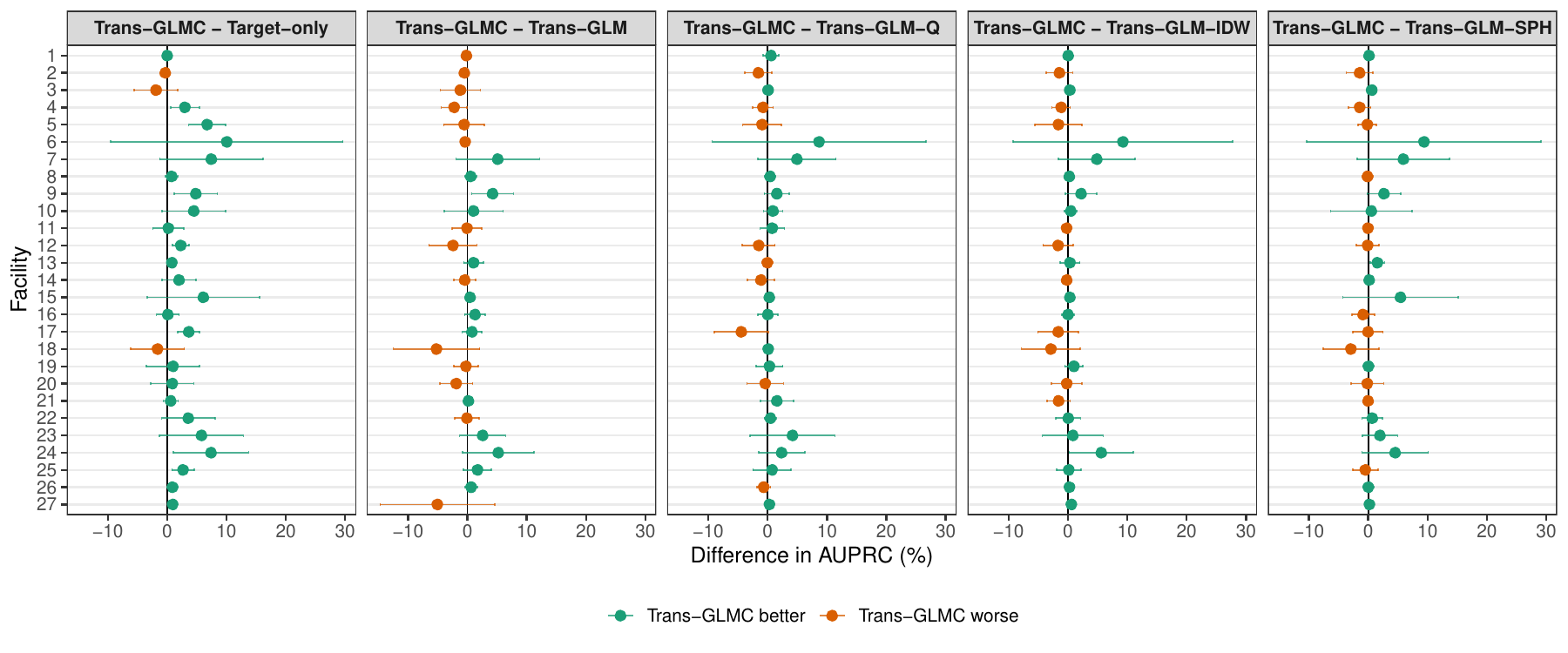}
    \caption{AUPRC}
  \end{subfigure}
  \caption{Application: Facility-specific AUROC and AUPRC differences
    between Trans-GLMC and each comparison model (Target-only,
    Trans-GLM, Trans-GLM-Q, Trans-GLM-IDW, Trans-GLM-SPH) across CHIME
    facilities.  Points show mean differences across 10 replicates;
    error bars are normal-approximation 95\% confidence intervals.}
  \label{app:fig:allmethods_auroc_auprc}
\end{figure}

\begin{figure}[H]
  \centering
  \begin{subfigure}[t]{\linewidth}
    \centering
    \includegraphics[width=\linewidth]{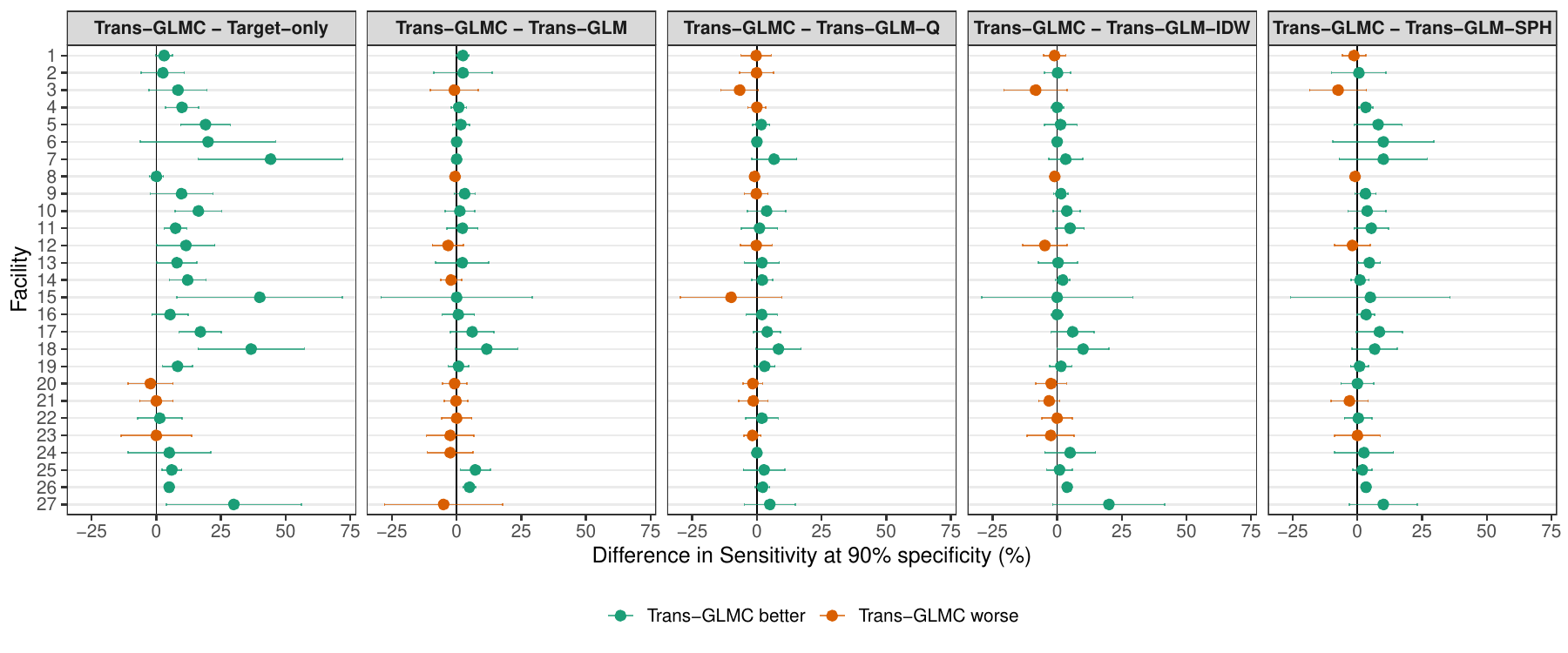}
    \caption{Sensitivity at 90\% specificity}
  \end{subfigure}\\[1ex]
  \begin{subfigure}[t]{\linewidth}
    \centering
    \includegraphics[width=\linewidth]{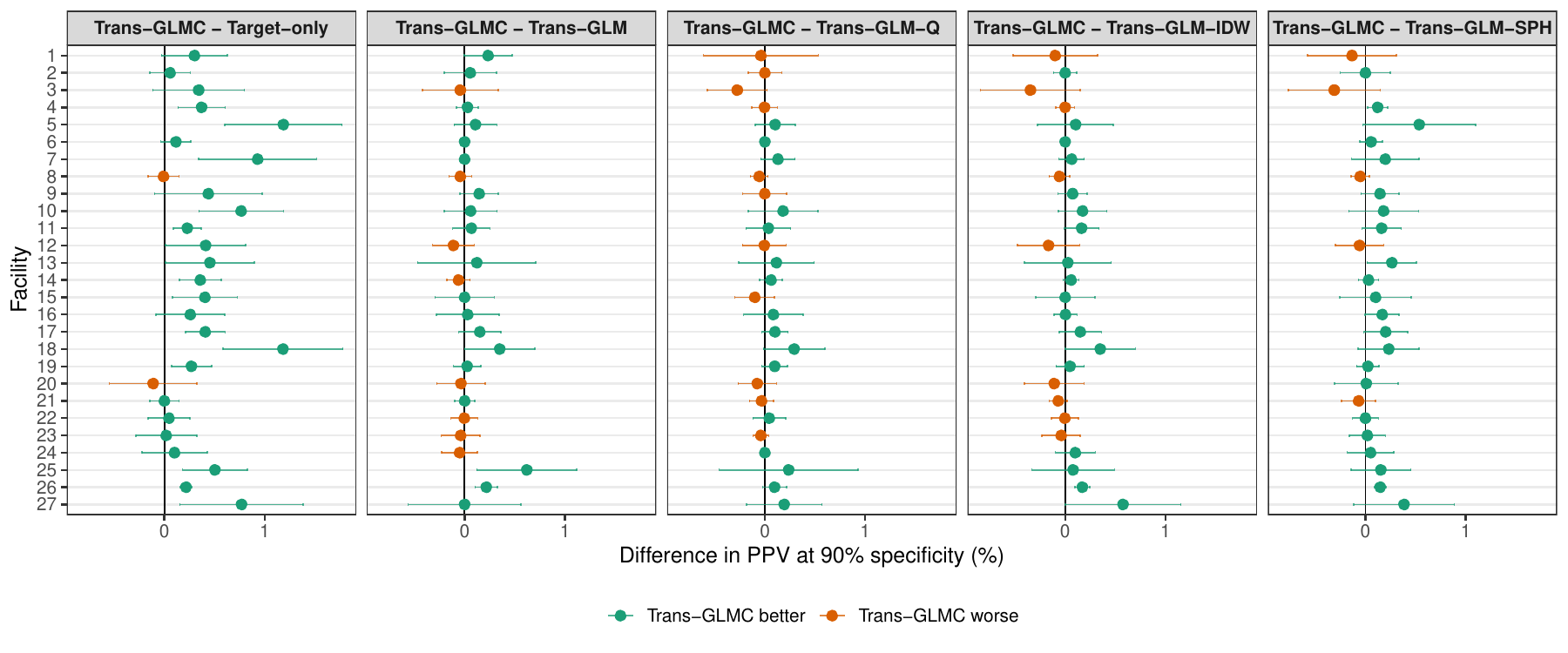}
    \caption{PPV at 90\% specificity}
  \end{subfigure}
  \caption{Application: Facility-specific differences in sensitivity
    and PPV at 90\% specificity between Trans-GLMC and each comparison
    model (Target-only, Trans-GLM, Trans-GLM-Q, Trans-GLM-IDW,
    Trans-GLM-SPH) across CHIME facilities.  Points show mean
    differences across 10 replicates; error bars are normal-approximation 95\%
    confidence intervals.}
  \label{app:fig:allmethods_90spec}
\end{figure}

\begin{figure}[H]
  \centering
  \begin{subfigure}[t]{\linewidth}
    \centering
    \includegraphics[width=\linewidth]{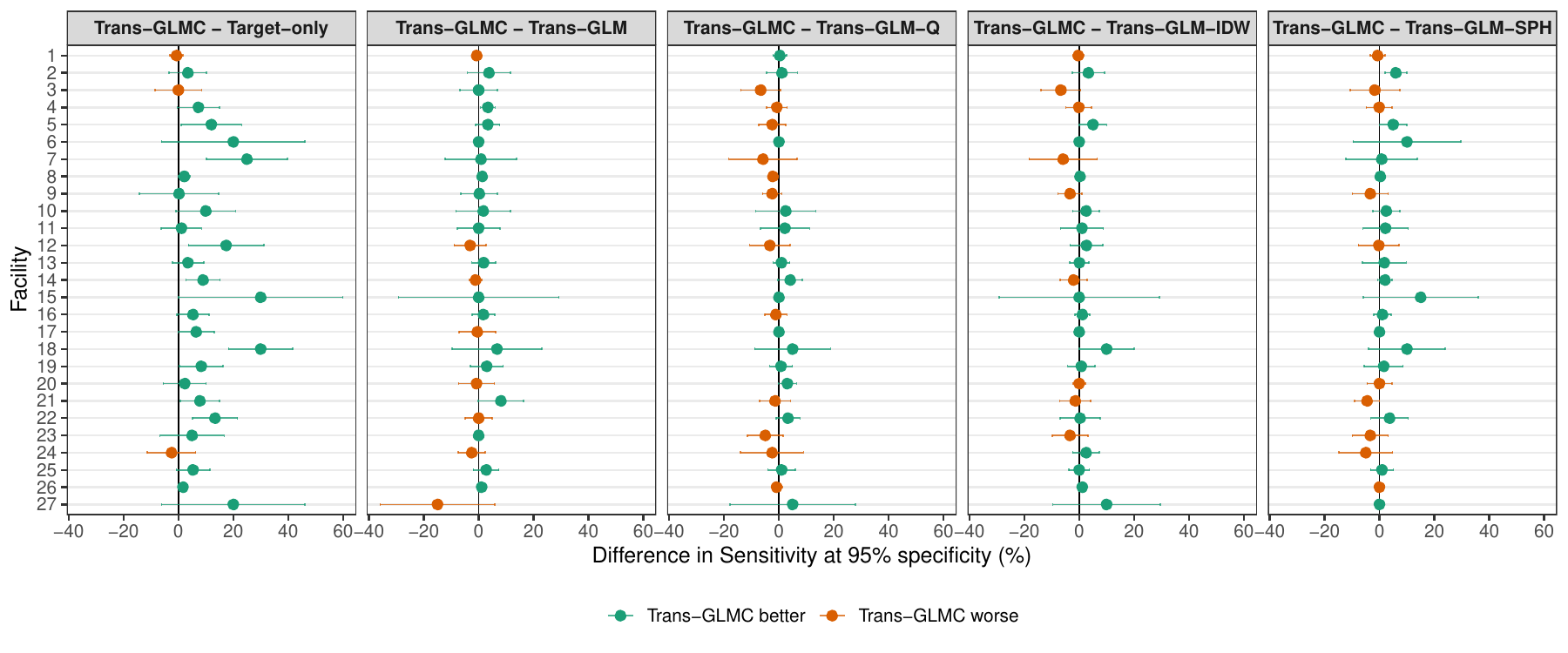}
    \caption{Sensitivity at 95\% specificity}
  \end{subfigure}\\[1ex]
  \begin{subfigure}[t]{\linewidth}
    \centering
    \includegraphics[width=\linewidth]{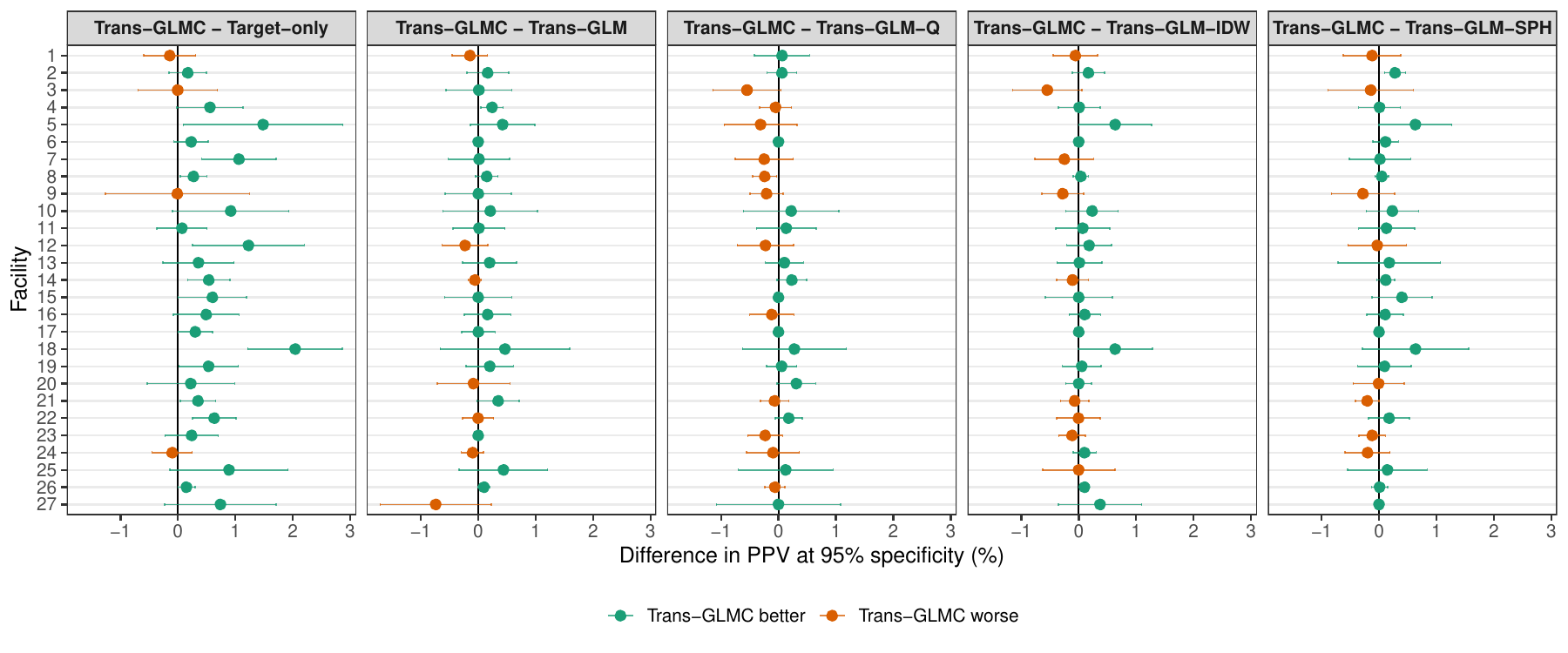}
    \caption{PPV at 95\% specificity}
  \end{subfigure}
  \caption{Application: Facility-specific differences in sensitivity
    and PPV at 95\% specificity between Trans-GLMC and each comparison
    model (Target-only, Trans-GLM, Trans-GLM-Q, Trans-GLM-IDW,
    Trans-GLM-SPH) across CHIME facilities.  Points show mean
    differences across 10 replicates; error bars are normal-approximation 95\%
    confidence intervals.}
  \label{app:fig:allmethods_95spec}
\end{figure}

\end{document}